\documentclass[review]{elsarticle}
\usepackage{hyperref}
\usepackage{geometry} 
\usepackage{setspace}
\usepackage{soul}
\usepackage{color}
\usepackage{changes}
\usepackage{amssymb}
\usepackage{algorithm}
\usepackage{epstopdf}
\usepackage{longtable}

\def\C{{\mathcal C}}

\def\L{{\mathcal L}}

\def\X{{\mathcal X}}
\def\Y{{\mathcal Y}}
\def\Z{{\mathcal Z}}
\usepackage{amsmath}
\newtheorem{theorem}{Theorem}[section]

\newtheorem{corollary}{Corollary}[section]
\newtheorem{proof}{Proof}
\DeclareMathOperator*{\argmin}{arg\,min}

\begin{document}

\begin{frontmatter}
	\title{A Binary Characterization Method for Shape Convexity and Applications}
	
	\author[mymainaddress]{Shousheng Luo}
	\author[mysecondaryaddress]{Jinfeng Chen}
	\author[mymainaddress]{Yunhai Xiao\corref{mycorrespondingauthor}}
	\cortext[mycorrespondingauthor]{Corresponding author}
	\ead{yhxiao@henu.edu.cn}
	\author[mythridaddress]{Xue-Cheng Tai}
	
	\address[mymainaddress]{Center for Applied Mathematics of Henan Province, Henan University, Kaifeng 475000, China}
	\address[mysecondaryaddress]{School of Mathematics and Statistics, Henan University, Kaifeng 475000, China}
	\address[mythridaddress]{Department of Mathematics,  Hong Kong Baptist University, Kowloon Tong, Hong Kong}

	\begin{abstract}
		Convexity prior is one of the main cue for human vision and shape completion with important applications in image processing, computer vision.
		This paper focuses on  characterization methods for convex objects and applications in image processing.
		We present a new method for convex objects representations using binary functions, that is, the convexity of a region is equivalent to a simple quadratic inequality constraint on its indicator function.
	   Models are proposed firstly by incorporating this result for image segmentation with convexity prior and convex hull computation of a given set with and without noises.
	   Then, these models are summarized to a general optimization problem
	   on binary function(s) with the quadratic inequality.
		Numerical algorithm is proposed based on linearization technique, where
		the linearized problem  is solved by a proximal alternating direction method of multipliers with guaranteed convergent.
		Numerical experiments demonstrate the efficiency and effectiveness of the proposed methods for image segmentation and convex hull computation in accuracy and computing time.
	\end{abstract}
	
\end{frontmatter}

\section{Introduction}
\label{sec1}
Shape priors are widely used in image processing, computer vision and
other fields. The commonly concerned generic shape priors contain
star shape \cite{2019star-convex,2022star,Veksler2008StarSP}, convexity \cite{Gorelick2016convexity,luo2019convex}, connectivity \cite{2019connecvtivity}, constant width \cite{Parameteric2022}.
The mathematical characterization methods for the concerned priors are very important in real applications.
In this paper, we focus on the characterization method of convexity and its applications because the convexity prior is widely used in
image segmentation and computer vision.
Convexity prior is one of the main cue for human vision and shape completion.
Therefore, many characterization methods for convexity have been proposed in the literature. These methods can be roughly categorized into
two classes:  boundary or curvature  methods
\cite{ukwatta2013efficient,bae2017ALMEuler,2022Luo} and region methods \cite{Gorelick2016convexity,royer2016convexity}.

The boundary methods are based on the fact that the convexity of a region is equivalent to the nonnegativity of the curvature of
its boundary. In order to compute the curvature $\kappa$ of an evolving boundary curve efficiently, the curve is usually embedded into a two dimensional function, named level set function $\phi$, and $\kappa=\nabla\cdot\frac{\nabla\phi}{|\nabla\phi|}$ \cite{goldman2005curvature} for each level set curve. In the pioneer work, the nonnegative curvatures of the boundary curve(s) are penalized only in the models to steer and preserve the convexity of the
concerned objects \cite{ukwatta2013efficient,yang2017ventricle,bae2017ALMEuler}.
Later, a special level set function, called signed distance function (SDF), is used to compute the  curvature. Not only the computation formula is simplified as $\Delta\phi$ due to $|\nabla\phi|=1$ for SDF, but also it is proved that
the convexity of the concerned region is equivalent to $\Delta\phi\geq0$ on the whole image domain rather than on the object boundary only \cite{yan2020convexity,luo2019convex}. This method is generalized to multiple convex objects \cite{luo2019convex}, convex ring \cite{2022Luo}, 3D convex object \cite{2021Luo} characterization, and applied to compute convex hull of given sets with and without outliers \cite{li2021new}.

The region methods are based on the mathematical definition of convexity, i.e.,
all the line segments between any two points in the region belong to the region as well.
For these methods, binary function, which is 1 on the concerned region and 0 otherwise, i.e., the indicator function of the concerned region, is introduced.
Based on the convexity definition and the introduced binary function,
all $1$-$0$-$1$ configurations on all lines in the image domain should be excluded
to obtain a convex result in \cite{Gorelick2016convexity}. Then
this method is extended to multiple convex objects segmentation \cite{gorelick2017multi}.
Another region based method employs the fact that
the line segment of the two ends of any curve in a convex region is in the convex region  \cite{royer2016convexity}.

We lack efficient algorithms for the involved minimization problems of the two aforementioned approaches.
The minimization problems for the region methods are NP-hard and only approximation solutions can be obtained   by penalization and linearization techniques
because the associated problems are combination optimization issues \cite{Gorelick2016convexity,gorelick2017multi,royer2016convexity}.
The minimization problems for the
boundary methods contain a nonsmooth and nonconvex term from curvature formula \cite{2022Luo,yang2017ventricle,bae2017ALMEuler}, which is very difficult to tackle.
Therefore, it is deserved to develop a new characterization method for convexity shape prior.

Very recently, we proposed a novel method for convex shape characterization in our short paper \cite{binary2021Luo}.
Let $o$ be the indicator function associated with the concerned object $D$.
The convexity $D$ is equivalently characterized by
a quadratic constraint on $o$ , that is
$
[b_r\ast o](x)[1-o(x)]\leq \frac{1}{2}[1-o(x)],~~ \forall r\geq0,
$
where '$\ast$' denotes the convolution operation of two functions, and $b_r$ is a positive radial function defined on the disc with radius $r$ and center $(0,0)$ such that $\int_{\mathbb{R}^2}b_r(x)dx=1$.

This method combines the region method and boundary method.
On one hand, this method characterizes convexity by the nonnegativity of curvature
on the boundary essentially.
On the other hand, binary function is adopted to denote the
object of interest like the region based method.
Thus, this method uses binary function rather than general level set function to compute the boundary curvature.

The current paper extends this method in three aspects.
Firstly, a general model
is proposed, which can handle
multiple objects segmentation with convexity prior and convex hull computation of a given set. A model for multiple objects segmentation with
convexity prior is proposed by incorporating
the characterization method into a probability model, and
two variational models for convex hull computation of a given set with and without noises are designed, respectively. Then, the models are
summarized as a general model mathematically.
Secondly, an efficient algorithm is designed for the general model.
Linearization technique for the quadratic terms is employed to
alleviate the solving difficulties,
and then a proximal alternating direction method of multipliers (ADMM) \cite{PADMM} is designed for the linearization problem.
The proximal ADMM is very simple and efficient with closed form solutions to all sub-steps.
Thirdly, an interactive technique is adopted
to improve segmentation accuracy for complicated images, e.g. images with cluttered edges or illumination bias. Extensive experiments show the effectiveness and efficiency of the proposed models and algorithms.

The rest of this paper is organized as follows. Section \ref{sec2} introduces some notations and preliminary results.
The proposed  models are presented in Section \ref{ourm}.
Section \ref{sec3} is devoted to the
algorithms for the general model.
Numerical experiments and some performance comparisons are demonstrated in Section \ref{sec4}. Section \ref{sec5} concludes this paper and discusses some interesting topics for further research.
\section{Preliminaries}\label{sec2}
\setcounter{equation}{0}
In this section, we will introduce some notations, and present a
brief review on the proximal ADMM for convex minimization problem.
\subsection{Notations}

Finite spaces are denoted by capital letters, e.g. $\mathcal{X}$ and $\mathcal{Y}$.
For a given  symmetric (self-adjoint) positive definitive (SPD) matrix $H$ and any $x,y\in\mathcal{X}=\mathbb{R}^d$, we define $\langle x, y \rangle_H=\langle Hx,y\rangle$, and  $\|x\|_H=\sqrt{\langle x,x\rangle_H}$, where $\langle\cdot,\cdot\rangle$ denotes the
canonical inner product, i.e. $\langle x, y\rangle=\sum_{i=1}^d{x_i y_i}$.
The subscribe $_H$ will be omitted when $H$ is the identical operator. And $x\geq y$ means that $x_i\geq y_i$ for all the components of
$x$ and $y$.
We use $^{\top}$ to denote the transpose (or conjugate) of a vector and a matrix.
Let $f:\mathcal{X}\longrightarrow (-\infty,+\infty]$ be a closed proper convex function, $\partial f(x)$ denotes the
subdifferential of $f$ at $x$. In addition, we also use ${0}$ and ${1}$ to denote the vector with all components equaling to 0 and 1, respectively.

Let $\Omega\subset\mathbb{R}^2$ be an image domain and $D$
be a subregion of $\Omega$.
Usually, the indicator function of a given set $\mathbb{S}$ is denoted by $\mathbb{I}_\mathbb{S}$, i.e. $\mathbb{I}_\mathbb{S}(x)=1$ for $x\in \mathbb{S}$ and 0 otherwise.
Let 
$b_r(x)=\rho(\|x\|)\geq0$ for $\|x\|<r$ and 0 otherwise,
and satisfy $\int_{\|x\|<r}b_r(x)dx=1$.
For any given two vector-valued functions
$p(x):=(p_0(x),p_1(x),\cdots,p_n(x))^\top$ and
$q(x):=(q_0(x),q_1(x),\cdots,q_n(x))^\top$,
operation $p\odot q$ means
$
p\odot q(x)= \left(p_0(x)q_0(x),p_1(x)q_1(x),\cdots,p_n(x)q_n(x) \right)^\top,
$
and the convolution between $b_r(x)$ and $p$ is defined as
\begin{equation}\label{convolution}
	\big(b_r * p\big)_i(x) =
	\left\{
	\begin{array}{ll}
		\int_{\mathbb{R}^2} b_r (x-y)p_i(y) dy,& i=1,2,\ldots,n,\\
		0,& i=0.
	\end{array}
	\right.
\end{equation}

\subsection{Brief review on proximal ADMM}\label{review}

Consider the following convex minimization problem
\begin{equation}\label{th:convergence}
	\begin{array}{ll}
		\min\limits_{x, y}~ h(x) + g(y),~~
		\text{s.t.}~ Ax + By = c,
	\end{array}
\end{equation}
where  $h:\X \rightarrow (-\infty,+\infty]$ and $g:\Y \rightarrow (-\infty,+\infty]$ are closed proper convex functions, $A: \X \rightarrow \Z$ and $B: \Y \rightarrow \Z$ are linear operators, and $c\in \Z$ is a given vector.
A pair $(\bar x,\bar y)\in\X\times\Y$ is said to be a solution of (\ref{th:convergence}) if there exists a vector $\bar z\in\Z$ such that the following Karush-Kuhn-Tucker (KKT) system holds
\begin{equation}\label{kkt}
	0\in A^{\top} \bar z+\partial h(\bar x), \quad 0\in B^{\top} \bar z+\partial g(\bar y),\quad\mbox{and} \quad A\bar x + B\bar y = c.
\end{equation}

Given an initial pair  $(x^0, y^0, z^0) \in \X \times \Y \times \Z$, the proximal ADMM 
\cite{PADMM} for (\ref{th:convergence}) takes the following iterative framework
\begin{equation}\label{fadmm}
	\left\{
	\begin{array}{l}
		x^{j+1} = \argmin\limits_{x \in \X} h(x) + \langle z^j, Ax \rangle + \frac{{\mu}}{2} \lVert Ax + By^j - c \lVert ^2 + \frac{1}{2} \lVert x - x^j \lVert _S^2,\\
		y^{j+1} = \argmin\limits_{y \in \Y} g(y) +  \langle z^j, By \rangle + \frac{{\mu}}{2} \lVert Ax^{j+1} + By - c \lVert ^2 + \frac{1}{2} \lVert y - y^j \lVert _T^2,\\
		z^{j+1} = z^j + \tau \mu (Ax^{j+1}+ By^{j+1} -c),\\
	\end{array} \right.
\end{equation}
where $z^j$ plays the role of multiplier, ${\mu} > 0$ is a penalty parameter, $\tau \in (0,(1+\sqrt{5})/2)$ is a step length, $S$ and $T$ are two self-adjoint positive semi-definite linear operators.
Under some proper assumptions, the sequence $\{(x^j, y^j, z^j)\}$  generated by (\ref{fadmm}) converges to an accumulation point $(\bar x,\bar y, \bar z)$ such that $(\bar x, \bar y)$ is an optimal solution of  (\ref{th:convergence}). For more details on the proof, one can refer to  \cite[Theorem B]{PADMM}.

\section{Convexity characterization and applications}\label{ourm}
In this section, we will present the characterization method for convex object(s). Then, models for applications to object(s) segmentation with convex prior and
convex hull computation are proposed, which are summarized into a variational model in the next section to design a general algorithm framework.

In \cite{binary2021Luo}, the following theorem for characterization of a single convex object was presented.
\begin{theorem}\label{th:single}
	Let $o$ be the indicator function of a region $D \subset\mathbb{R}^2$. Then $D$ is convex if and only if
	\begin{equation}
		(b_r * {o})(1 - {o}) \leq \frac{1}{2}(1 - {o}), ~\forall r>0.
	\end{equation}
\end{theorem}
One can refer to \cite{2021Luo}
for the proof of Theorem \ref{th:single}.
For the case of multiple convex objects, it is easy to obtain the following corollary.

\begin{corollary}\label{th:mul}
	Let $u_i$ be the indicator function of region $D_i \subset \mathbb{R}^2$ for $i=1,\dots, P$. Let $u(x)= (u_0(x),u_1(x), u_2(x), \ldots, u_P(x))^\top$, where
	$u_0$ is the indicator function of the background $\Omega\backslash \bigcup_{i=1}^PD_i$.
	Then, $D_1, D_2, \ldots, D_P$ are convex if and only if
	\begin{equation}\label{cons2}
		\C_{r}(u):= (u - 1)\odot(b_{r}*u)- \frac{1}{2}(u - 1) \geq 0, ~\forall r>0.
	\end{equation}
\end{corollary}
\begin{proof}
	From Theorem \ref{th:single}, we know that the convexity of the $i$-th region $D_i$ can be characterized by
	\begin{align}\label{eq:ch1}
		(u_i - 1)(b_{r} * u_i) - \frac{1}{2}(u_i - 1) \geq 0.
	\end{align}
	Reformulating both sides of this inequality as a vector and using the definition of operations $\odot$ and $\ast$ defined in (\ref{convolution}), we can obtain the assertion of this result.
\end{proof}

\subsection{Image segmentation with convexity prior}
Image  segmentation  is  one  of  the  fundamental  topics  in  the  fields  of  image  processing  and  computer vision, which aims to partition an image domain into
several meaningful sub-regions.
Over the past decades, numerous approaches have been proposed in the literature.
These methods can be roughly categorized into model-based methods \cite{MS1989,CV2001} and (deep) learning based methods \cite{encoder-decoder,he2017mask}.
Although the learning based methods have achieved great success and are
the hot topics nowadays, this paper focuses on model-based methods for image segmentation with convexity prior.

For a given image $I(x)$ defined on $\Omega$,
we assume the foreground consists of $P$ regions of interest, named $D_i$ for $i=1,\ldots,P$,
and the background is denoted by  $D_0:=\Omega \backslash \bigcup_{i=1}^P D_i$.
Let $u_i$ be the indicator function of $D_i$ for $i=0,1,\cdots,P$, and $u=(u_0,u_1,\cdots,u_P)^\top$.
Therefore,
$u\in \hat{\mathbb{U}}=\{u=(u_0,u_1,\cdots,u_P)^\top| u_i \in \{0,1\}, \sum_{i=0}^{P}u_i=1\}$.
According to the results of Corollary  \ref{th:mul},
the segmentation model for multiple convex objects can be formulated as
\begin{equation}
	\begin{aligned} \label{eq:modu}
		\min\limits_{u\in\hat{\mathbb{U}}} \int_{\Omega}\sum_{i=0}^P f_i u_i dx +  \lambda\sum_{i = 0}^P  \L_{\sigma}(u_i)
		,~~\textrm{s.t.} ~\mathcal{C}_{r}({u}) \geq 0,
	\end{aligned}
\end{equation}
where
$f_i$ (will be given later) is to measure the similarity of $I$ in $D_i$,  and
$\L_{\sigma}(u_i)= \sqrt{\frac{\pi}{\sigma}} \int_{\Omega}u_i(x) \big[G_\sigma*(1-u_i)\big](x)dx$ is an approximation of the boundary length of $D_i$
\cite{31} with $G_\sigma(x)=\frac{1}{2\pi\sigma^2}\text{exp}\big(-\frac{\|x\|^2}{2\sigma ^2}\big)$  and $\sigma^2 \ll 1$. Here, $\lambda >0$ is a trade-off parameter to balance the two terms.

In real applications, labels on the objects and background are necessary to
alleviate the difficulties of segmentation and improve segmentation accuracy.
Denote the possible labels on $D_i$ by $R_i$ for $i=0,1,\cdots,P$. Define
$\delta_{i}=(\delta_{i0},\cdots,\delta_{iP})^\top$
and $\delta_{ij}(x)=1$ if $x\in R_i$ and $0$ if $x\in R_j(j\neq i)$ for $i,j=0,1,\cdots,P$
and
let $ f:=(f_0,f_1,\cdots,f_P)^\top$.
Then, we can modify and concise (\ref{eq:modu}) as
\begin{equation}
	\begin{aligned} \label{eq:modr}
		\min\limits_{ u\in\hat{\mathbb{U}}\bigcap \mathbb{C}_1}\langle  f, u\rangle + \lambda\mathcal{L}_\sigma(u),
		~~\textrm{s.t.}~ \mathcal{C}_r(u) \geq 0,
	\end{aligned}
\end{equation}
where $\langle f, u\rangle=\int_{\Omega}\sum_{i=0}^P f_i u_i dx,\mathcal{L}_\sigma(u)=
\sum_{i=0}^P\mathcal{L}_\sigma(u_i)$, and
$\mathbb{C}_1=\{u\in\mathbb{R}^{P+1}|u_i(x)=\delta_i(x),x\in R_i, i=0,1,2,\cdots, P\}$.

As for the similarity function $f_i$, we adopt a probability-based method used in \cite{Jun2013GMM}.
For $x$ in the domain of image $I$, let the probability of $x$ belonging to the region $D_i$ be $p_i(x|I)$. Then, the
similarity function $f_i$ is computed as
\begin{equation}\label{eq:f}
	f_i(x):=-\text{ln}p_i(x|I).
\end{equation}
There are many methods to estimate the probability $p_i(x|I)$ in the literature \cite{rother2004grabcut}.
Here, we employ the one in  \cite{binary2021Luo}, which is
detailed as follows.
According to the subscribed labels $R_i (i=0,1,\cdots,P)$ on the foreground and background,
the probability density function $p_i(x|I)$ is computed by
\begin{equation}\label{eq:pb}
	p_i(x|I)=\frac{\text{exp}(-d_i(x,I(x);R_i))}
	{\sum_{i=0}^{P}\text{exp}(-d_i(x,I(x);R_i))},
\end{equation}
where
\begin{equation}\label{eq:dist}
	d_i(x,I(x);R_i)=\left\{
	\begin{array}{ll}
		\sum_{y\in R_i}\|I(y)-I(x)\|, ~~~~~~~~~~~~~~~~~{i=0},\\[1mm]
		\sum_{y\in R_i}\big[\|I(y)-I(x)\| + \omega \|x-y\|^2\big],  {i=1,\ldots,P},
	\end{array} \right.
\end{equation}
in which $\omega>0$ is a given parameter.

\subsection{Convex hull computation of given set}
The convex hull of a set of points is the smallest convex set containing  the given points, which is one of the fundamental problems in computational geometry and computer vision.
Convex hull issues arise from various areas, such as data clustering \cite{Fuzzy_Systems}, robot motion planning \cite{H781966}, collision detection \cite{T8059847}, image segmentation \cite{condat2017convex}, and disease diagnosis \cite{zhang2009convex}.
There are a lot of methods to compute the exact convex hull of a given set, e.g.
quickhull method in \cite{barber1996quickhull, 1977ConvexHulls}.
However, few algorithms can tackle the case containing noises or outliers.
Variational method is one of the ways for this issue \cite{li2021new}.  Using the characterization method in Theorem \ref{th:single}, we propose a model to compute the convex hull of a given set $\mathbb{S}$. This method not only can compute the convex hull exactly but also can be extended to compute the convex hull of $\mathbb{S}$ with outliers.

Assume $\mathbb{S}$ is the given set. The exact convex hull of $\mathbb{S}$ is to seek a minimal convex region including $\mathbb{S}$. If the given set is noise free, the convex hull can be obtained by solving
\begin{equation}\label{eq:mod_convexhull}
	\begin{aligned}
		\min\limits_{u\in\hat{\mathbb{U}}\bigcap \mathbb{C}_2} \langle b,{u}\rangle,~~
		\textrm{s.t.}~ \mathcal{C}_{r}(u) \geq 0,
	\end{aligned}
\end{equation}
where $b(x)=(0,\mathbb{I}_\Omega(x))^{\top}$,
$\mathbb{C}_2=\{u=(u_0,u_1)^\top|u_1(x)=\mathbb{I}_\mathbb{S}(x),x\in \mathbb{S}\}$ guarantees the finding region includes the given set, and the objective function denotes the area $\int_\Omega u_1(x)dx$ of the finding region $\{x|u_1(x)=1\}$.
The constraint set $\mathbb{C}_2$ is to make the yielding convex set include the given set $\mathbb{S}$.

If the given set $\mathbb{S}$ is polluted by noise, we should compute a convex set to approximate the real convex hull, i.e. the convex set
is allowed to exclude some points in $\mathbb{S}$ that may be noises.
In order to exclude the noises in $\mathbb{S}$, the model (\ref{eq:mod_convexhull}) is modified as
\begin{equation}\label{eq:convexhull_noise}
	\min_{u\in\hat{\mathbb{U}}}~\langle{b},u\rangle + \gamma\int_{\Omega} (u_1-1)^2\mathbb{I}_{\mathbb{S}} dx + \lambda \mathcal{L}{_\sigma(u)},~~
	\textrm{s.t.}~\mathcal{C}_{r}(u)\geq0,
\end{equation}
where the inclusion constraint in (\ref{eq:mod_convexhull}) is moved in the objective function as a penalty, a boundary length regularization term $\mathcal{L}_\sigma(u)$ is introduced to eliminate outliers.
Let $\hat{b}=(0,\gamma\mathbb{I}_\mathbb{S})$. 
By the idempotence of binary function, the model (\ref{eq:convexhull_noise}) can be simplified as
\begin{equation}\label{eq:convexhull_noisec}
	\min_{u\in\hat{\mathbb{U}}}~\langle b-\hat{b},u\rangle +\lambda \mathcal{L}{_\sigma(u)}, ~\textrm{s.t.}~\mathcal{C}_{ r}({{u}}) \geq 0
\end{equation}.
\section{Numerical algorithm}\label{sec3}
This section is devoted to the numerical algorithms for the models of image segmentation (\ref{eq:modr}) and convex hull computation with (\ref{eq:convexhull_noisec}) and without (\ref{eq:mod_convexhull}) noise.
Firstly, we unify the concerned models as a general form
\begin{equation}
	\begin{aligned} \label{eq:mod_all}
		\min\limits_{u\in\hat{\mathbb{U}}\bigcap\mathbb{C}} ~\langle g, u\rangle + \lambda\sqrt{\frac{\pi}{\sigma}} {\langle} u,  G_\sigma *(1- u) {\rangle},~~
		\textrm{s.t.} ~ \mathcal{C}_{r}(u) \geq 0.
	\end{aligned}
\end{equation}
By choosing $\mathbb{C}=\mathbb{C}_1$ (resp. $ \mathbb{C}_2$ and $\mathbb{R}^{P+1}$) and $g=f$ (resp. ${b}$ and $b-\hat{b}$),
the model (\ref{eq:mod_all}) can cover (\ref{eq:modr}) (resp.  (\ref{eq:mod_convexhull}) and (\ref{eq:convexhull_noisec})).
Therefore, we need to present the algorithms for (\ref{eq:mod_all}) only.

\subsection{Linearization of (\ref{eq:mod_all})}
Let ${u}^k$ be an estimation of the minimizer to (\ref{eq:mod_all}).
By the first-order Taylor expansion  at ${u}^k$,
the boundary length term and the
convexity constraint term in
(\ref{eq:mod_all}) can be approximated by
\begin{align}
	\mathcal{L}_\sigma(u)&\cong
	\sqrt{\frac{\pi}{\sigma}}\langle {u}, G_\sigma*(1 - 2{u}^k) \rangle + \sqrt{\frac{\pi}{\sigma}}\langle {u}^k, G_\sigma* {u}^k \rangle,\label{eq:lin_length}\notag\\
	\mathcal{C}_{r}({{u}})
	&\cong {{u}}^k\odot  b_{r}*{u}^k + 2({{u}}-{u}^k)\odot b_{r}*{{u}}^k   - 1\odot  b_{r}*{{u}} - \frac{1}{2}{{u}} + \frac{1}{2}\notag\\
	&={A}_r^k u + \frac{1}{2}  - {{u}}^k\odot  b_{r}*{{u}}^k,\notag
\end{align}
where ${A}^k_r u:=2{ u}\odot b_r*{{u}}^k - \frac{1}{2}{u} - b_r*{u}$.
By ignoring the constant term $\sqrt{\frac{\pi}{\sigma}}\langle {u}^k, G_\sigma* {u}^k \rangle$, the model (\ref{eq:mod_all}) can be approximated by
\begin{equation}
	\begin{aligned} \label{eq:modL}
		\min\limits_{u\in\hat{\mathbb{U}}\bigcap \mathbb{C}}~\langle g,u\rangle +   {\lambda}\sqrt{\frac{\pi}{\sigma}} \big{\langle} {u}, G_\sigma*(1-2{u}^k) \big{\rangle},~~
		\text{s.t.} ~A_r^k {{u}} \geq c_r^k,
	\end{aligned}
\end{equation}
where
$c_r^k:= {u}^k\odot b_r*{u}^k - \frac{1}{2} $. In order to minimize (\ref{eq:modL}) efficiently, we only choose a few radial functions $b_{r_e}(e=1,2\cdots,E)$ for
the constraint in  (\ref{eq:modL}), the details and effectiveness of which will be discussed in the numerical experiments section. Let
$
{A}_k =[
\begin{matrix}
	{A}_{r_1}^k, {A}_{r_2}^k, \cdots, {A}_{r_E}^k
\end{matrix}
]^\top,
~~
{c}_k = [
\begin{matrix}
	c_{r_1}^k, c_{r_2}^k, \cdots,	c_{r_E}^k
\end{matrix}
]^{\top}.
$
Relaxing $\hat{\mathbb{U}}$ as
${\mathbb{U}}=\{u\in\mathbb{R}^{P+1}|u_i\in [0,1],\sum_{i=0}^Pu_i=1\}$, which is commonly used in the literature,
we can reformulate (\ref{eq:modL}) as
\begin{equation}\label{eq:modLL}
	\begin{aligned}
		\min\limits_{u\in \mathbb{U}\bigcap \mathbb{C}, v}~ \underbrace{
			\langle u, g+\lambda\sqrt{\frac{\pi}{\sigma}} G*(1-2{u}^k)\rangle}\limits_{F(u;{u}^k)}
		+\underbrace{\delta_{\mathbb{R}^{E(P+1)}_+}(v)}\limits_{G(v)}
		~\text{s.t.}~ A_ku -v= c_k,
	\end{aligned}
\end{equation}
where $v$ is called a remaining variable, $\delta_{\mathbb{R}^{E(P+1)}_+}(v)$ is
the characterization function of $\mathbb{R}^{E(P+1)}_+$, that is
$\delta_{\mathbb{R}^{E(P+1)}_+}(v)=0$ for $v\in {\mathbb{R}^{E(P+1)}_+}(x)$ and $+\infty$ otherwise, with $\mathbb{R}^{E(P+1)}_+$ denoting the first quadrant.
The minimization problem (\ref{eq:modLL}) is a convex optimization issue, which can be solved by efficient algorithms, e.g. the proximal ADMM.

\subsection{The proximal ADMM for (\ref{eq:modLL})}
This section is devoted to the proximal ADMM for  (\ref{eq:modLL}).
Using the proximal ADMM reviewed in (\ref{fadmm}),
we can obtain the following iteration procedure for (\ref{eq:modLL})
\begin{equation}\label{eq:subproblem}
	\begin{aligned}
		\left\{
		\begin{array}{l}
			u^{k,j+1}=\argmin\limits_{u\in \mathbb{U}\bigcap \mathbb{C}} F(u;{u}^k)+\langle  z^{k,j}, A_ku\rangle+\frac{\mu}{2}\lVert  A_ku -  v^{k,j} -  c\lVert_2^2
			+ \frac{1}{2}\lVert u-u^{k,j}\lVert_{S_k}^2\\
			v^{k,j+1}=\argmin\limits_{v}G(v) -  \langle z^{k,j}, v \rangle + \frac{{\mu}}{2} \lVert A_ku^{k,j+1} -v - c \lVert ^2 + \frac{1}{2} \lVert v - v^{k,j} \lVert_{T_k}^2\\
			z^{k,j+1}= z^{k,j}+\tau\mu(  A_k u^{k,j+1} -  v^{k,j+1} -  c),
		\end{array}
		\right.\notag
	\end{aligned}
\end{equation}
where $\tau\in (0, (1+\sqrt{5})/2),\mu>0$, and $S_k$ and $T_k$ are self-adjoint positive semi-definite matrices defined subsequently.
We can obtain closed-form solution by
choosing proper $S_k$ and $T_k$.
Firstly, for fixing $v^{k,j}$ and $z^{k,j}$,
let $S_k:=\alpha_k I-\mu A_k^\top A_k$,  where $\alpha_k>0$ such that $S_k$ is positive and semi-definite.
Then, the subproblem of $u$ update in (\ref{eq:subproblem}) can be simplified
\begin{align*}
	u^{k,j+1}
	=&\argmin_{u} \ \delta_{\mathbb{U}\bigcap \mathbb{C}}(u)+\frac{\alpha_k}{2}\lVert u- \xi^{k,j} \lVert_2^2=\Pi_{\mathbb{U}\bigcap \mathbb{C}}( \xi^{k,j}),\notag
\end{align*}
where
$\xi^{k,j} := \frac{1}{\alpha_k} \big(\mu A_ku^{k} +(\alpha_k I -\mu A_k^\top A_k)u^{k,j} -  g- \lambda\sqrt{\frac{\pi}{\sigma}} G*(1-2{u}^k) \big).$
The projection can be done elementwise, and we have
\begin{align}
	u^{k,j+1}(x)=&\left\{\begin{array}{ll}
		\delta_i(x), x\in R_i,i=0,1,\cdots,P,\\
		\text{Proj}_{\mathbb{U}}(\xi^{k,j}(x)), \text{otherwise}.
	\end{array}\right.
\end{align}
For the projection on $\mathbb{U}$, we need to seek $t^\ast$ such that the summation of $\max\{\xi^{k,j}-t^\ast,0\}$  is $1$
\cite[Proposition 28.26]{Combettes2017Convex}.
Here, we adopt the method in \cite{2022Projections}  to determine $t^\ast$ by sorting the given vector.

Secondly, by choosing $T_k:=0$,
the update of $v$ has the following form
\begin{equation}
	\begin{aligned}
		v^{k,j+1}
		=&	\argmin_{v}\delta_{\mathbb{R}^{E(P+1)}_+}(v) + \frac{\mu}{2}\lVert  v -  A_k u^{k,j+1} +  c - { z^{k,j}}/\mu\lVert_2^2 \\
		=&\max\{0, A_k {u}^{k,j+1} -  c_k+ { z^{k,j}}/\mu\}.
	\end{aligned}
\end{equation}
Using the latest $u^{k,j+1}$ and $v^{k,j+1}$, the multiplier $z$ is then updated immediately.
The convergence of the iteration procedure is guaranteed for selections of $\alpha_k, S_k$ and $T_k$,
because it is the semi-proximal ADMM  \cite{PADMM} reviewed in Subsection \ref{review}. One may refer to \cite[Theorem B]{PADMM}  for more details on its proof.
\begin{theorem}\label{assu}
	Let $\{(u^{k,j},v^{k,j},z^{k,j})\}$ be the sequence generated by Algorithm \ref{alg:PADMM} from an initial point $\{(u^{k,0},v^{k,0},z^{k,0})\}$. Assume that $\tau\in(0,(1+\sqrt{5})/2)$ and $\alpha_k$ is chosen such that $S_k$ is positive semi-definite. Then, the sequence $\{(u^{k,j},v^{k,j},z^{k,j})\}$ converges to a accumulation point $\{(\hat{u}^{k}, \hat{v}^{k}, \hat{z}^{k})\}$, and $(\hat{u}^{k}, \hat{v}^{k})$ is the solution of the primal
	problem (\ref{eq:modLL}).
\end{theorem}

In light of above analysis, the proximal ADMM for the linearized problem (\ref{eq:modLL}) can be
summarized as Algorithm \ref{alg:PADMM}.
Because the proximal ADMM is for the linearization sub-problem (\ref{eq:modLL}) of (\ref{eq:mod_all}) only, it is terminated after a few rather than large iterations (for convergence) to save computation cost.
\begin{algorithm}[!t]
	\caption{(Proximal ADMM for problem (\ref{eq:modLL}))}\label{alg:PADMM}
	\begin{itemize}
		\item[Step 0.] Choose positive scalars $\mu,\alpha_k>0$ and $\tau\in(0,(1+\sqrt{5})/2)$ such that $S_k$ is positive and semi-definite. Set $(u^{k,0}, v^{k,0}, z^{k,0})=(u^{k},0,0)$ , $j=0$ and maximal iteration number $J$.
		\item[Step 1.] While $j<J$:
		\begin{itemize}
			\item[-] Compute $\xi^{k,j}$ and $u^{k,j+1}(x) =\Pi_{\mathbb{U}\bigcap\mathbb{C}}(\xi^{k,j}(x)))$;
			\item[-] Compute $v^{k,j+1} = \max\{0, A_k {u}^{k,j+1} -  c+ {z^{k,j}}/\mu\}$;
			\item[-] Compute $z^{k,j+1} = z^{k,j} +\tau\mu(A {u}^{k,j+1}-v^{k,j+1}-c_k)$;
			\item[-] $j=j+1$.
		\end{itemize}
		\item[Step 2.] Output $u^{k,J}$.
	\end{itemize}
\end{algorithm}

\subsection{Algorithm for (\ref{eq:mod_all})} \label{subsection:algorithm}
It is very easy to obtain the algorithm (Algorithm \ref{alg:proposed}) for (\ref{eq:mod_all}) after presenting algorithm \ref{alg:PADMM} for the linearization problem (\ref{eq:modLL}). Here some details are discussed.
\begin{algorithm}[!t]
	\caption{(Algorithm for problem (\ref{eq:mod_all}))}\label{alg:proposed}
	\begin{itemize}
		\item[Step 0.]
		Input some necessary data according to the concerned problems (labels for segmentation and sets for convex hull computation),
		set $\lambda$ and $\sigma$ for the boundary length term, tolerance $\epsilon$, $\theta$ and $r_0$ narrow belt  determination, $\omega$ for computing the probability density function, choose
		the radial functions $b_r$ for convexity constraint, and  let $k=0$.
		\item[Step 1.] Determine $g$ according to the issues,  image segmentation or convex hull computation
		\item[Step 2.] Initialize  $u^0$, and set $err=1>\epsilon$.
		\item[Step 3.] While $err>\epsilon$, do the following steps:
		\begin{itemize}
			
			\item[step 3.1] Update ${u}^{k}$ on the narrow belt $S_\theta$ using the output of Algorithm \ref{alg:PADMM} to obtain $u^{k+1}$;
			\item[step 3.2]  $k:= k+1$;
			\item[step 3.3] Compute $err=\|{u}^{k}-{u}^{k-100}\|/\|{u}^{k-100}\|$  if $\mod(k,100)=0$.
		\end{itemize}
		\item[Step 4.] Output result $u^k$.
		\item[Step 5.] Interaction: Go to step 2 after given additional labels for image segmentation issue if $u^k$ is not desired.
	\end{itemize}
\end{algorithm}

\emph{Initialization methods: }
We adopt different initialization methods for
the segmentation and convex hull issues.
For image segmentation problem (\ref{eq:modr}), the initial estimate $u^0_j$ is set as the convex hull of the labels on the $j$-th foregrounds ($j=1,2,\cdots,P$), and $u^0_0=1-\sum_{j=1}^Pu_j^0$. On the other hand, for the issues of convex hull computation (i.e. $P=1$), the initialization of $u_1^0$ is just the indicator function of the given set, and $u_0^0=1-u_1^0$ in (\ref{eq:mod_convexhull})(\ref{eq:convexhull_noisec}).

\emph{Update in a narrow belt}:
We should note that the proposed convex constraint is essentially a local method, which is only effective near the boundary(ies) of the objects. In order to avoid the influences of the intensities far from the boundary(ies) and improve the stability of the algorithm, we only update $u^k$ in a narrow belt of the boundary(ies): $S_\theta=\{x||u(x)-b_{r_0}*u(x)|<\theta\}$ where $r_0,\theta>0$ are two user-specified parameters.

\emph{Termination criterion:}
As for the termination of Algorithm \ref{alg:proposed},
we stop the iteration  when the relative varying between the estimates at two iterations is smaller than a given tolerance. In order to avoid early stop and improve stability, the relative error are computed every 100 iterations.

\emph{Interactive procedure:}
Obviously, it is very hard to draw proper labels
one shot for complicate images to yield
desired results for segmentation issue.
Therefore, it is necessary sometimes to
subscribe additional labels and run the
code repeatedly if the yielding results are not
desired. So, interactive procedure is combined in Algorithm \ref{alg:proposed}.

\section{Numerical experiments}\label{sec4}
In this section, we conduct a series of experiments
for image segmentation and convex hull computation to verify the effectiveness of the proposed models  and the efficiency of the  algorithms.
All experiments are run on a PC with Inter(R) Core(TM) i7-6700 CPU @ 3.40GHz 3.40GHz and 8.00GB memory using Matlab2016a.

All the parameters in Algorithms \ref{alg:PADMM} and \ref{alg:proposed} and the proposed models are set the following default values if they are not specifically given.
In Algorithm \ref{alg:PADMM}, we select penalty parameter $\mu=1$ and $\alpha_k=1+\mu$ to obtain the self-adjoint positive semi-definite matrix $S_k$, $\tau=1$ for step length, $J=5000$ for maximal iteration number. As for initialization of $u^0$, it has been discussed in Subsection \ref{subsection:algorithm}. In Algorithm \ref{alg:proposed}, for image segmentation problem, we let
$\lambda=2$ in model (\ref{eq:mod_all})  and $\sigma=0.01$ for boundary length penalization, $\theta=0.1$ and $r_0=3$ are used for the determination of the belt $S_{\theta}$, $\omega=0.1$ in  (\ref{eq:dist}) for the probability computation, $\epsilon=10^{-3}$ for the tolerance.
As for the radial function $b_r$, $b_r(x)=0$ if $\|x\|>r$, and $b_r(x)=1/{\pi r^2}$ if $\|x\|\leq r$
is selected in all experiments.
In addition,
only three radial functions are used for $k$-th step with radii being
$\text{mod}(l+1,14),\text{mod}(l+3,14),\text{mod}(l+5,14)$ in Algorithm \ref{alg:proposed}, where
and  $l:=2\text{mod}(\lfloor{k/100}\rfloor,7)$ mod denotes the remainder operation.

Shape-distance between
the output result and the corresponding ground truth,
which is also known as the Hausdorff distance  \cite{li2021new, 2022Luo},
is computed to measure the accuracy.
Let $D^\prime$ be the estimate of a shape $D$ (ground truth) for a given problem. Then, the shape-distance between $D^\prime$ and $D$ is defined as
\begin{equation}\label{disdd}
	\text{S-dist} (D^{'},D):=\frac{\max\left\{\max\limits_{x\in D} \min\limits_{y\in D^\prime}  \|x-y\|_2, \max\limits_{y\in D^\prime} \min\limits_{x\in D} \|y-x\|_2\right\}}{2\sqrt{Area(D)/\pi}}.
\end{equation}

\subsection{Influence of the choices of $r$s}\label{exp:inf}
It is obvious that the choices of radii $r$s in the model (\ref{eq:mod_all}) are very important to keep the convexity of the results.
In order to investigate its influence and provide a guidance for real applications, we do some
experiments on an image with different  choices of $r$s, i.e.  $ \{1, 2, 3, 4, 5\}$,
$ \{1, 4, 7, 10, 13\}$, or  $ \{ 1, 12, 23, 34, 45\}$,
and display the results from left to right
in Fig. \ref{fig:rt}, where some parts of interest are zoomed.
Observing the images in Fig. \ref{fig:rt},
we find that the smaller
$r$s (say $ \{1, 2, 3, 4, 5\} $)
cannot eliminate nonconvexity of smooth boundary if the curvature is near zero.
The larger $r$s (say $ \{1, 12, 23, 34, 45\}$) cannot suppress
small oscillations of boundary.
But the moderate values (say $ \{1, 4, 7, 10, 13\}$) can
overcome these defects, and have the ability to yield convex objects.
Besides, the running time for these three  cases are
about $7s$, $15s$, and $87s$, respectively. This is reasonable because a larger $r$ will result in higher computational cost in the convolution operation.
\begin{figure}[!t]
	\centering
	\includegraphics[width=3cm]{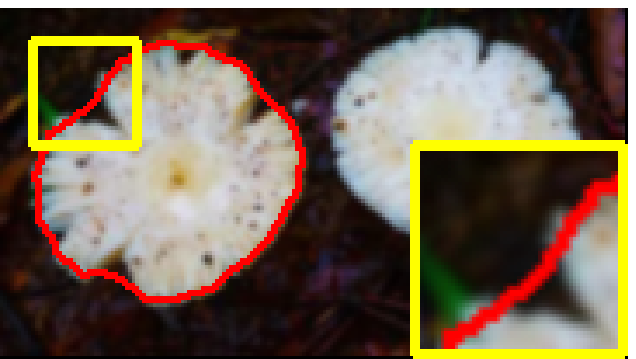}
	\includegraphics[width=3cm]{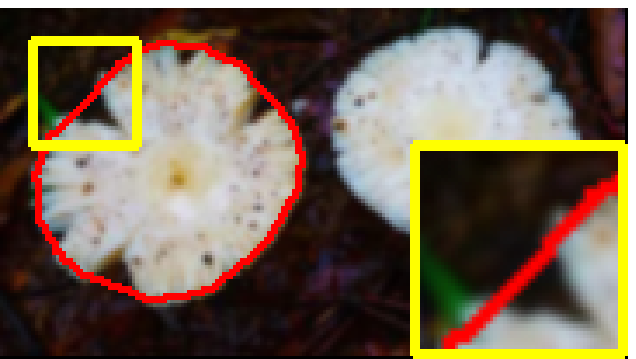}
	\includegraphics[width=3cm]{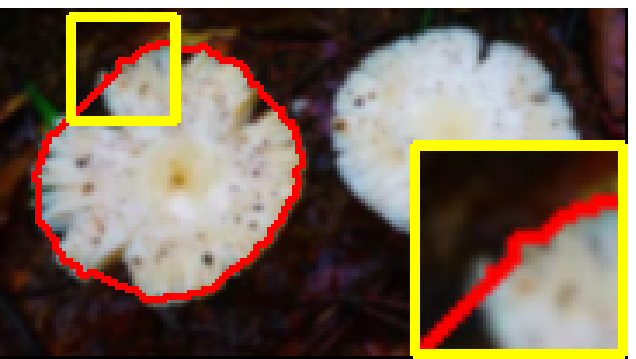}
	\caption{The radius configurations for the results from left to right are $ \{1, 2, 3, 4, 5\}$, $ \{1, 4, 7, 10, 13\}$ and $ \{1, 12, 23, 34, 45\}$.}
	\label{fig:rt}
\end{figure}

Based on the observations above,
we choose the aforementioned radii of radial functions for the experiments in this paper to keep the convexity of results and save computational cost.

\subsection{Image segmentation}\label{exp:seg}
The proposed method is performed on images with single and multiple convex objects directly. Then, our method is applied to ring shape with convex boundary prior using two steps, i.e. extract inner and outer boundary separately. At last, experiments on
complicate images using interactive procedure are demonstrated.

The proposed method is compared with the 1-0-1 method \cite{Gorelick2016convexity}
and level set function (abbreviated as LS) method \cite{2022Luo}.
We compare the proposed method with
1-0-1 method with $5\times5$ and $11\times 11$ stencils and penalty parameter 10, denoted by 101(5) and 101(11), respectively.
One can refer to  \cite{Gorelick2016convexity} for more details.

\subsubsection{Single object segmentation}
In this part, experiments on an image set consisting of $32$ images are conducted.
These 32 images are chosen from the data set at
\url{http://vision.csd.uwo.ca/code/}.
The superiority of the proposed method to the 1-0-1 and LS methods is demonstrated visually and
quantitatively (see Fig. \ref{fig:SegImg} and Tab. \ref{tab:shdist_single}).
\begin{figure*}[!t]
	\centering
	\includegraphics[height=1.4cm, width=1.4cm]{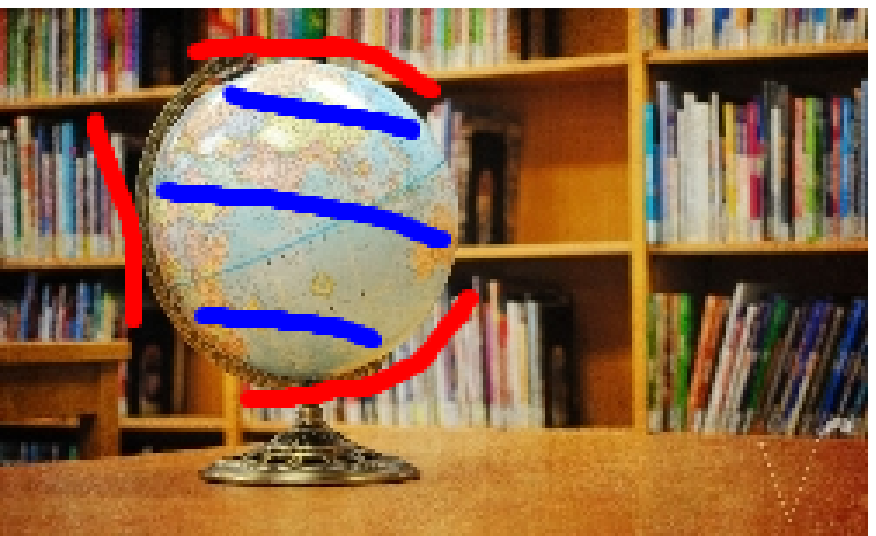}
	\includegraphics[height=1.4cm, width=1.4cm]{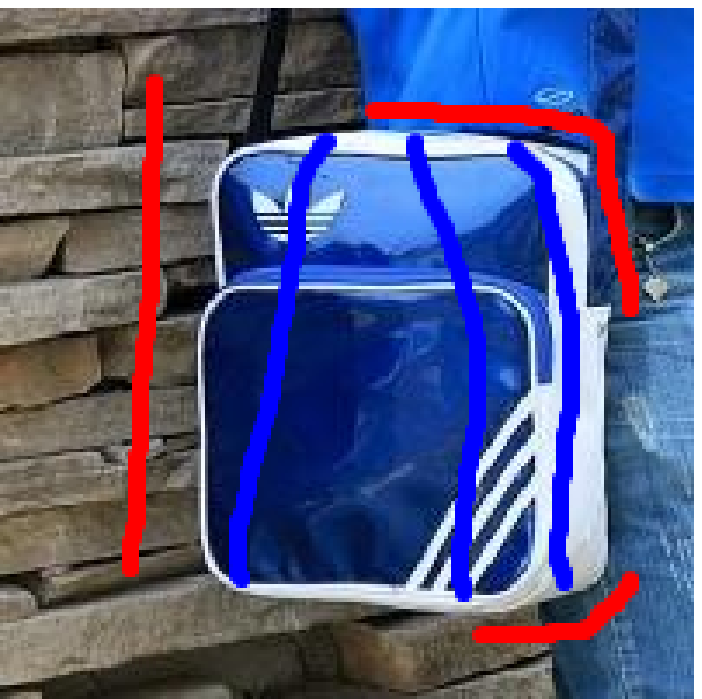}
	\includegraphics[height=1.4cm, width=1.4cm]{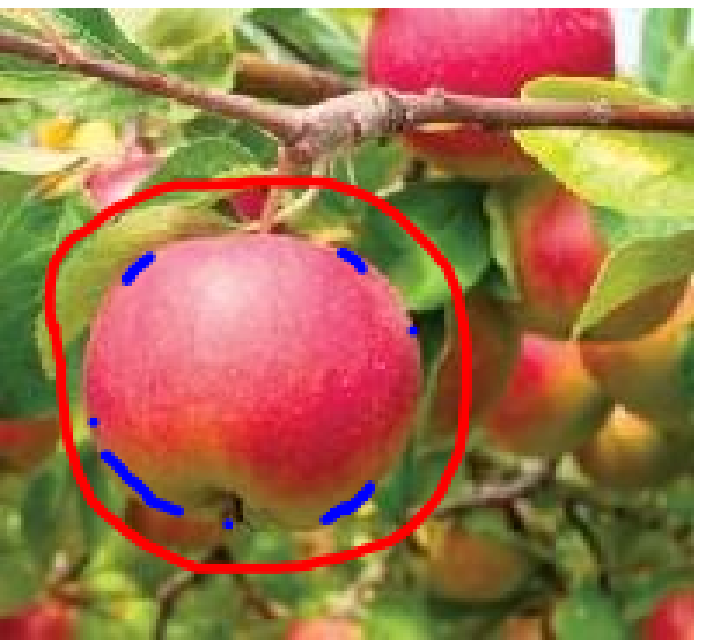}
	\includegraphics[height=1.4cm, width=1.4cm]{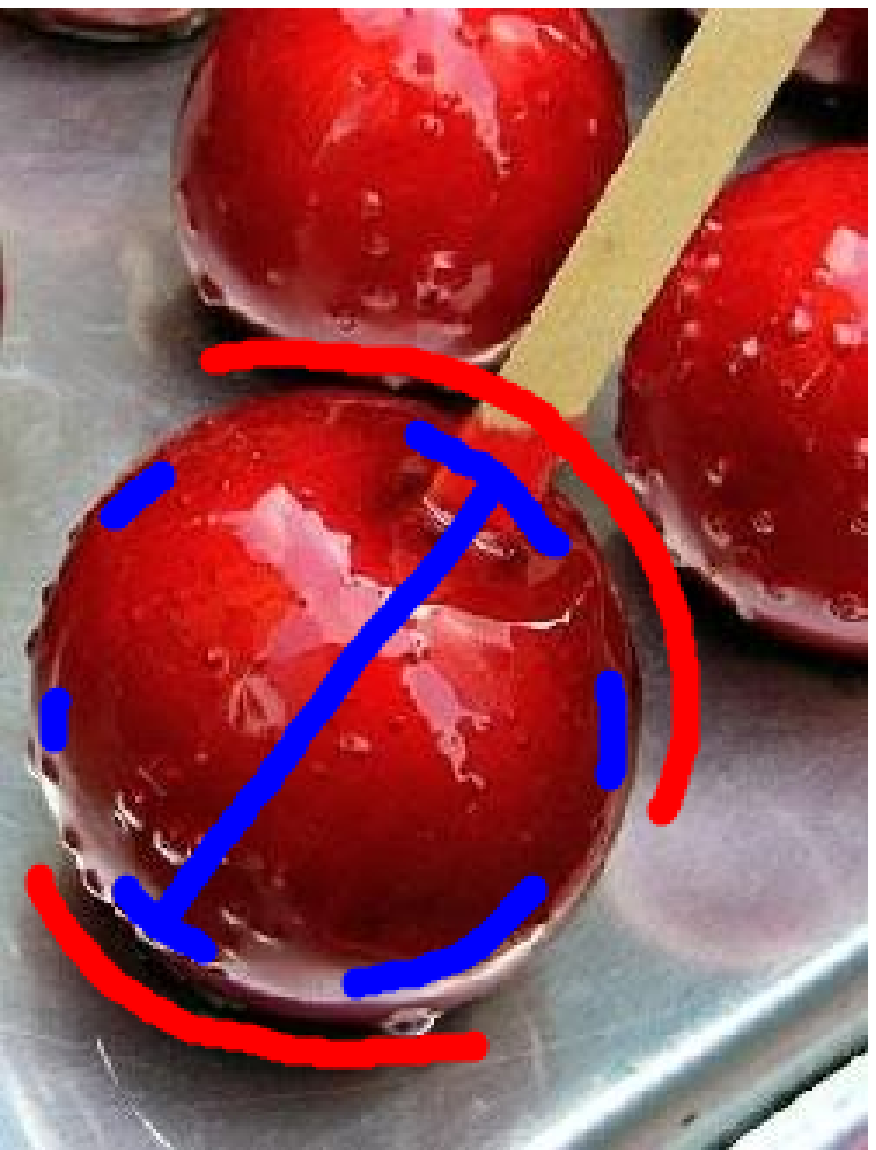}
	\includegraphics[height=1.4cm, width=1.4cm]{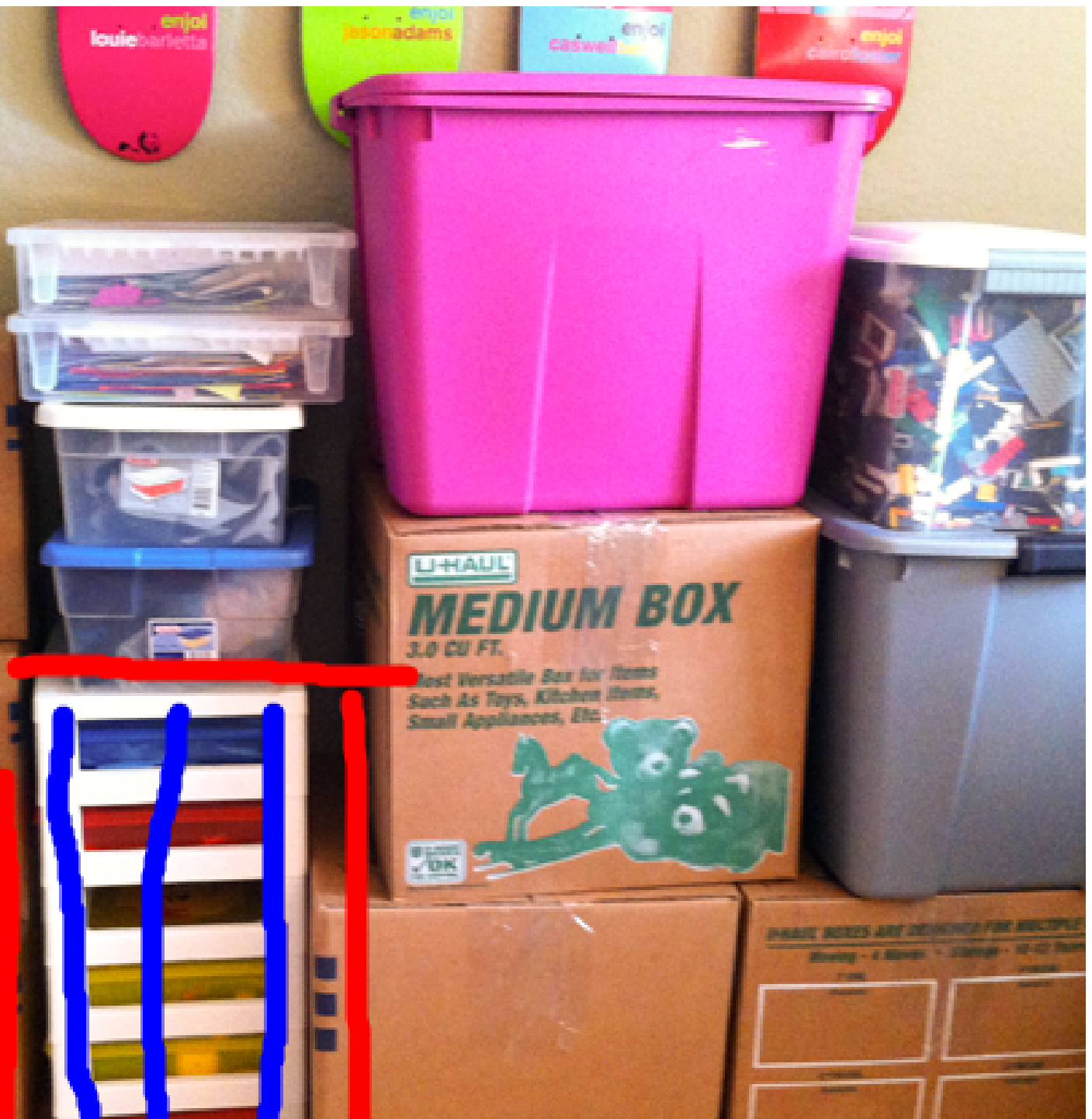}
	\includegraphics[height=1.4cm, width=1.4cm]{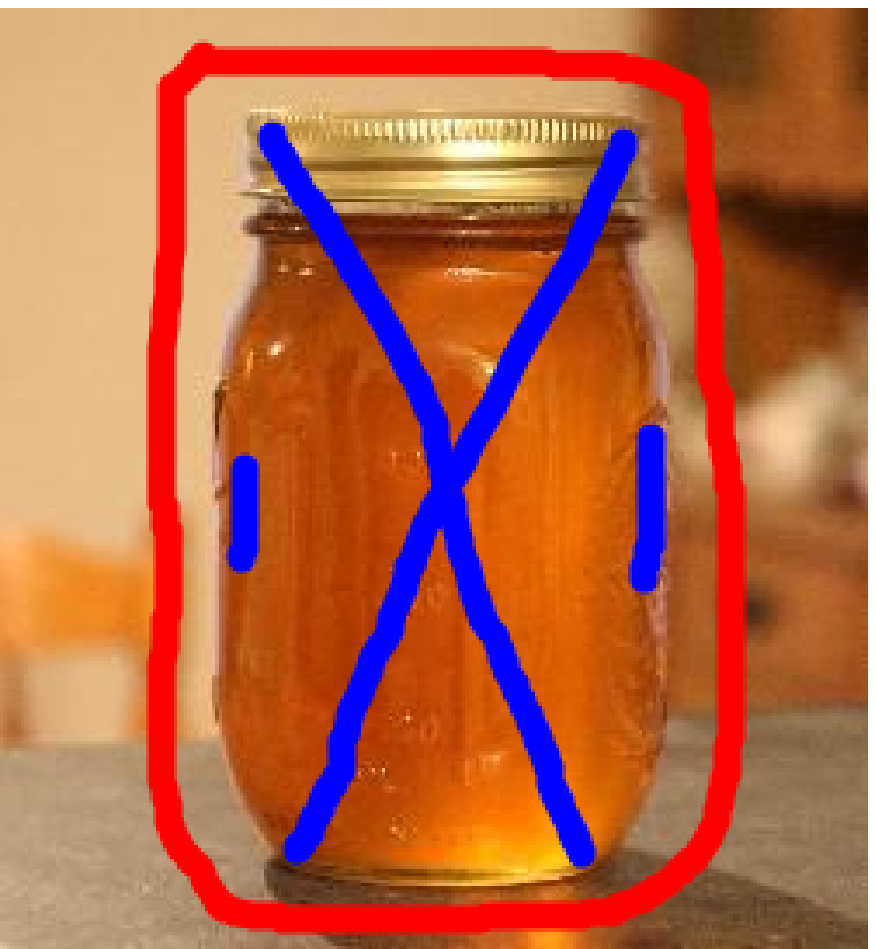}
	\includegraphics[height=1.4cm, width=1.4cm]{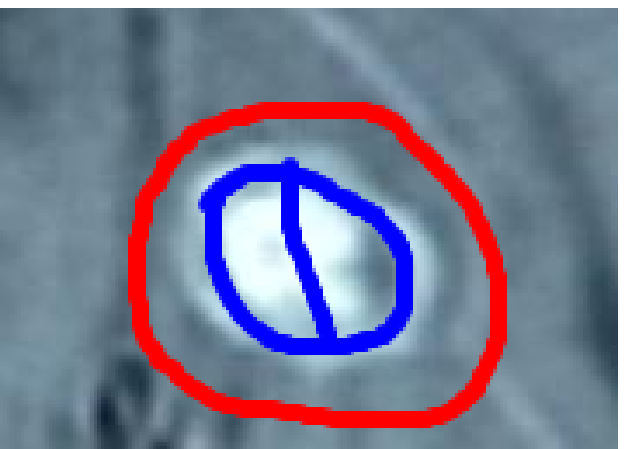}
	\includegraphics[height=1.4cm, width=1.4cm]{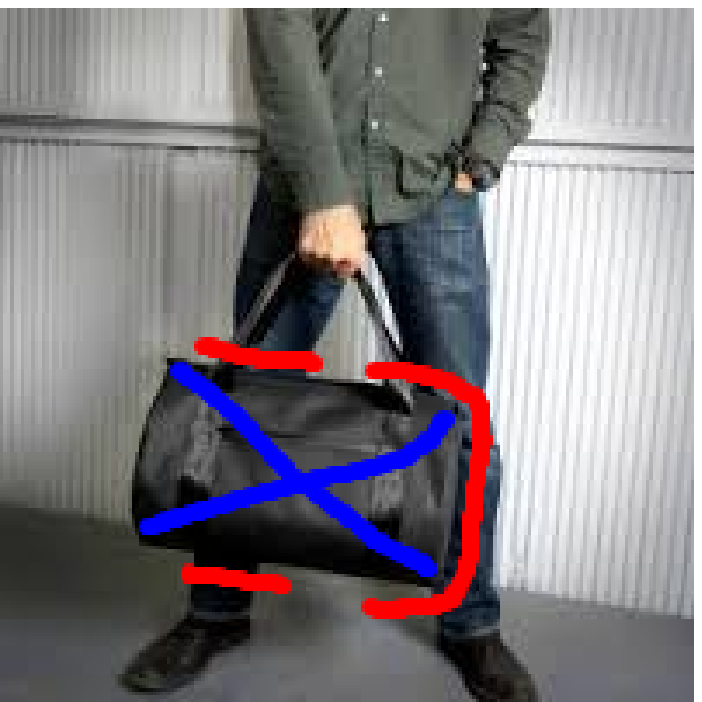}
	\includegraphics[height=1.4cm, width=1.4cm]{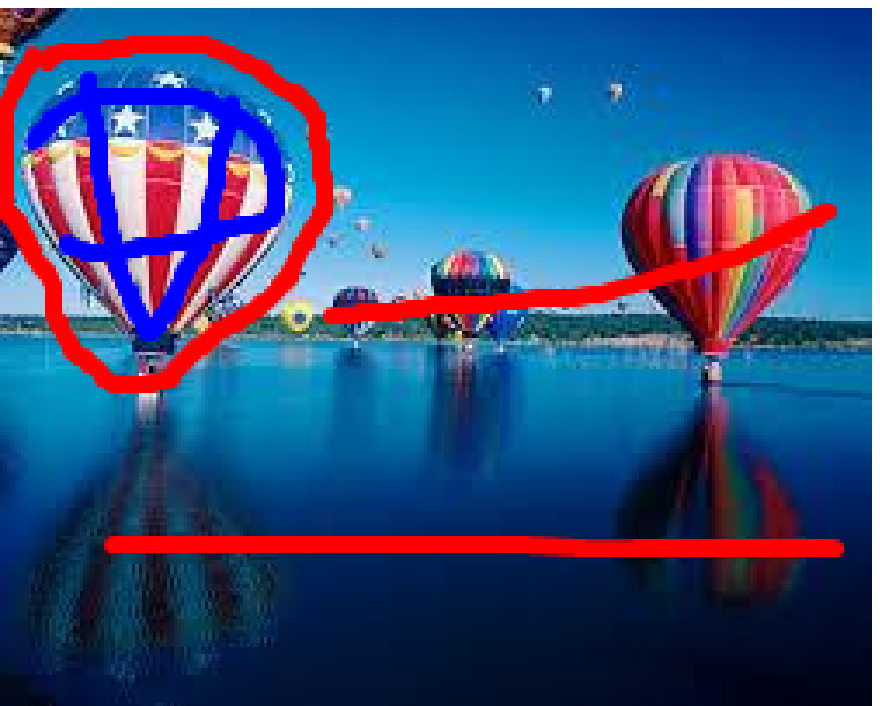}
	\includegraphics[height=1.4cm, width=1.4cm]{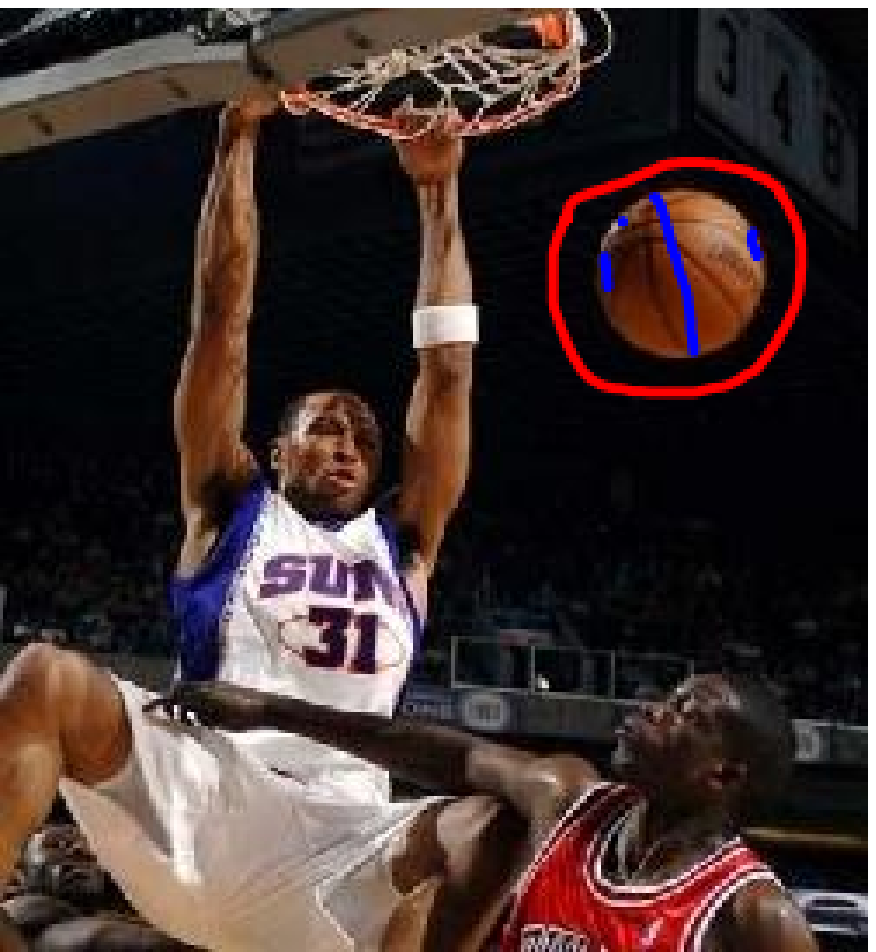}
	\includegraphics[height=1.4cm, width=1.4cm]{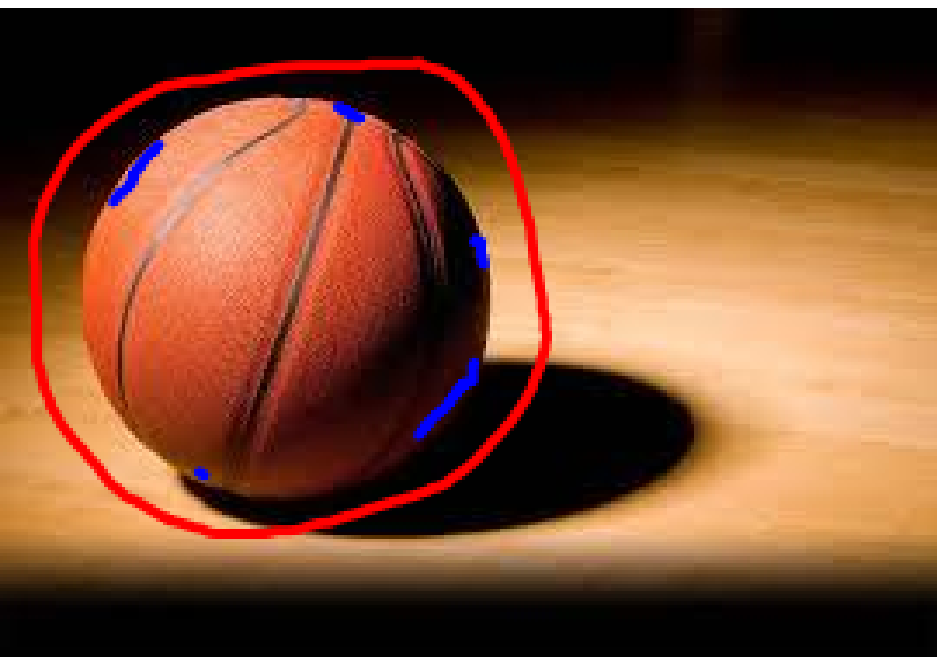}\\
	\includegraphics[height=1.4cm, width=1.4cm]{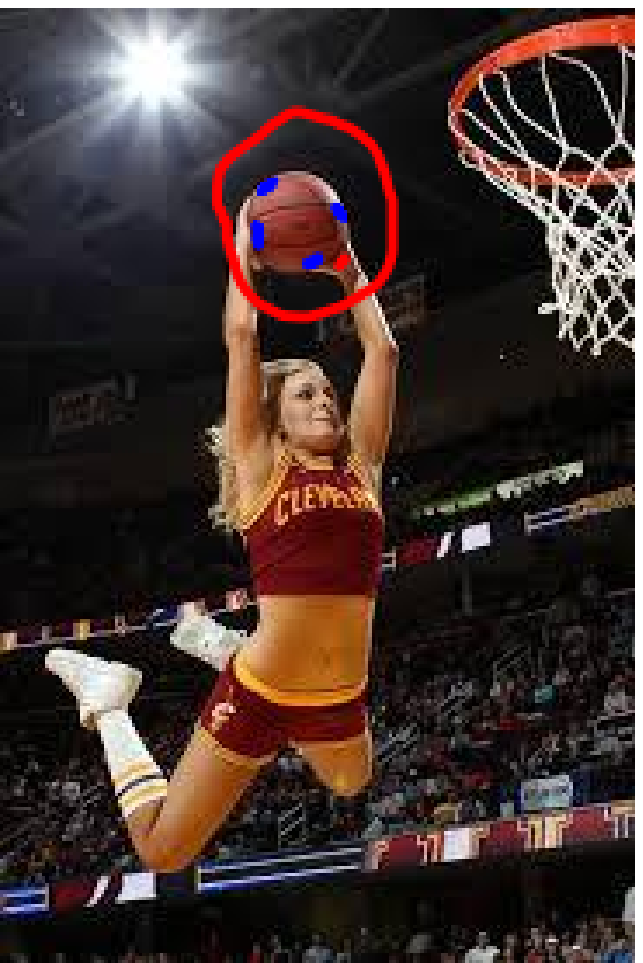}
	\includegraphics[height=1.4cm, width=1.4cm]{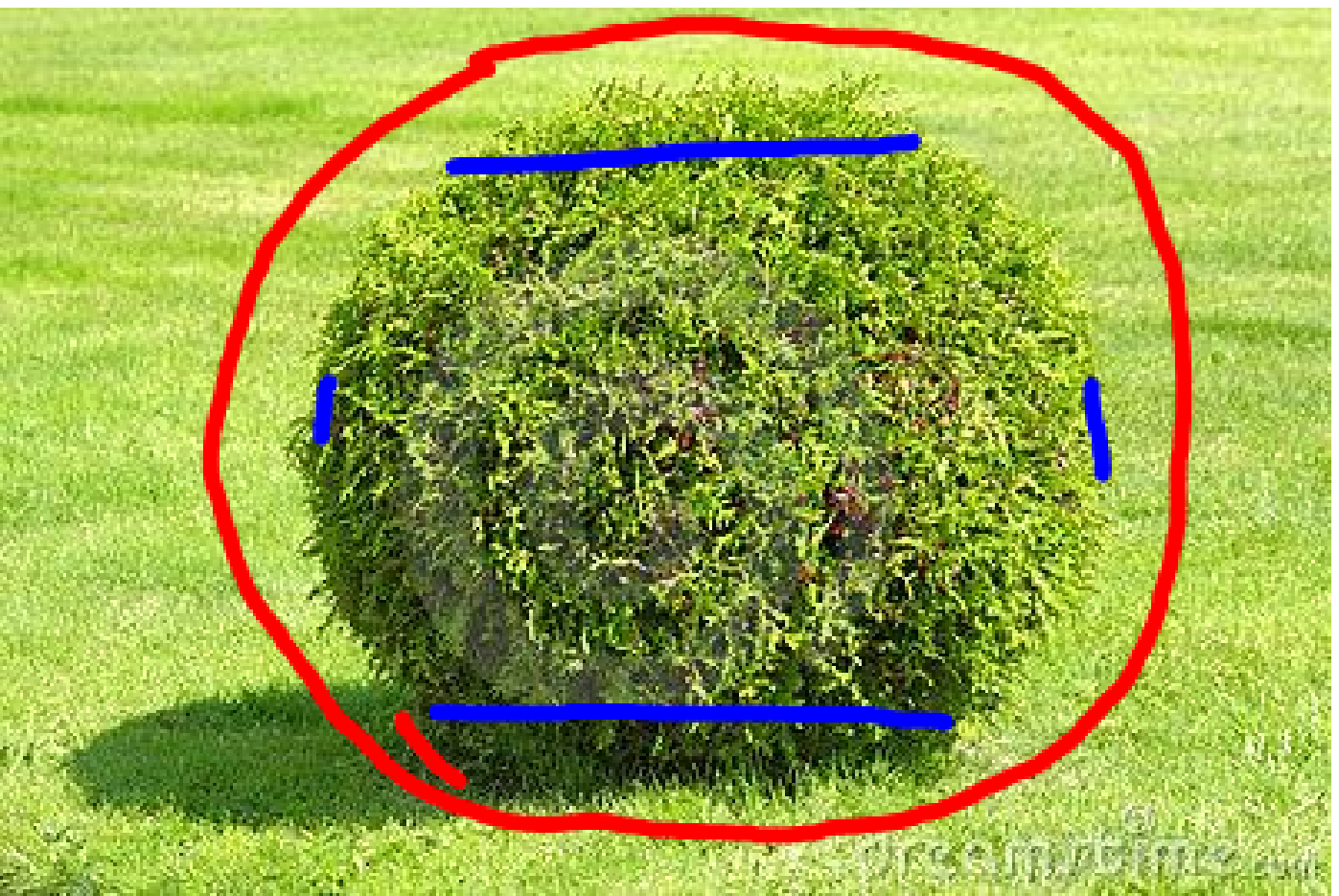}
	\includegraphics[height=1.4cm, width=1.4cm]{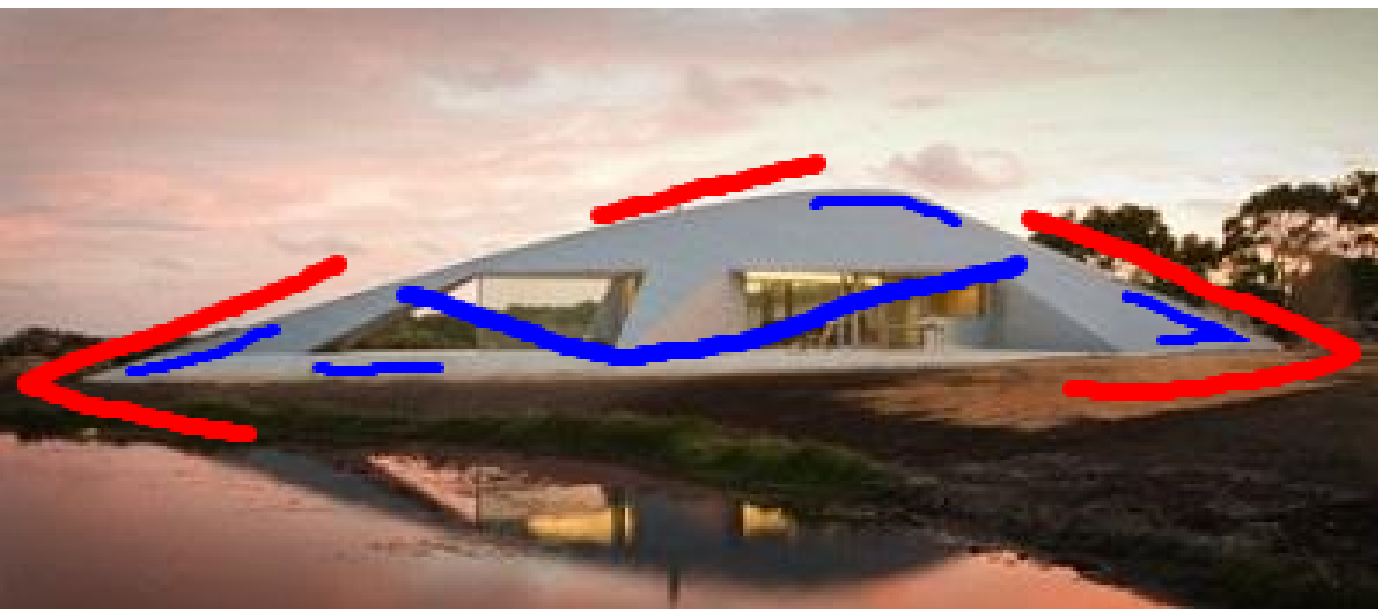}
	\includegraphics[height=1.4cm, width=1.4cm]{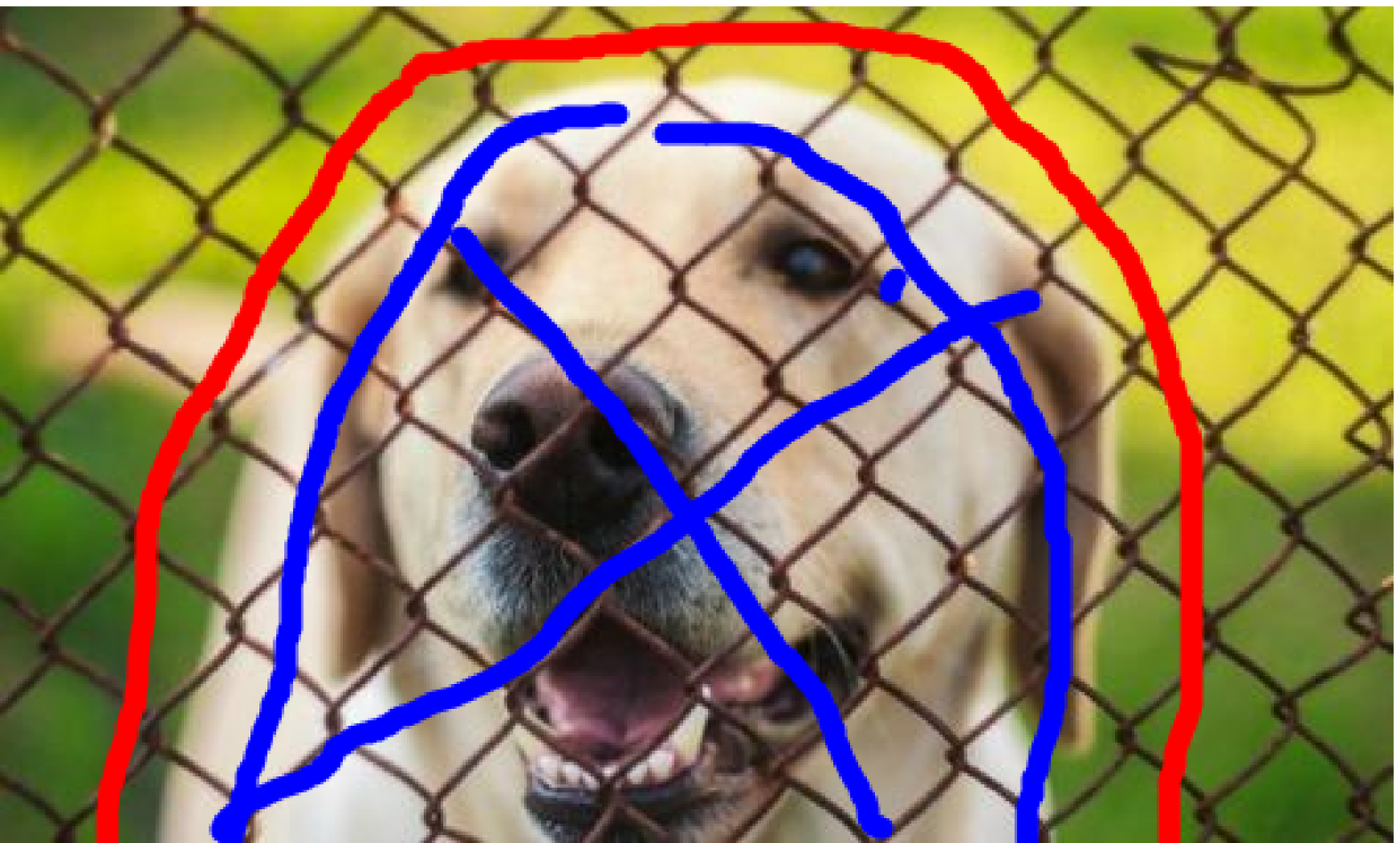}
	\includegraphics[height=1.4cm, width=1.4cm]{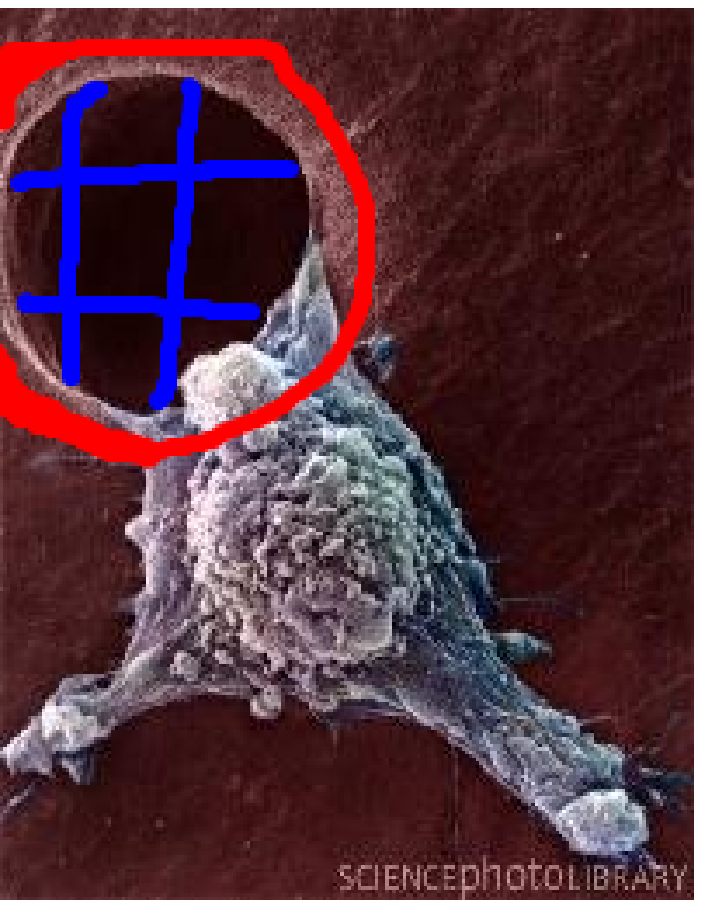}
	\includegraphics[height=1.4cm, width=1.4cm]{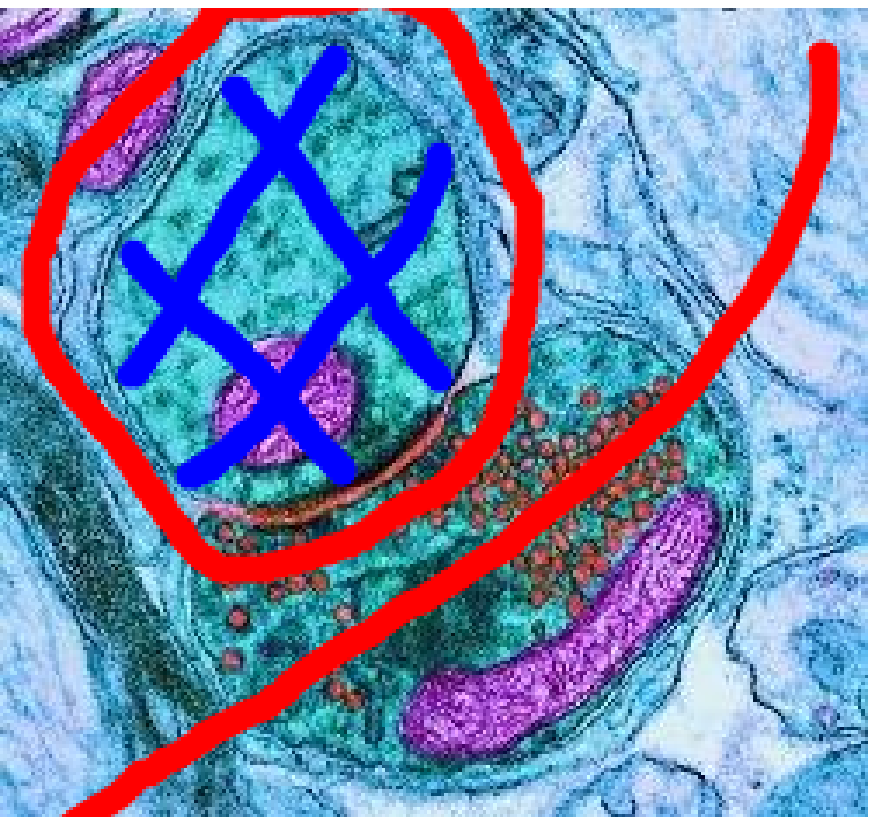}
	\includegraphics[height=1.4cm, width=1.4cm]{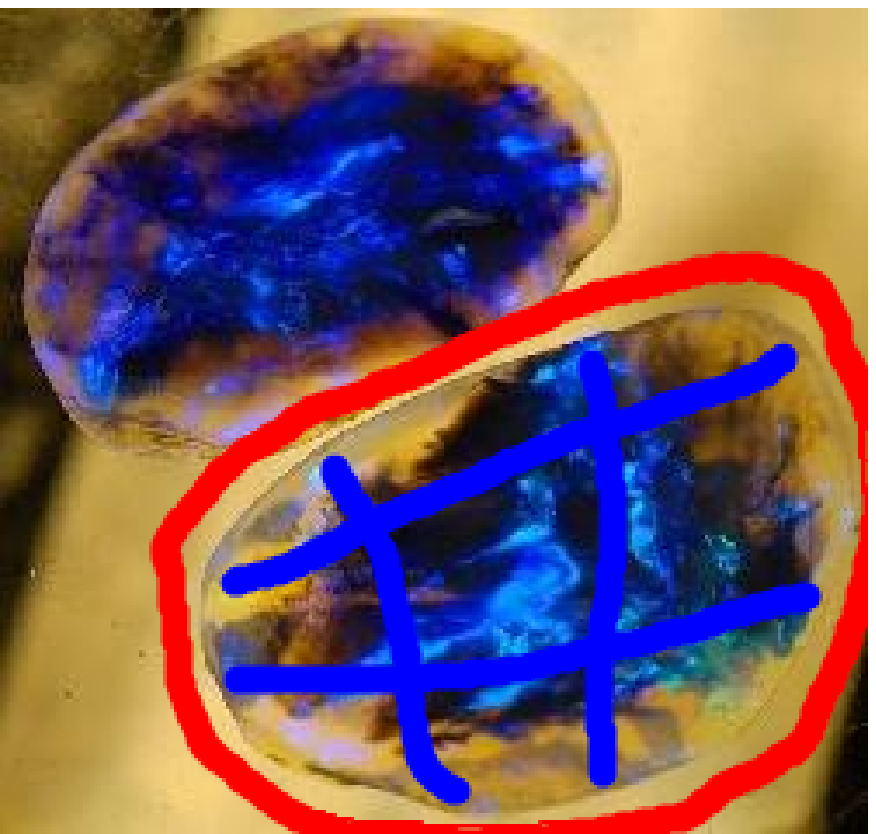}
	\includegraphics[height=1.4cm, width=1.4cm]{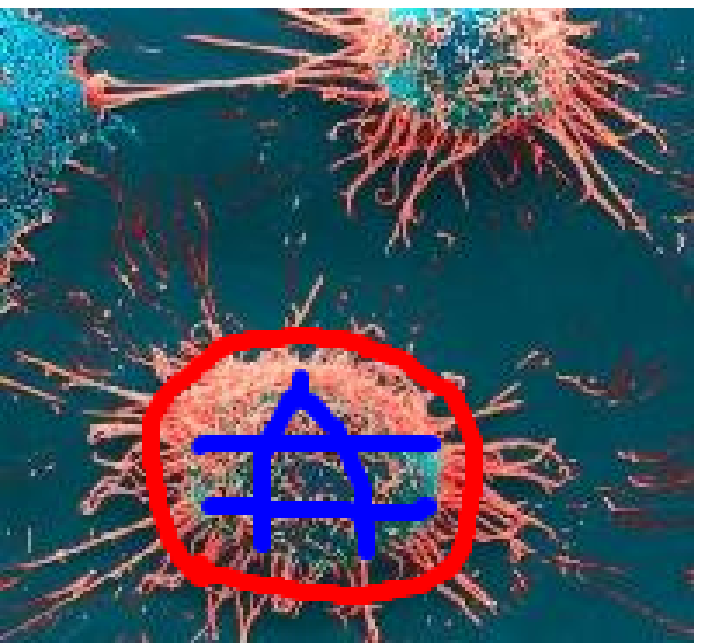}
	\includegraphics[height=1.4cm, width=1.4cm]{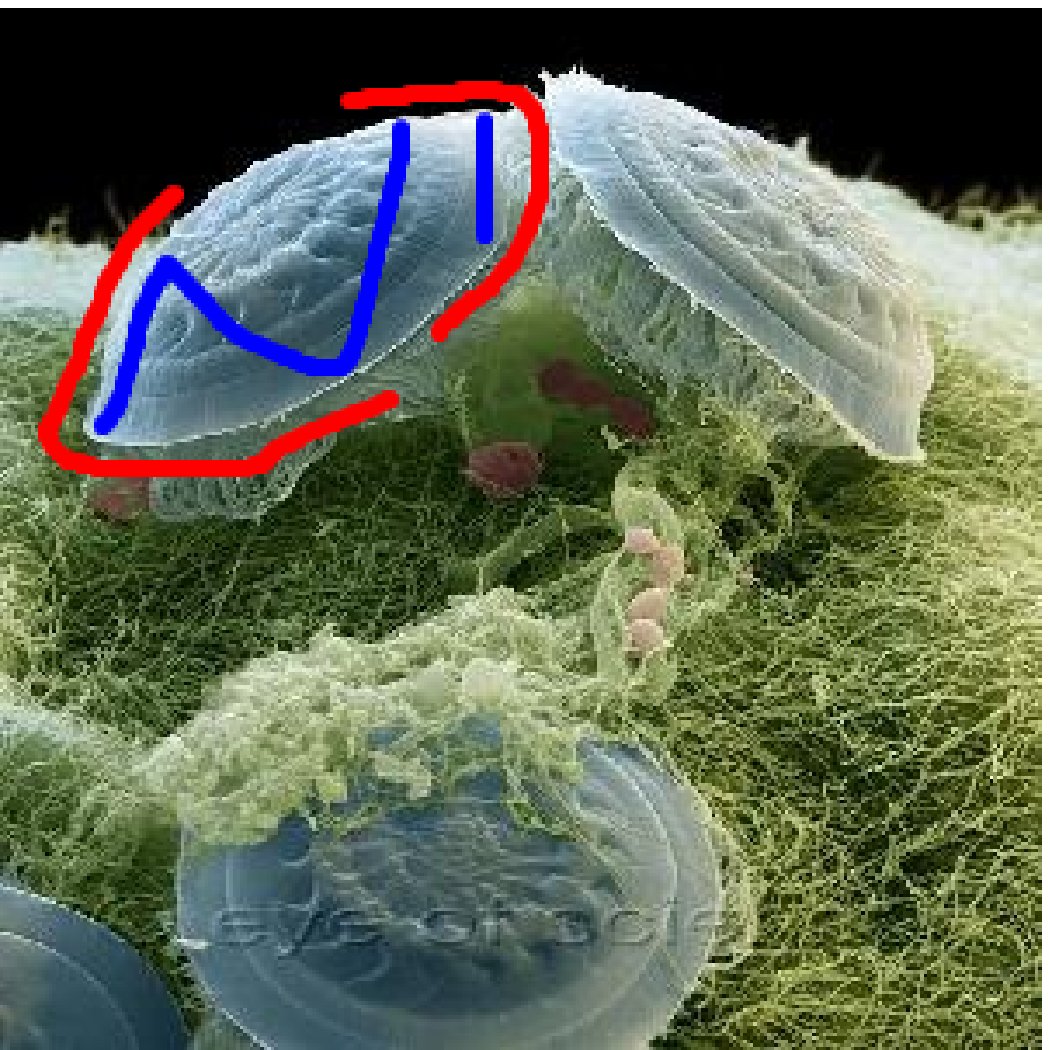}
	\includegraphics[height=1.4cm, width=1.4cm]{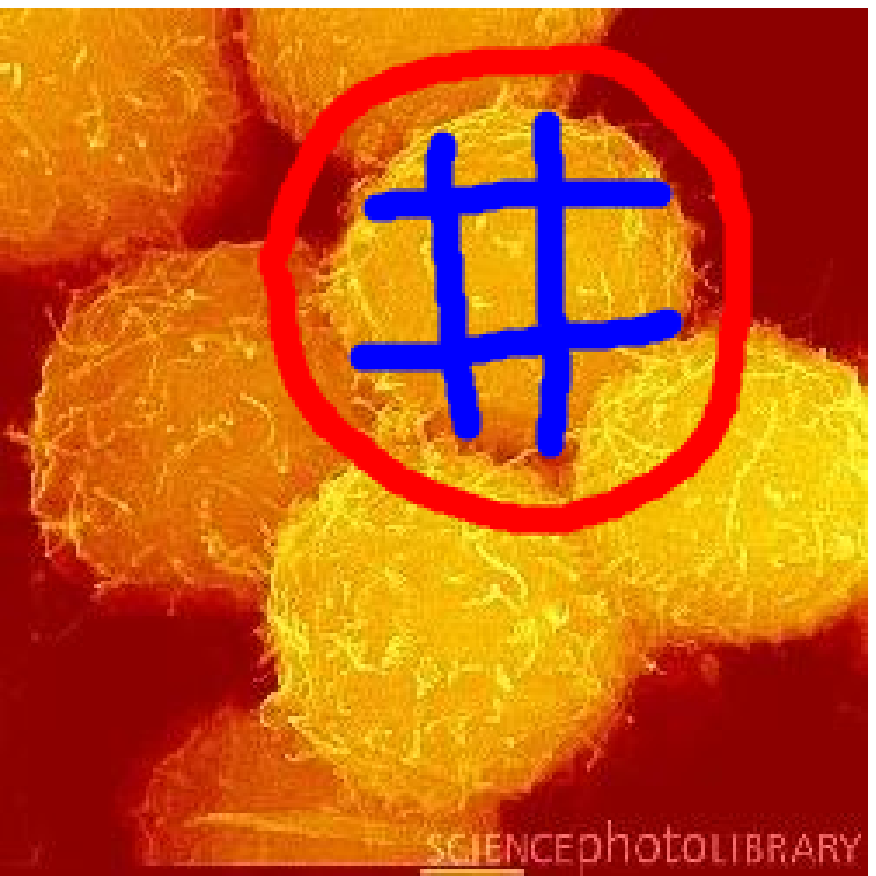}
	\includegraphics[height=1.4cm, width=1.4cm]{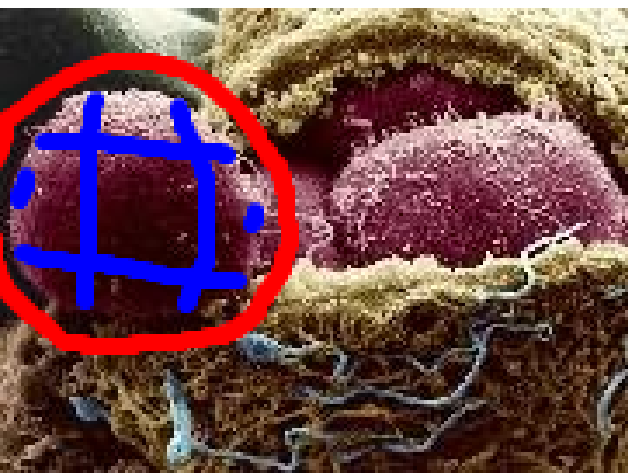}\\
	\includegraphics[height=1.4cm, width=1.4cm]{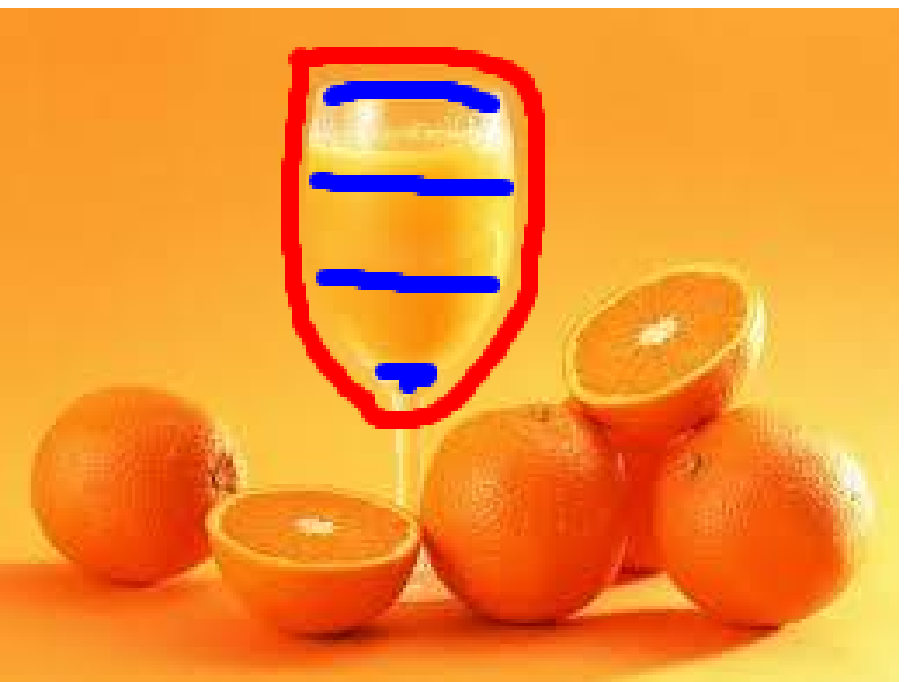}
	\includegraphics[height=1.4cm, width=1.4cm]{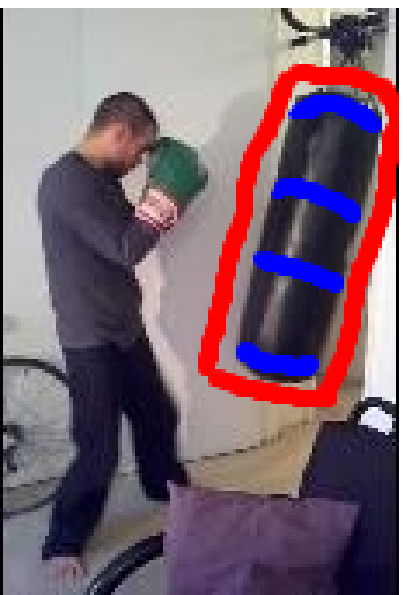}
	\includegraphics[height=1.4cm, width=1.4cm]{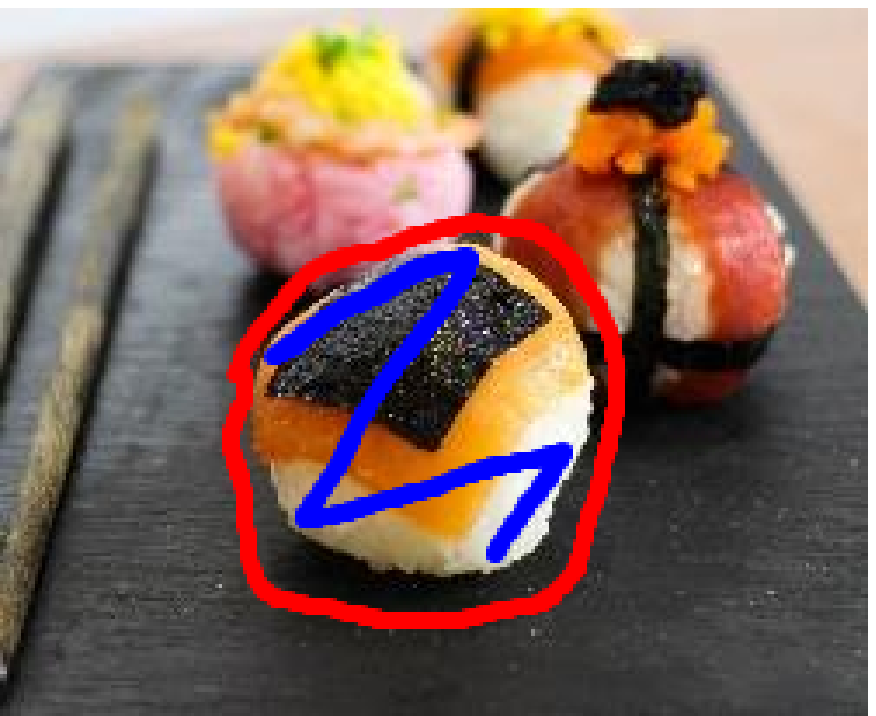}
	\includegraphics[height=1.4cm, width=1.4cm]{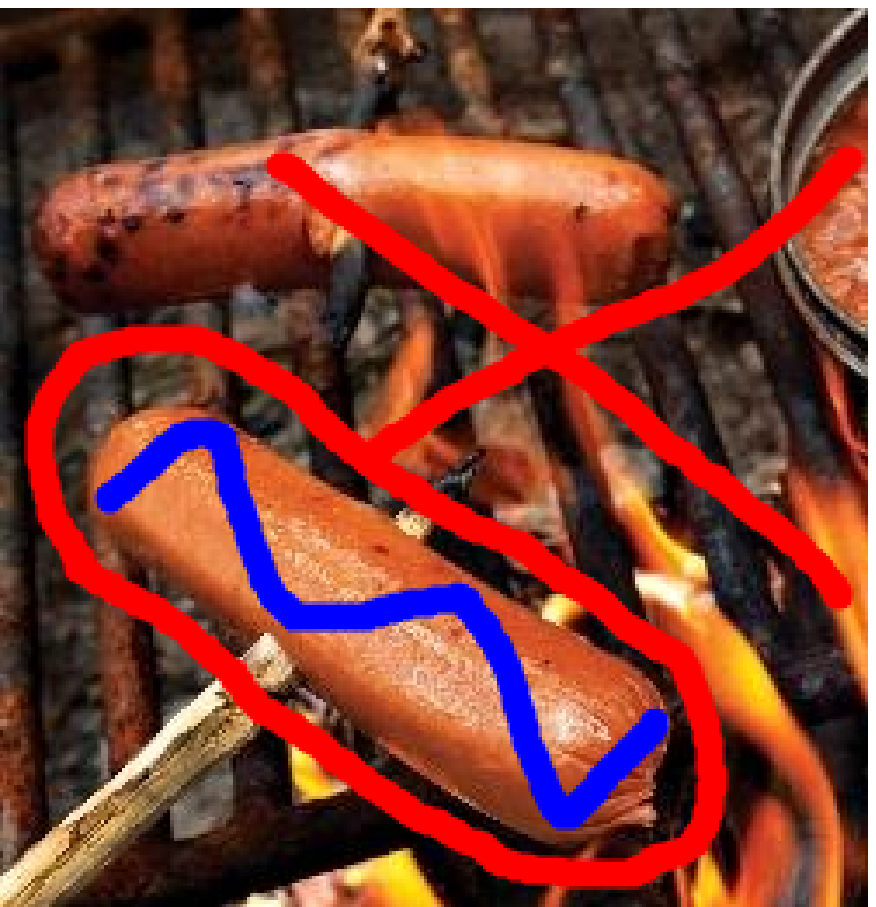}
	\includegraphics[height=1.4cm, width=1.4cm]{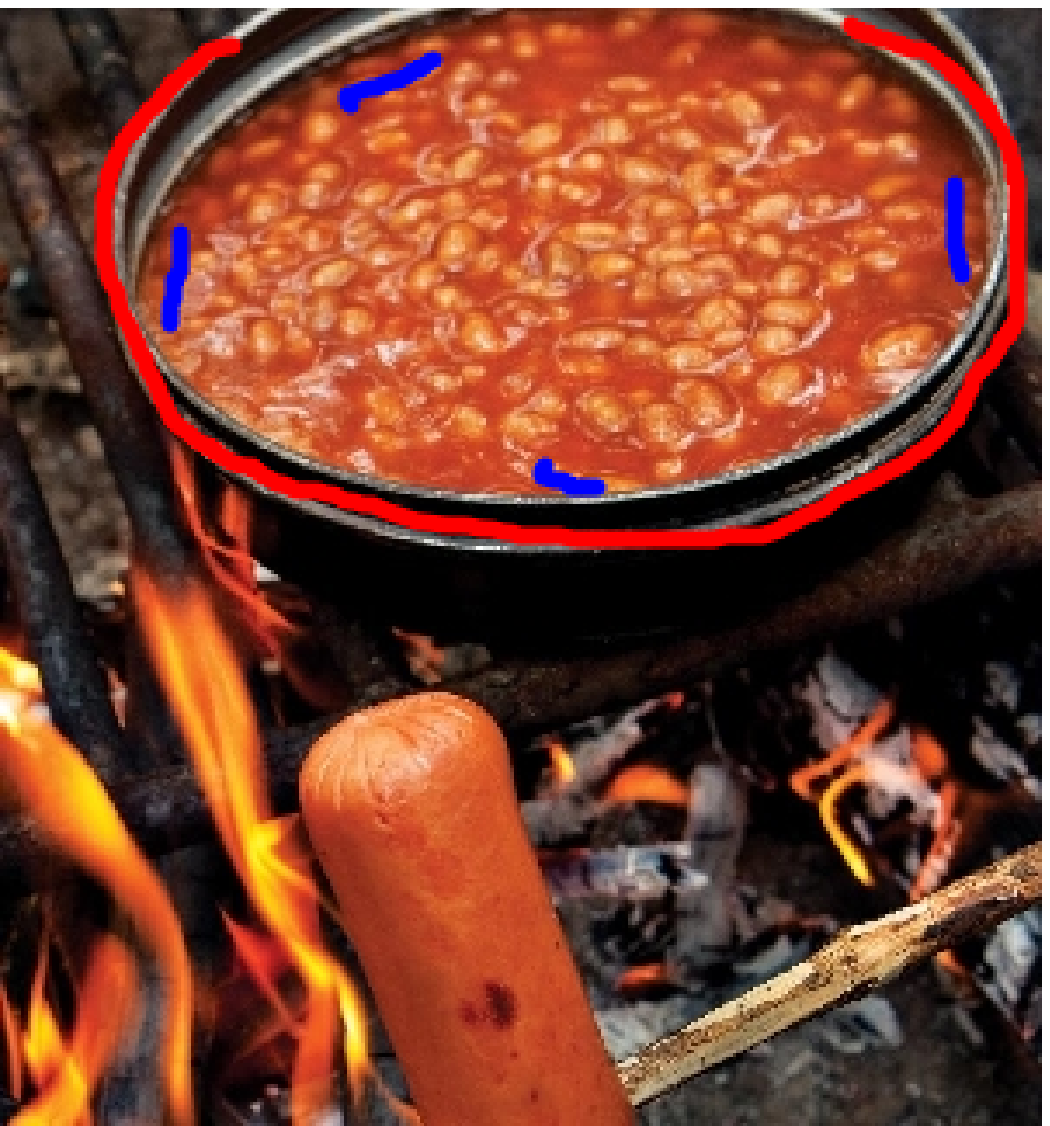}
	\includegraphics[height=1.4cm, width=1.4cm]{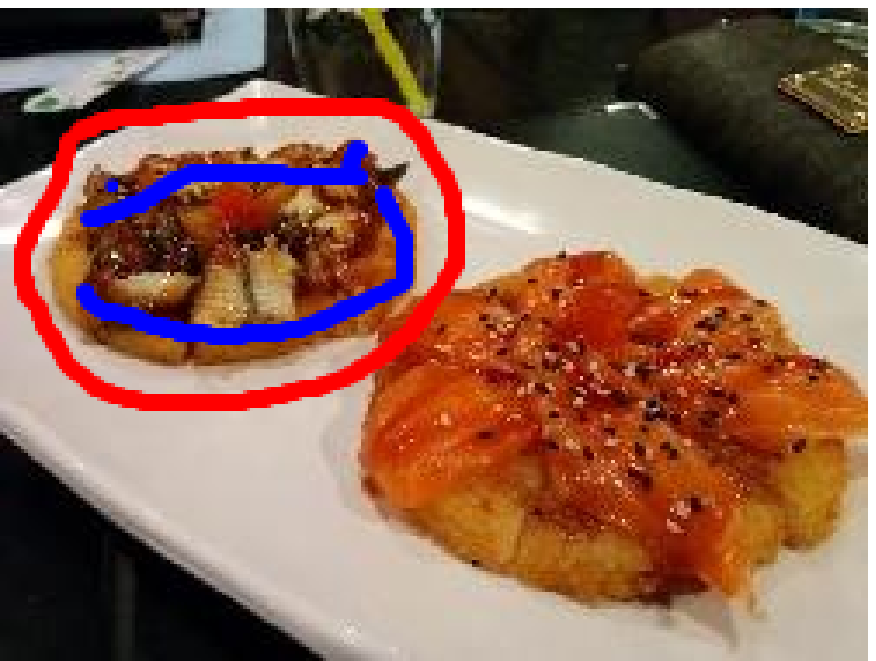}
	\includegraphics[height=1.4cm, width=1.4cm]{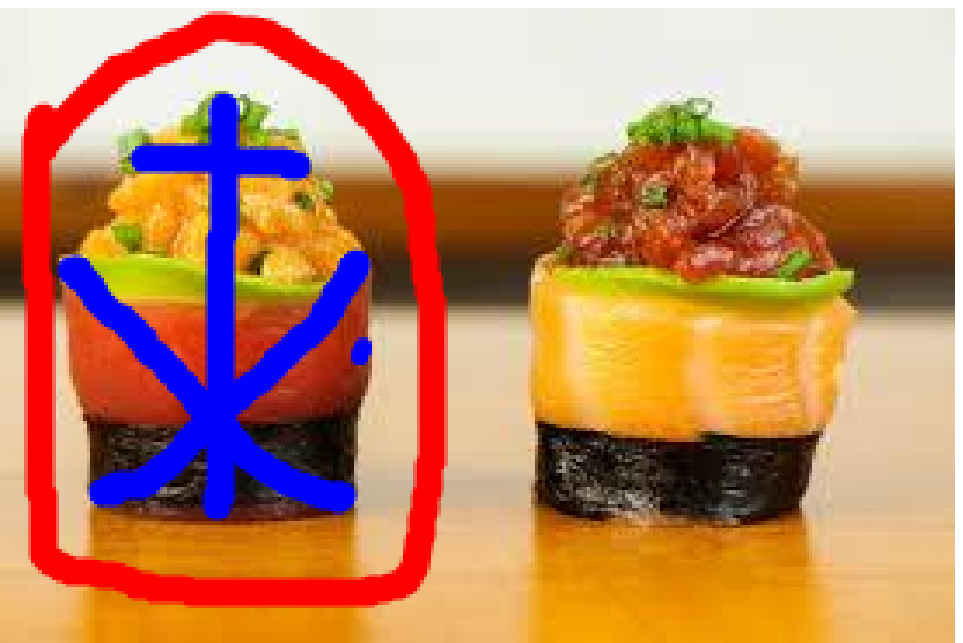}
	\includegraphics[height=1.4cm, width=1.4cm]{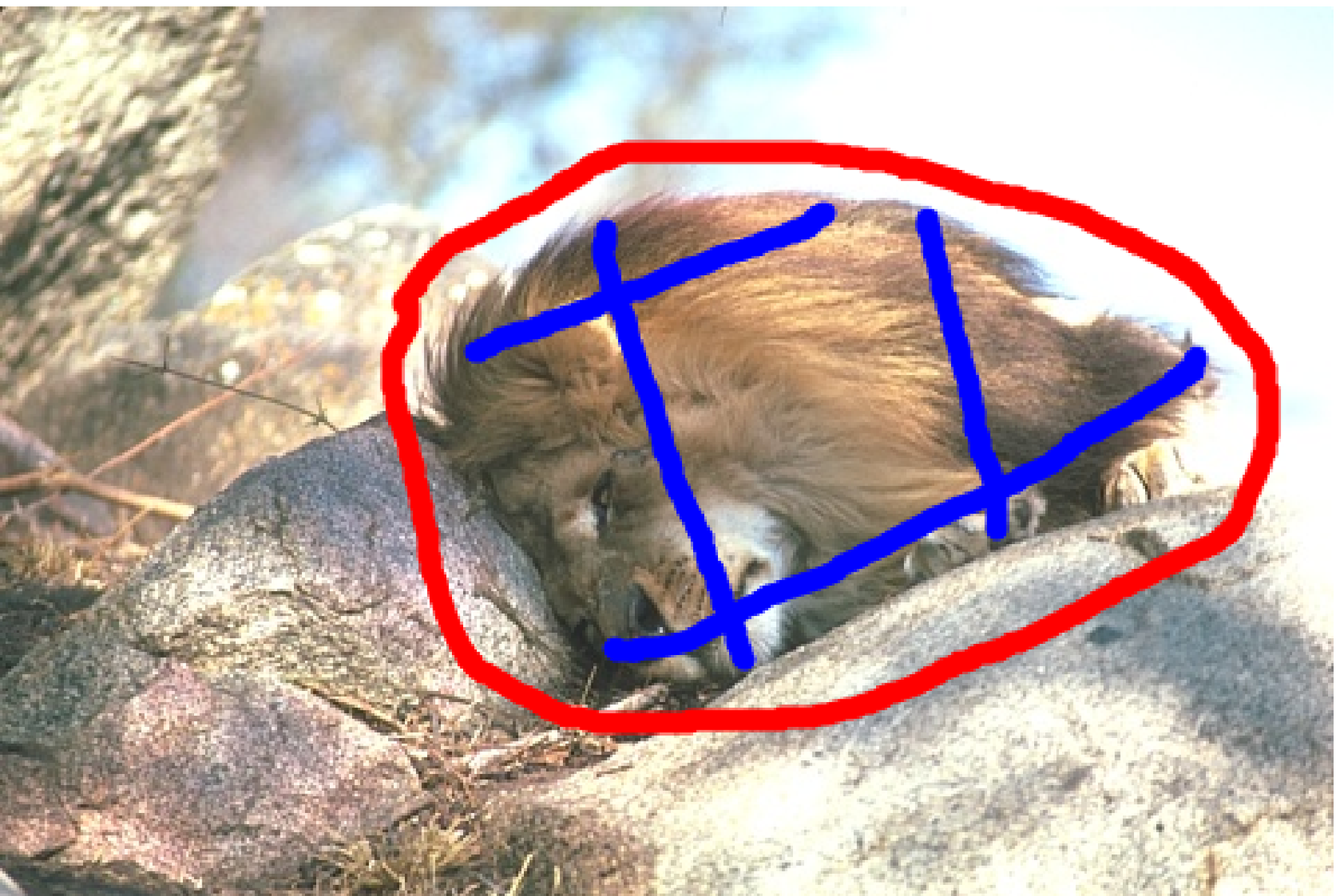}
	\includegraphics[height=1.4cm, width=1.4cm]{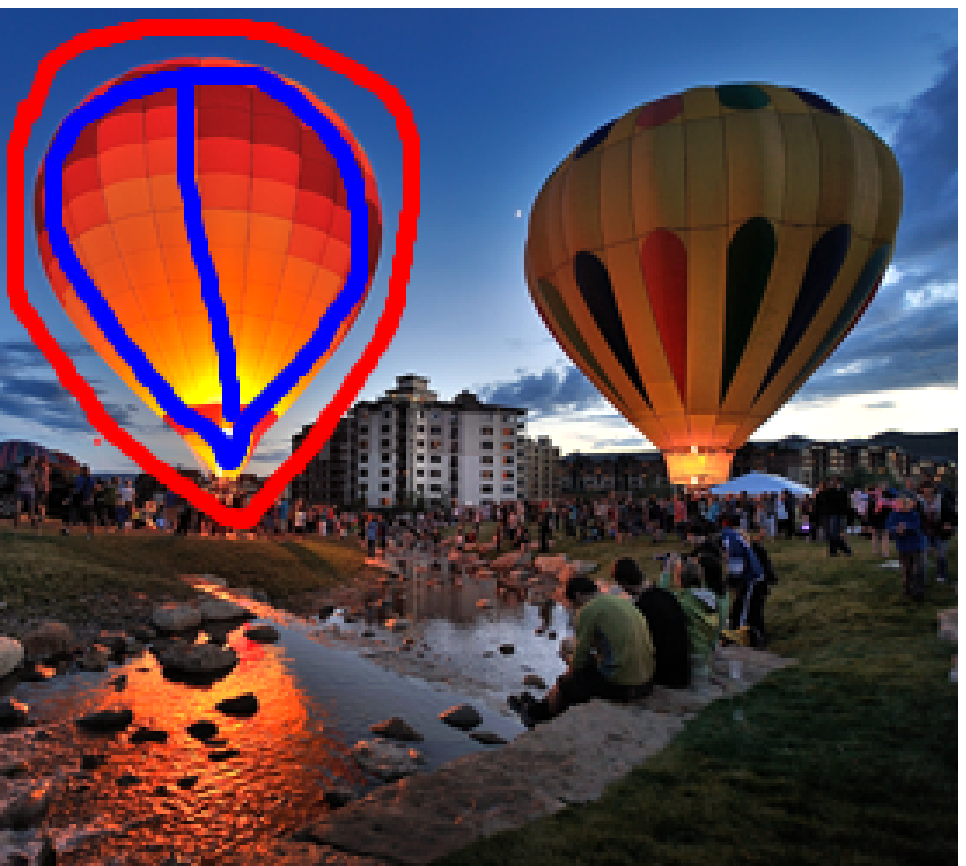}
	\includegraphics[height=1.4cm, width=1.4cm]{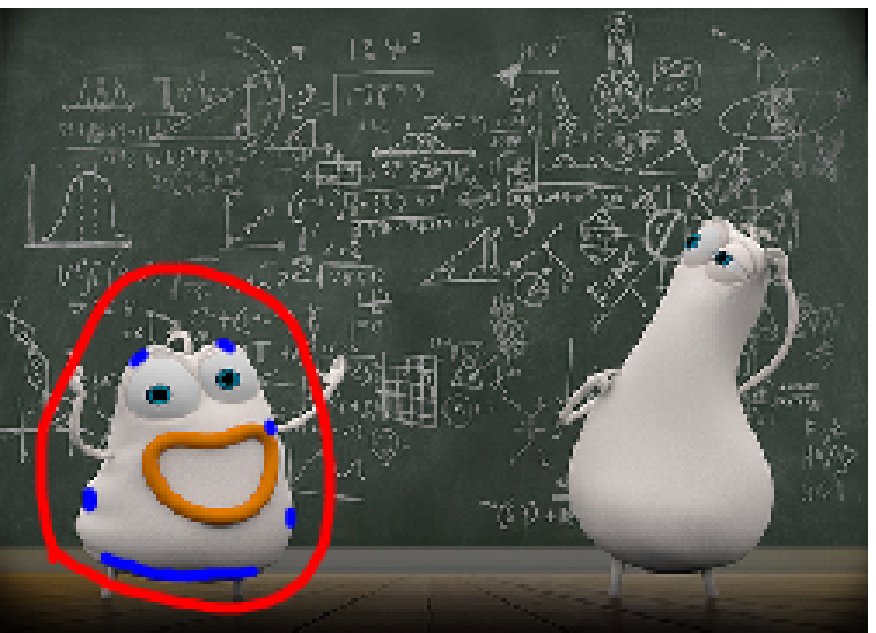}
	\caption{Test images with new labels:
		The labels for background and foreground are marked in red and blue, respectively.}\label{fig:Label}
\end{figure*}

\begin{figure}[!t]
	\centering
	\includegraphics[height=1.6cm, width=1.6cm]{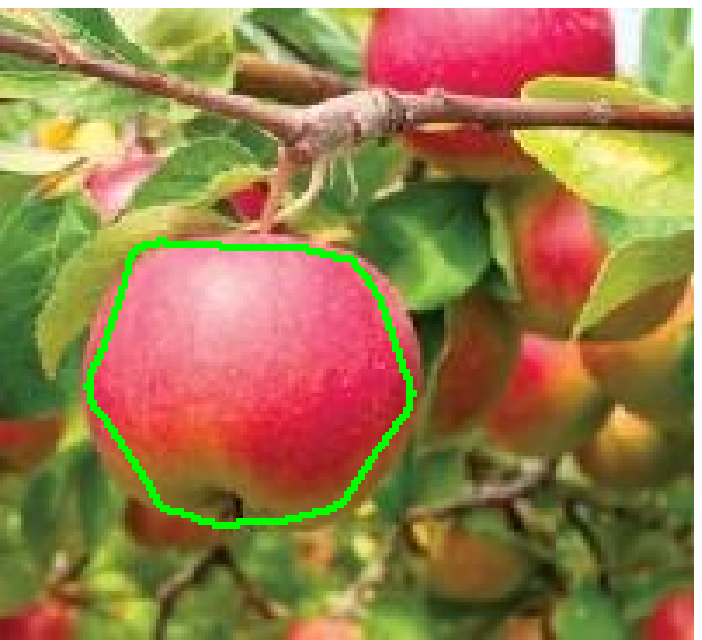}
	\includegraphics[height=1.6cm, width=1.6cm]{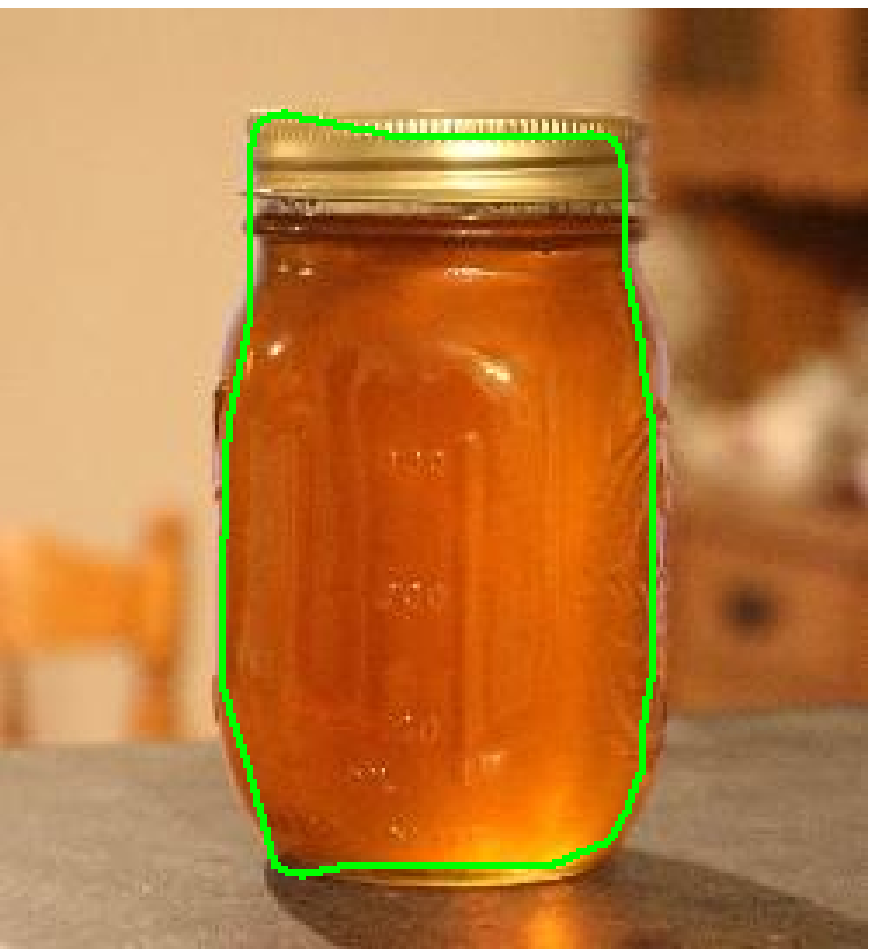}
	\includegraphics[height=1.6cm, width=1.6cm]{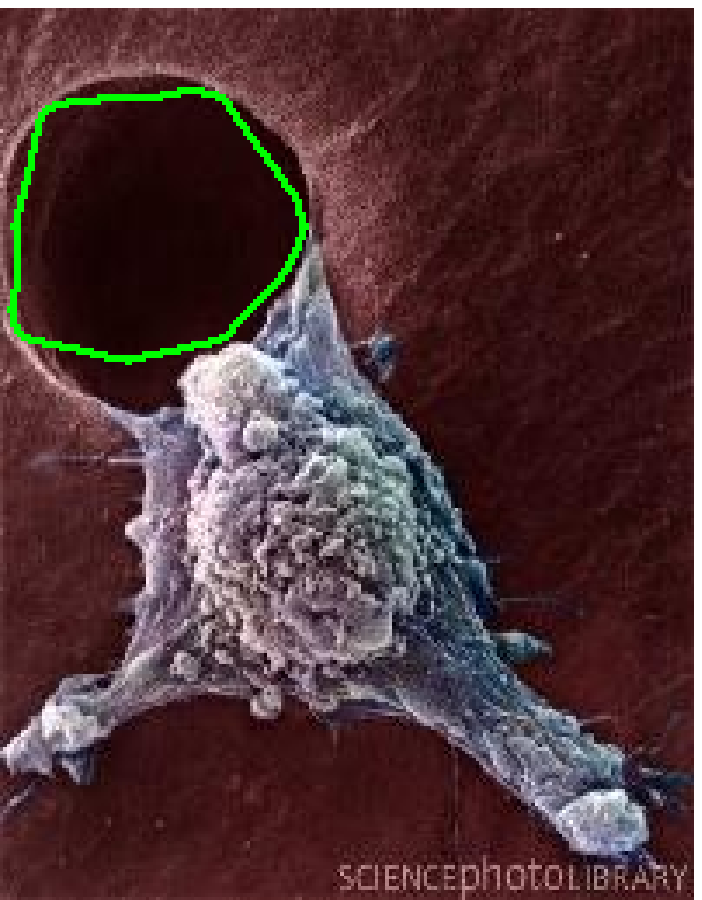}
	\includegraphics[height=1.6cm, width=1.6cm]{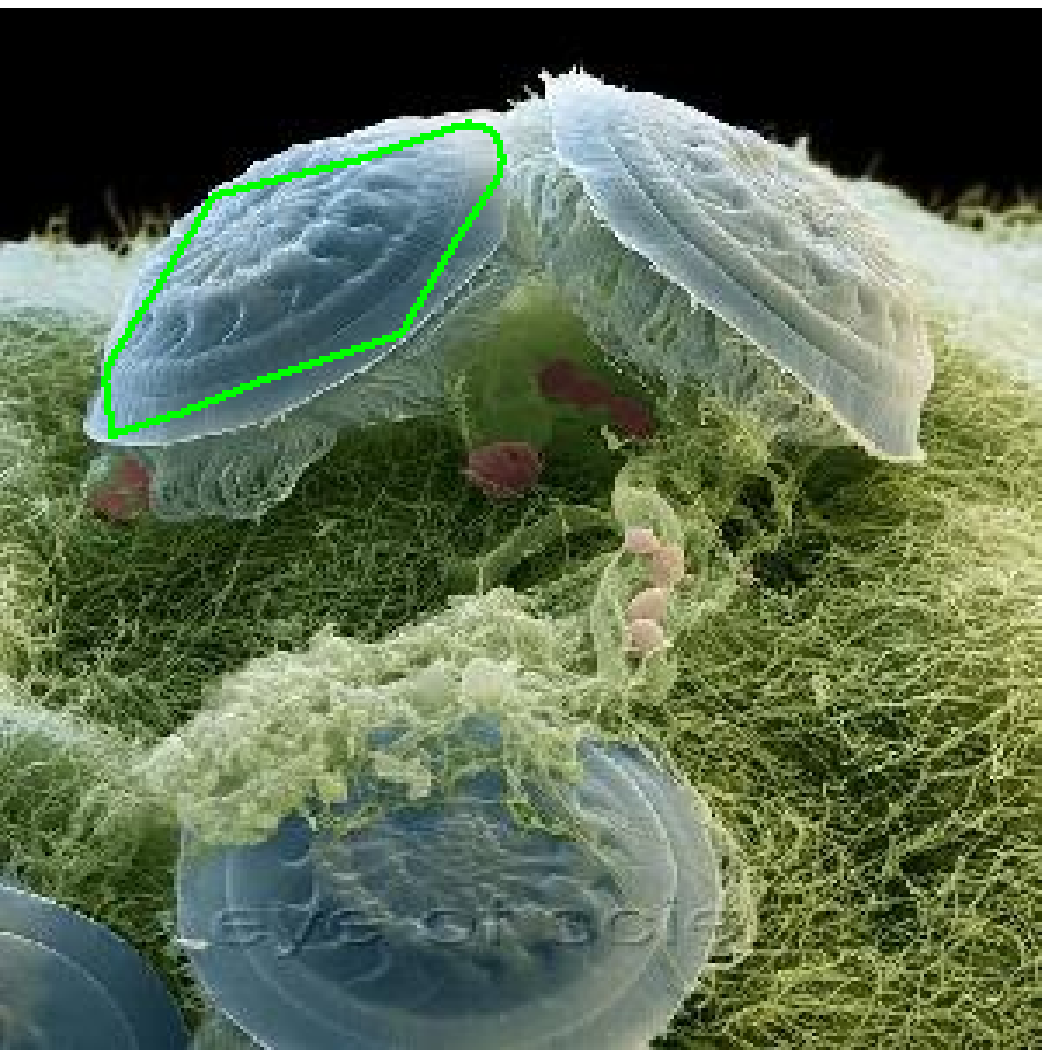}
	\includegraphics[height=1.6cm, width=1.6cm]{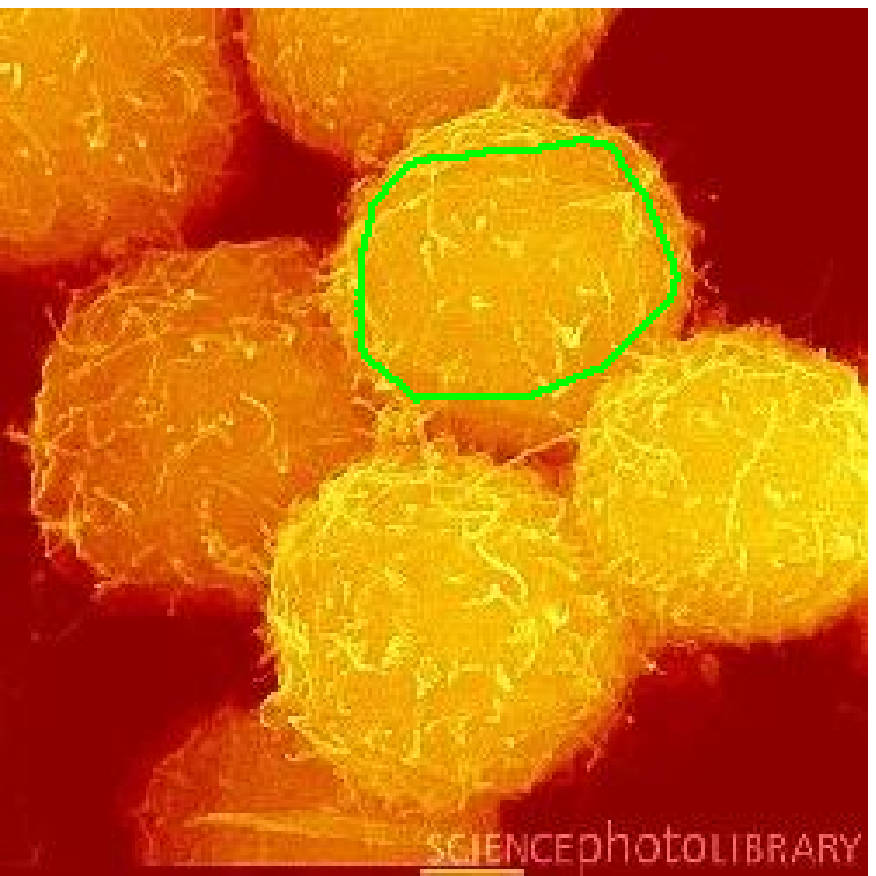}\\
	
	\includegraphics[height=1.6cm, width=1.6cm]{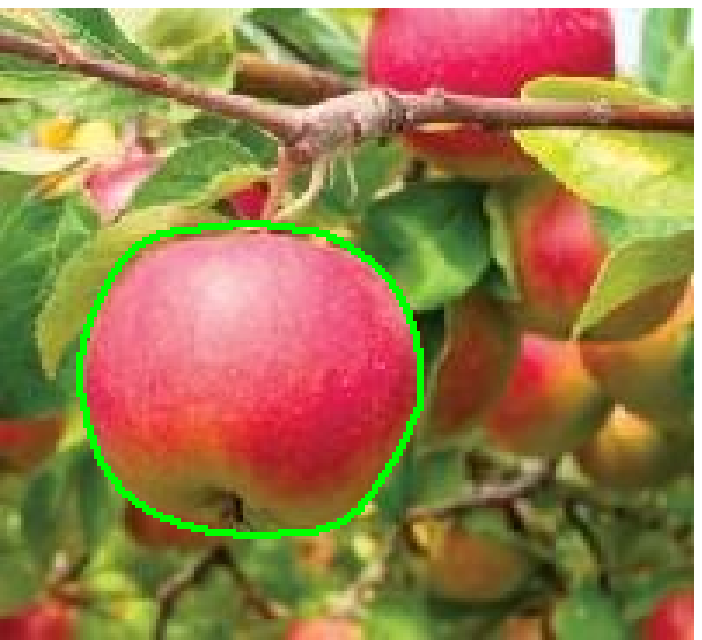}
	\includegraphics[height=1.6cm, width=1.6cm]{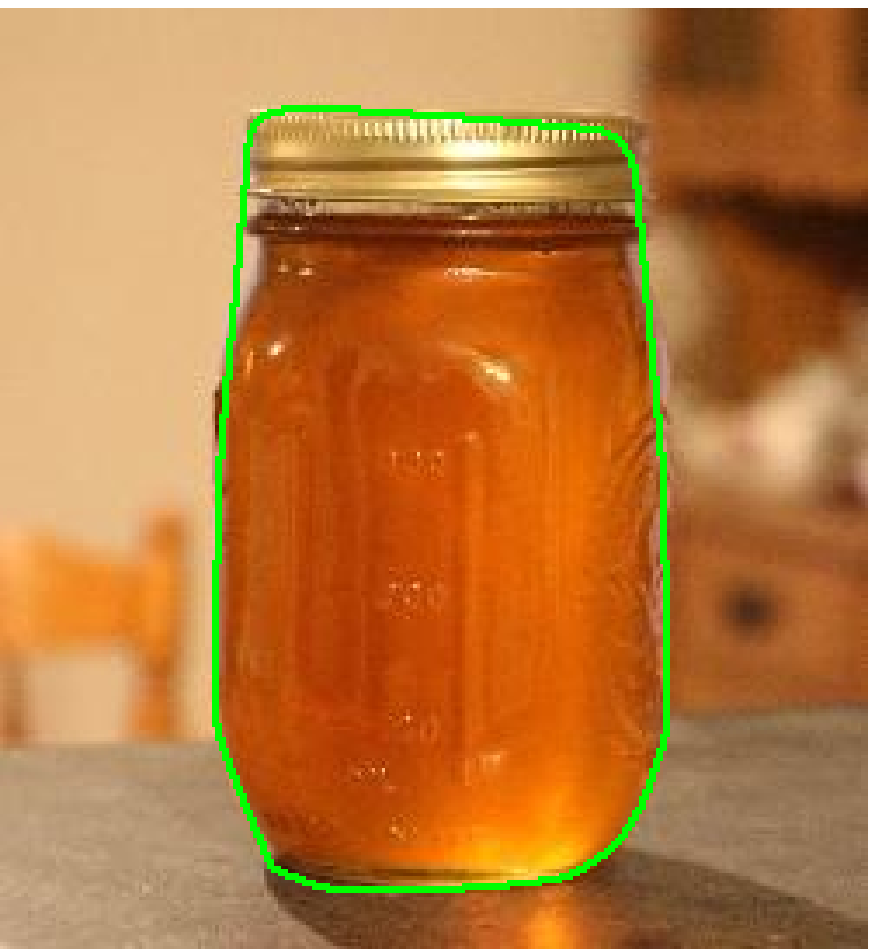}
	\includegraphics[height=1.6cm, width=1.6cm]{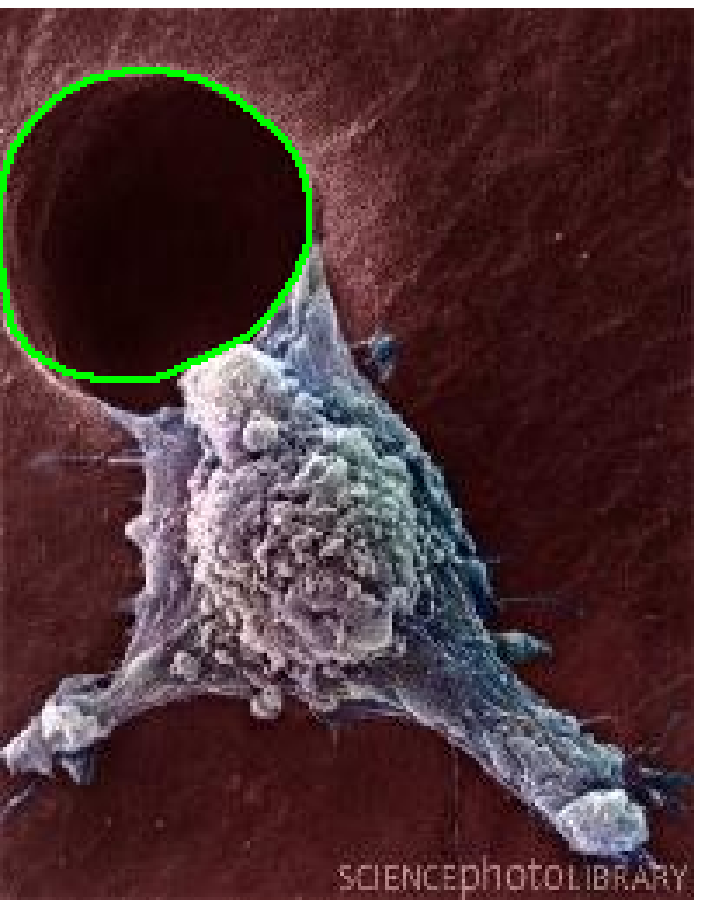}
	\includegraphics[height=1.6cm, width=1.6cm]{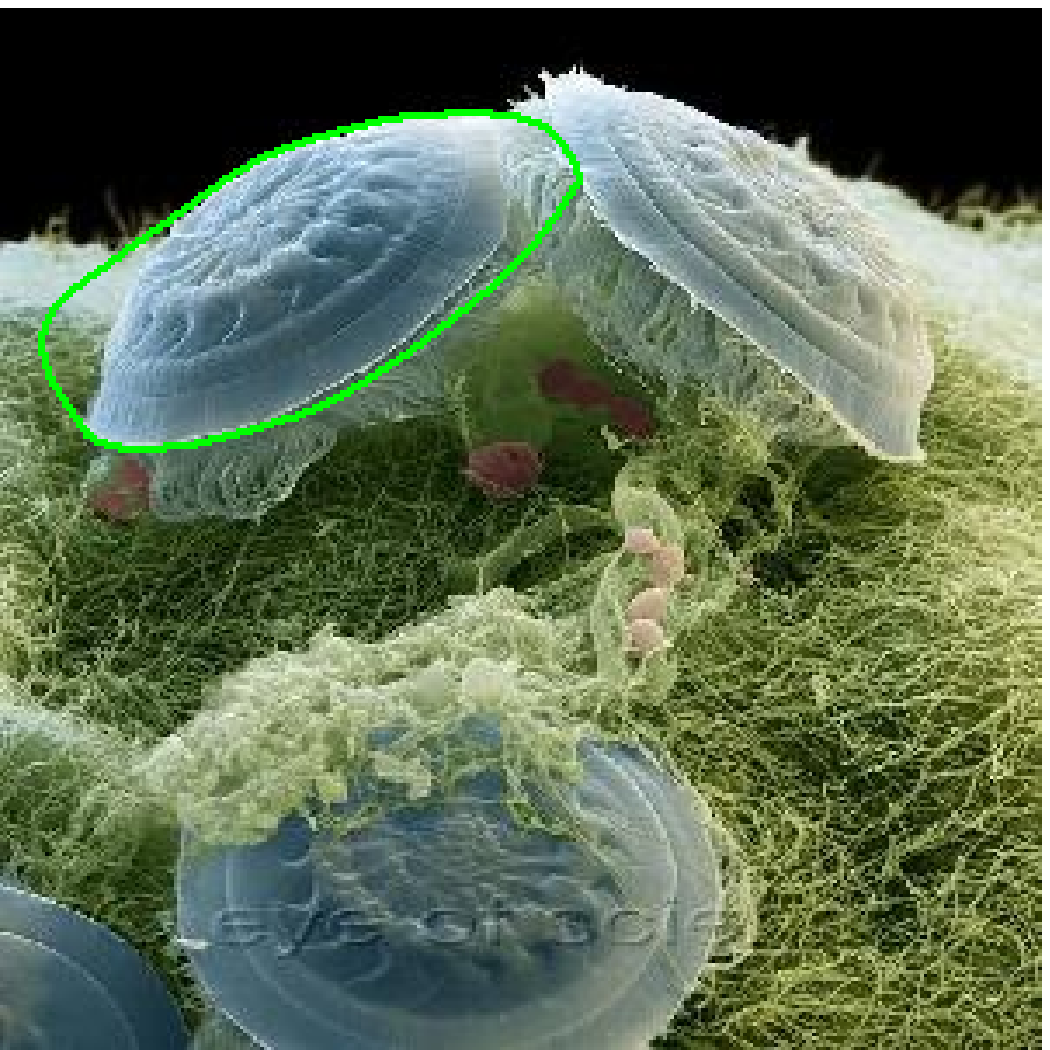}
	\includegraphics[height=1.6cm, width=1.6cm]{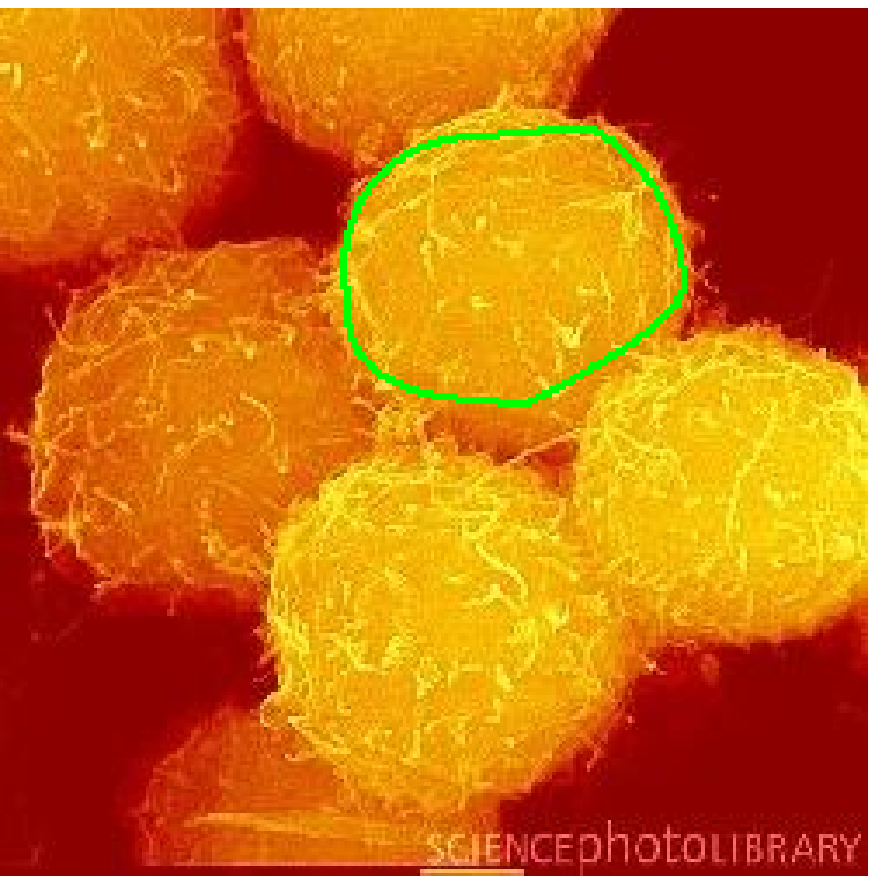}\\
	
	\includegraphics[height=1.6cm, width=1.6cm]{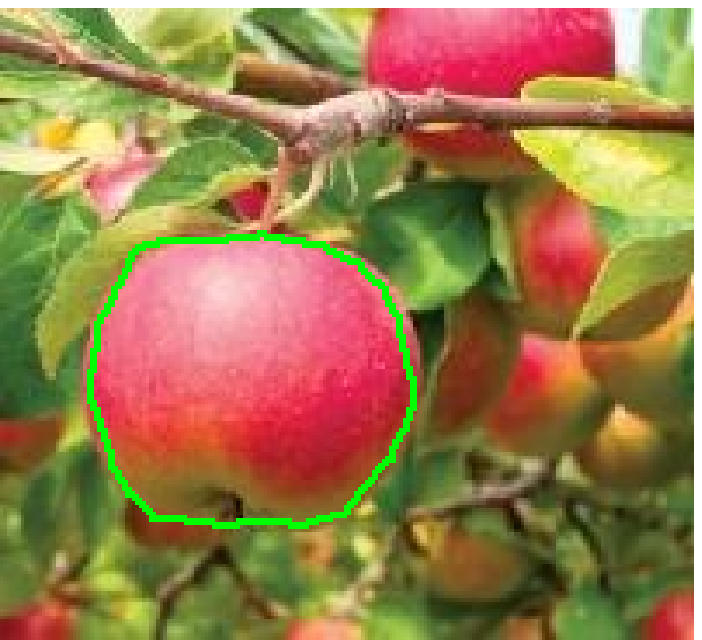}
	\includegraphics[height=1.6cm, width=1.6cm]{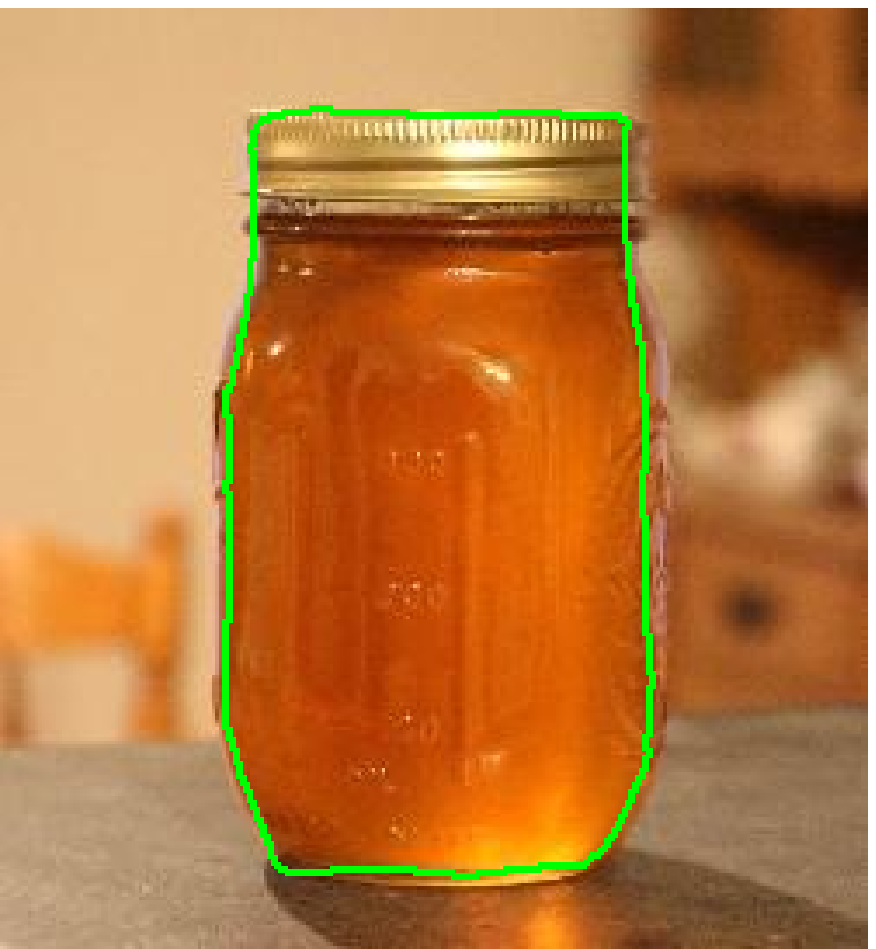}
	\includegraphics[height=1.6cm, width=1.6cm]{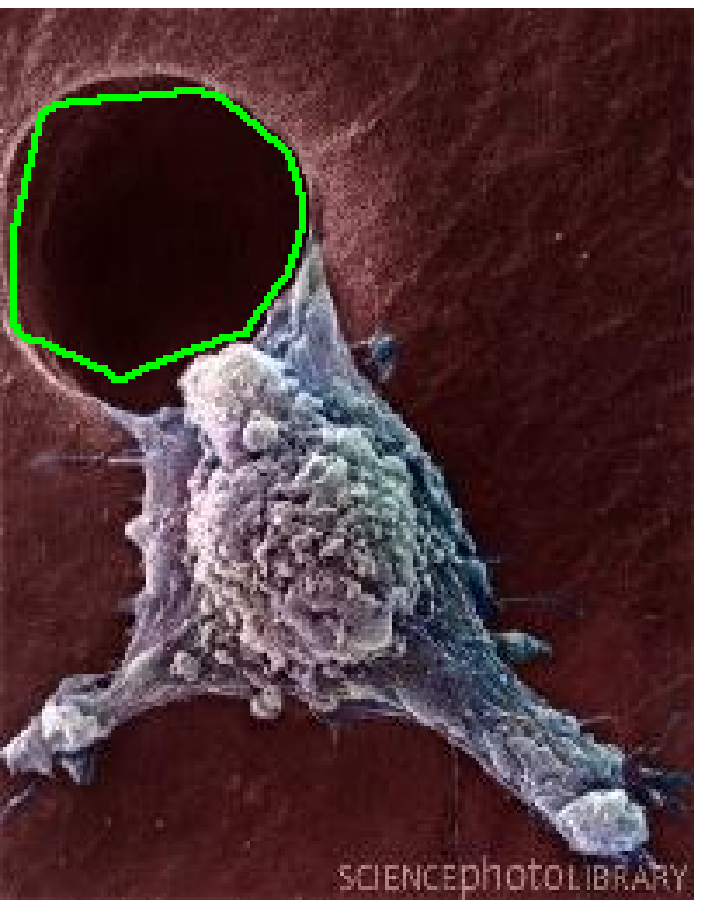}
	\includegraphics[height=1.6cm, width=1.6cm]{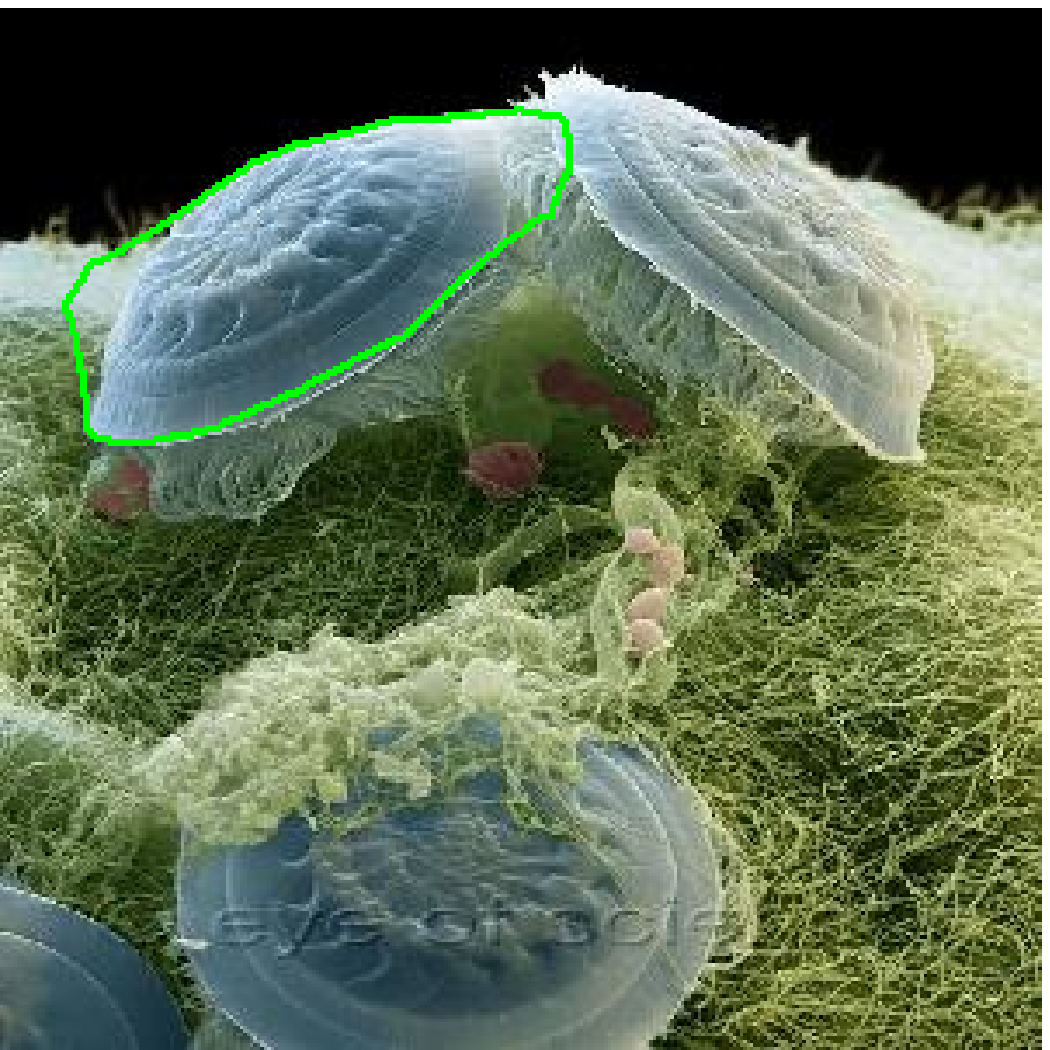}
	\includegraphics[height=1.6cm, width=1.6cm]{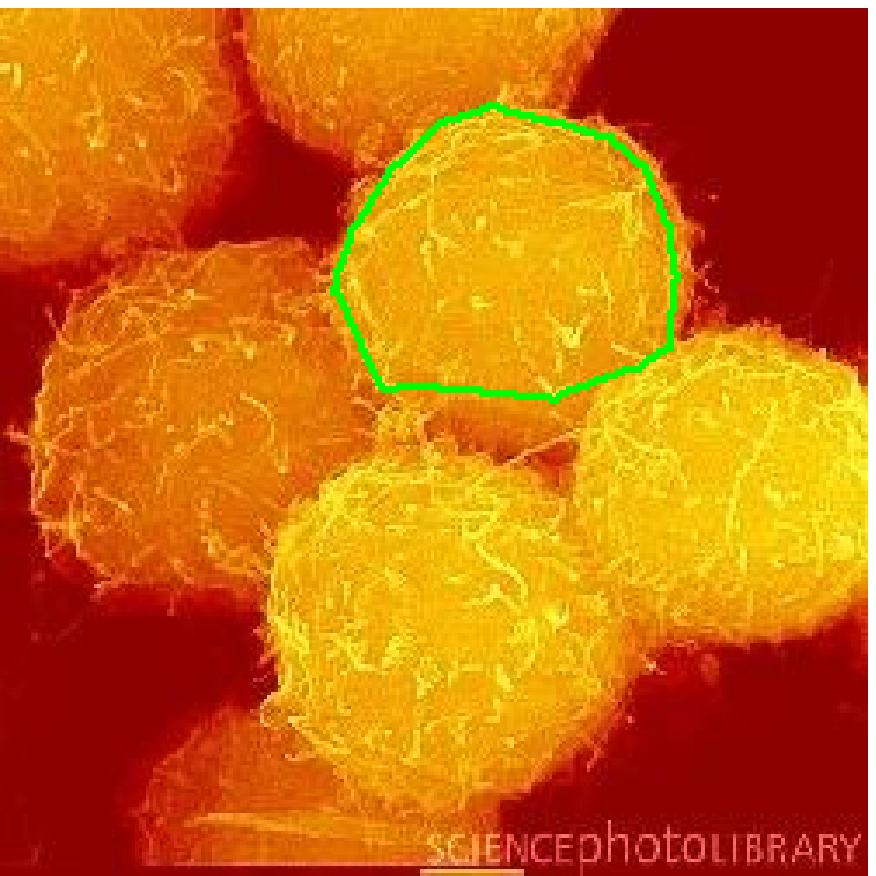}\\
	
	\includegraphics[height=1.6cm, width=1.6cm]{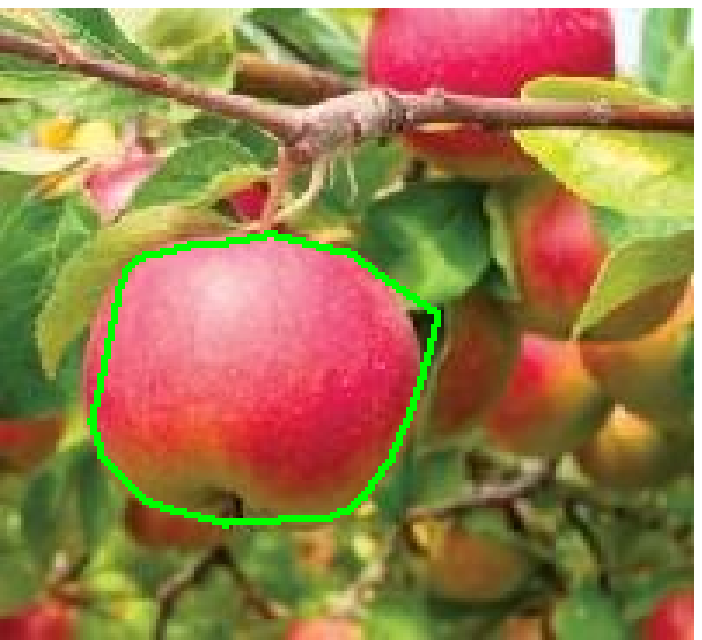}
	\includegraphics[height=1.6cm, width=1.6cm]{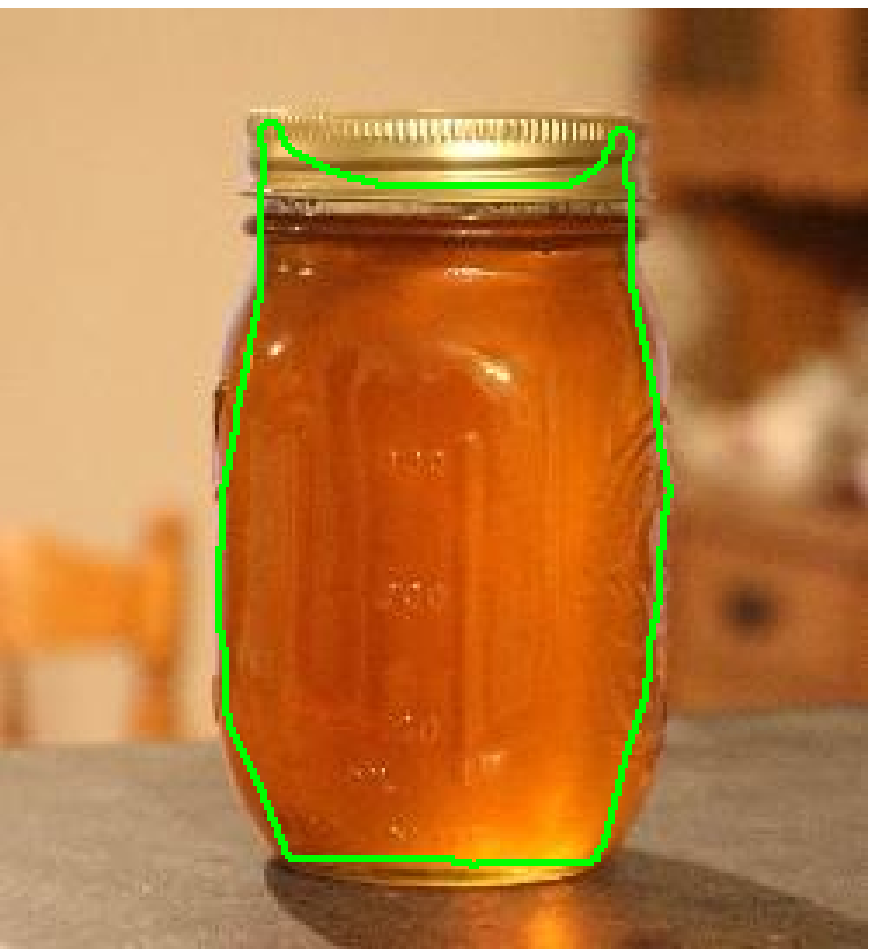}
	\includegraphics[height=1.6cm, width=1.6cm]{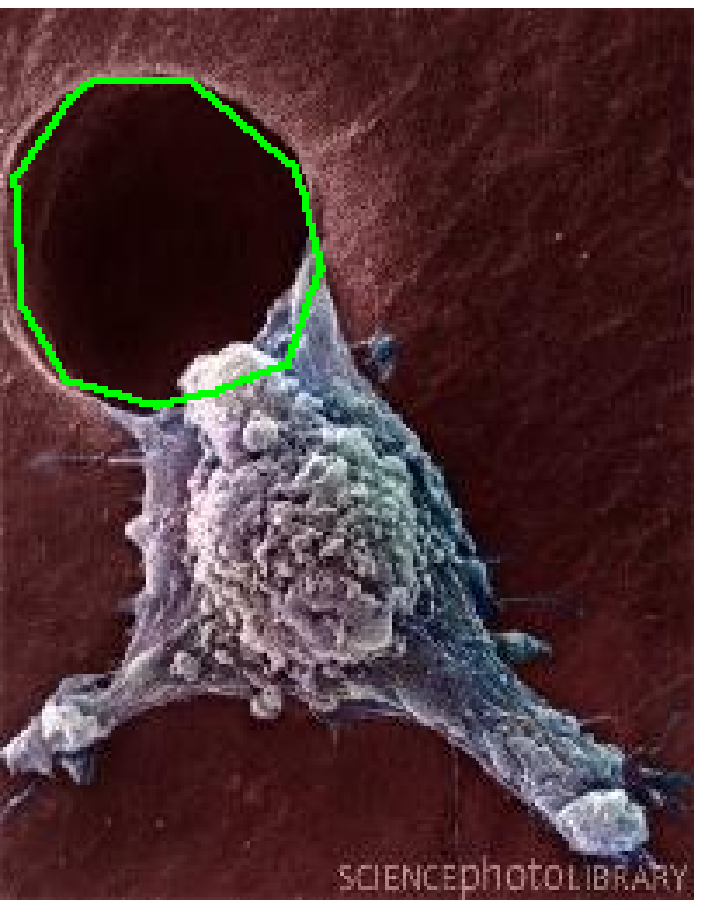}
	\includegraphics[height=1.6cm, width=1.6cm]{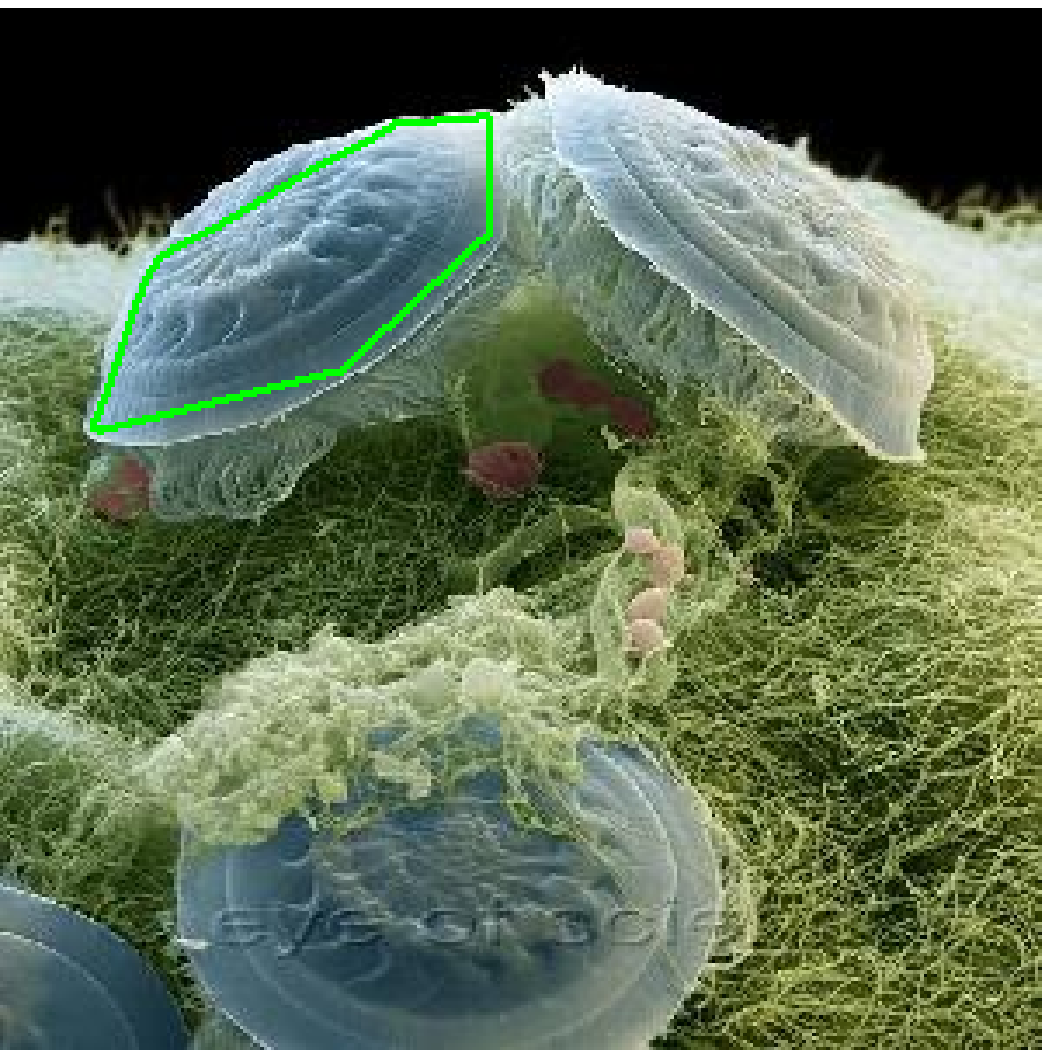}
	\includegraphics[height=1.6cm, width=1.6cm]{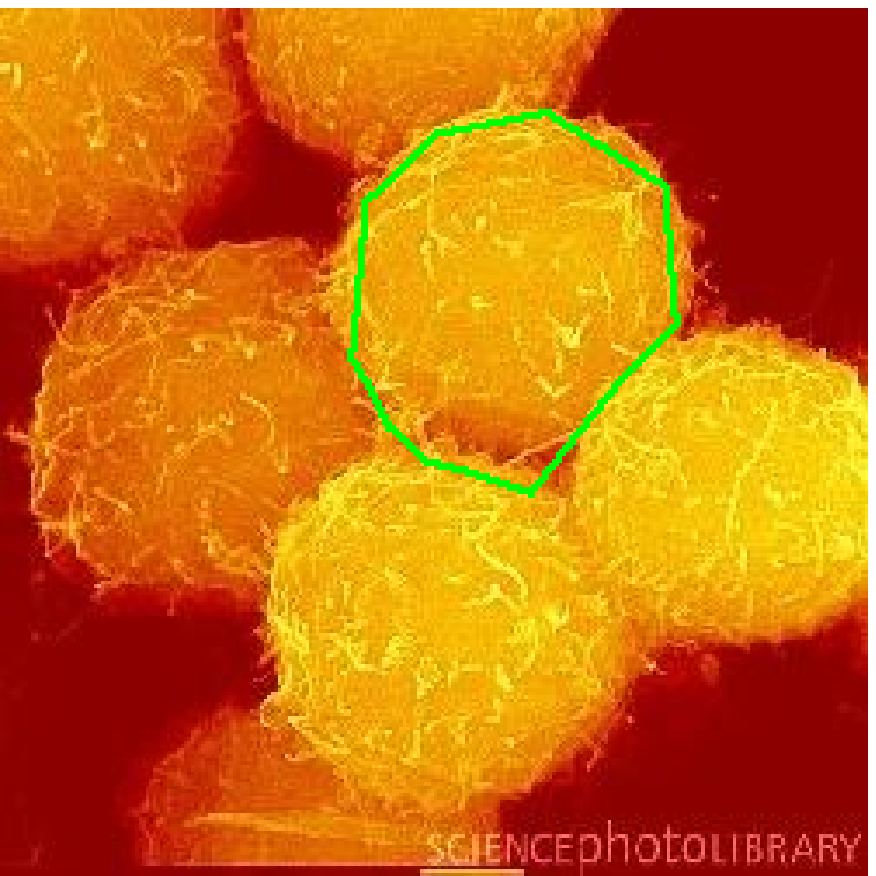}\\
	
	\includegraphics[height=1.6cm, width=1.6cm]{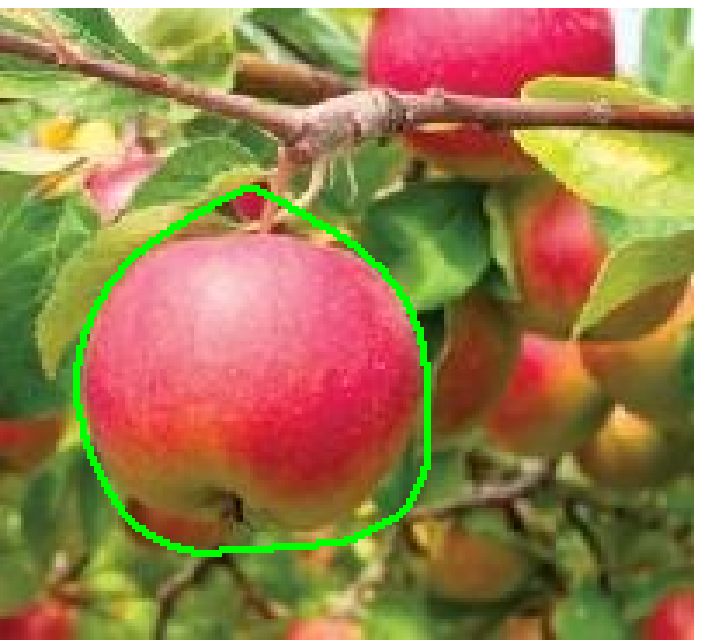}
	\includegraphics[height=1.6cm, width=1.6cm]{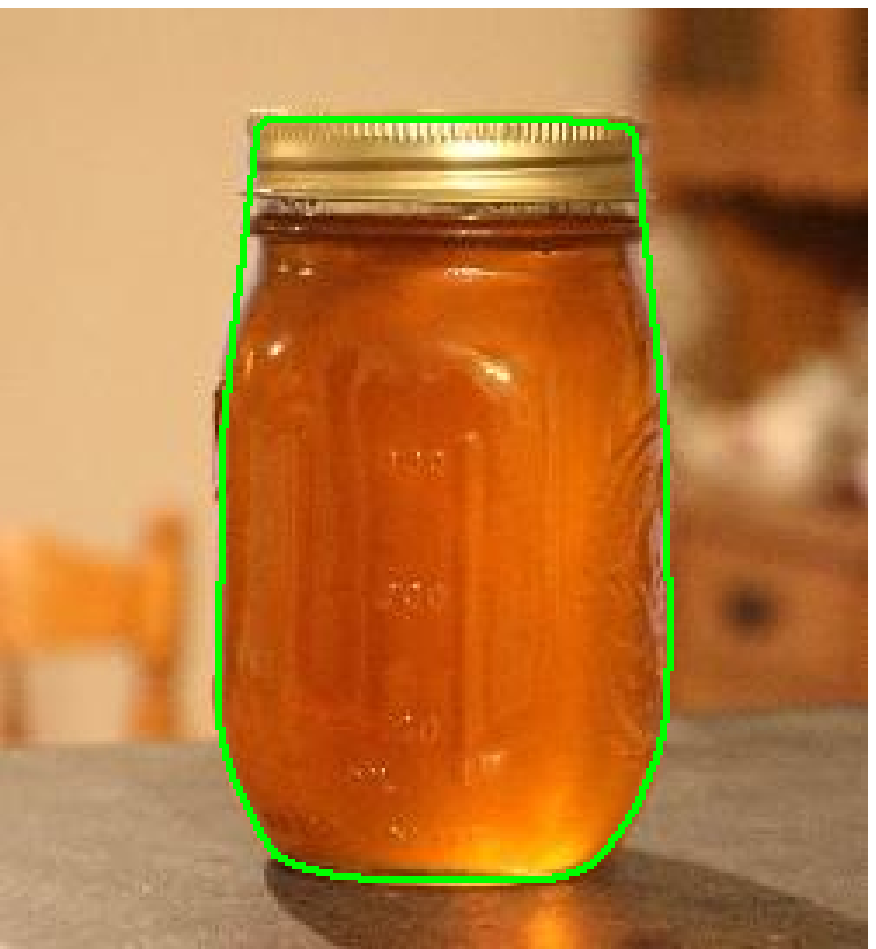}
	\includegraphics[height=1.6cm, width=1.6cm]{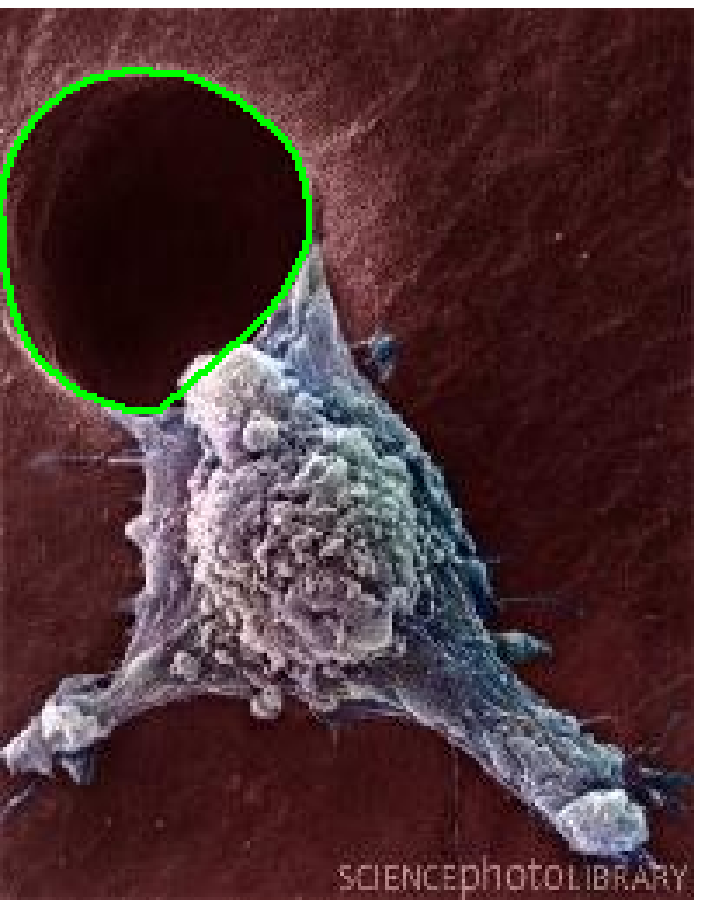}
	\includegraphics[height=1.6cm, width=1.6cm]{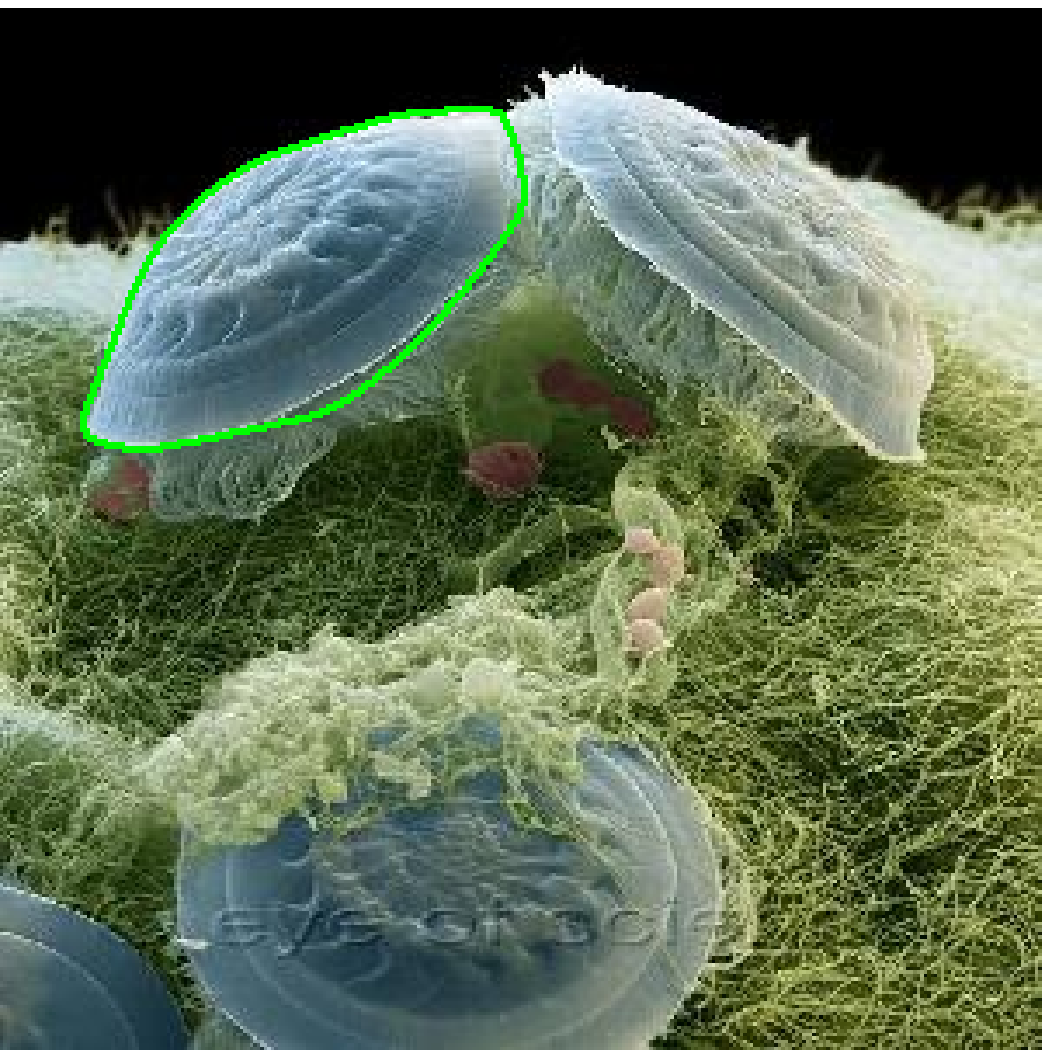}
	\includegraphics[height=1.6cm, width=1.6cm]{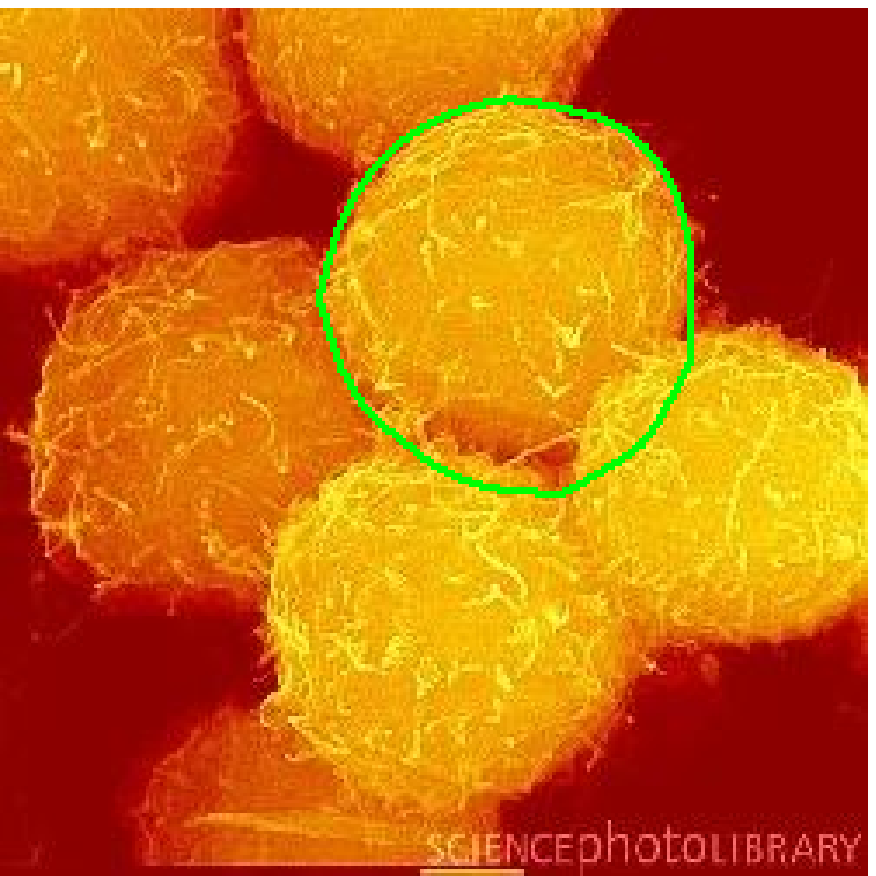}\\
	
	\includegraphics[height=1.6cm, width=1.6cm]{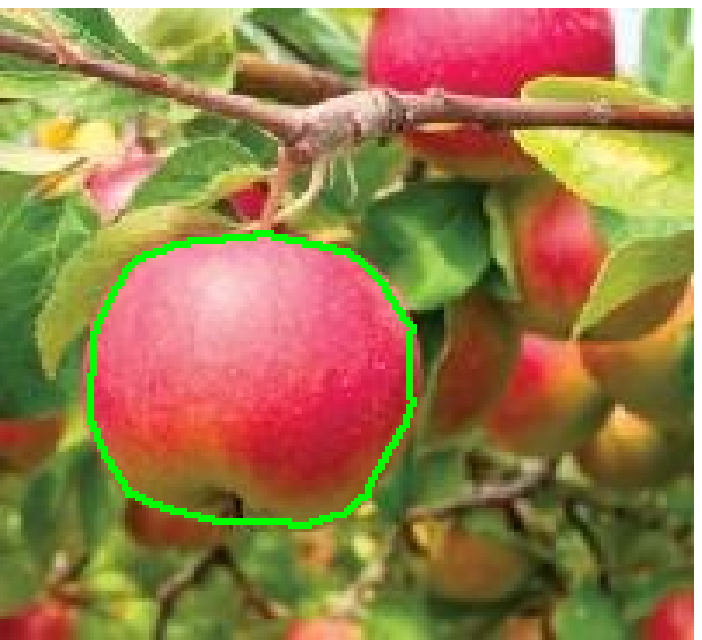}
	\includegraphics[height=1.6cm, width=1.6cm]{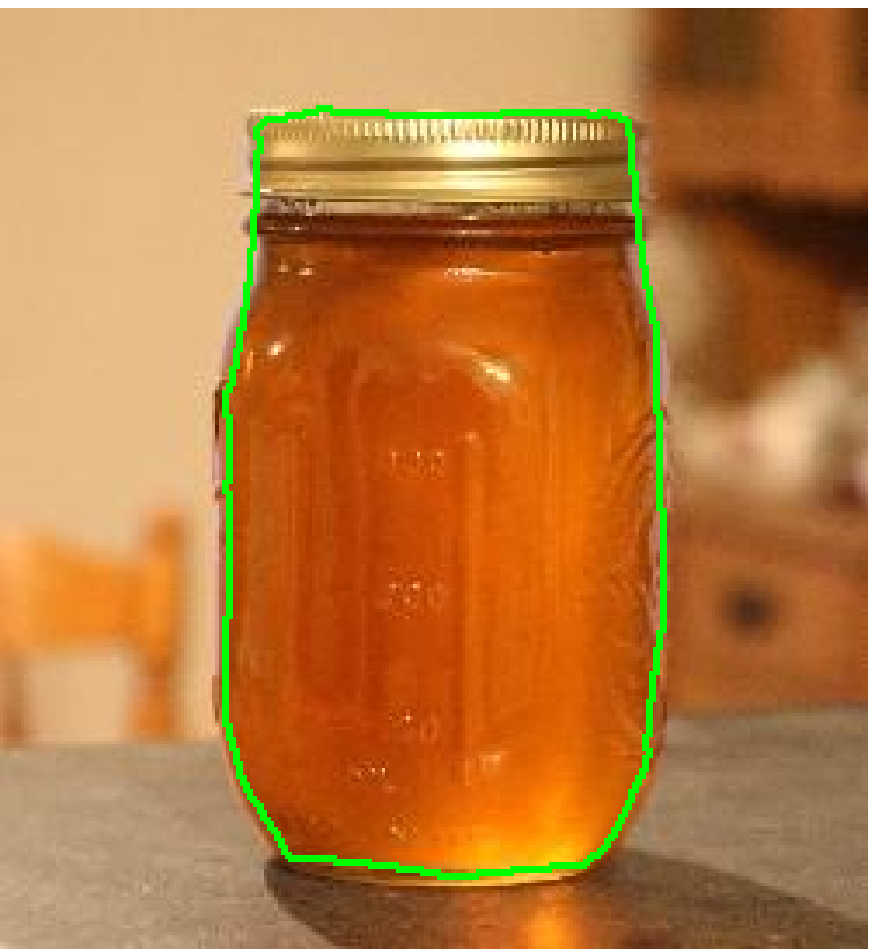}
	\includegraphics[height=1.6cm, width=1.6cm]{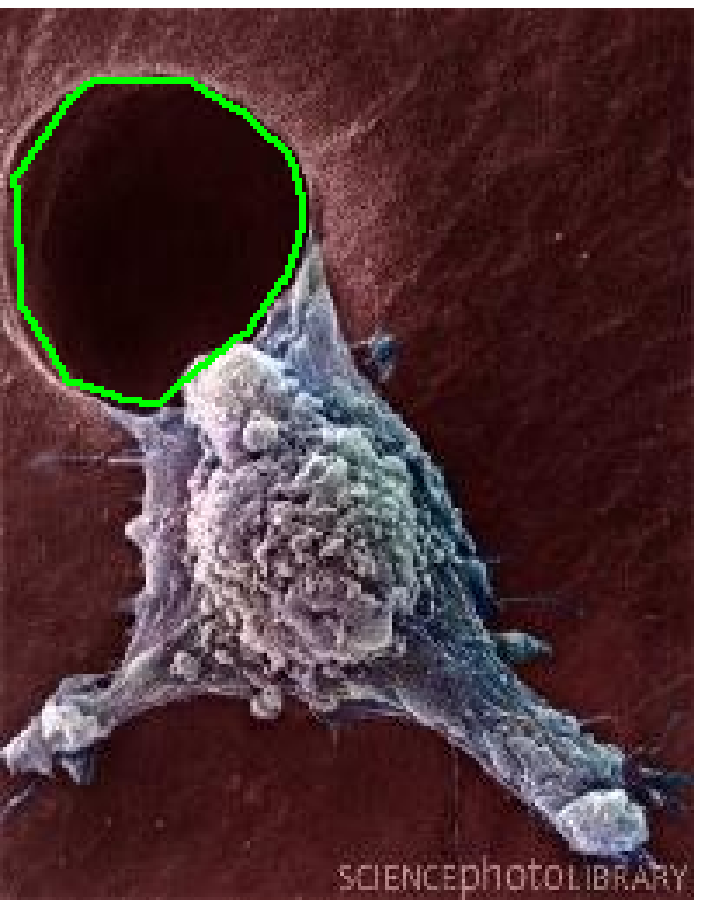}
	\includegraphics[height=1.6cm, width=1.6cm]{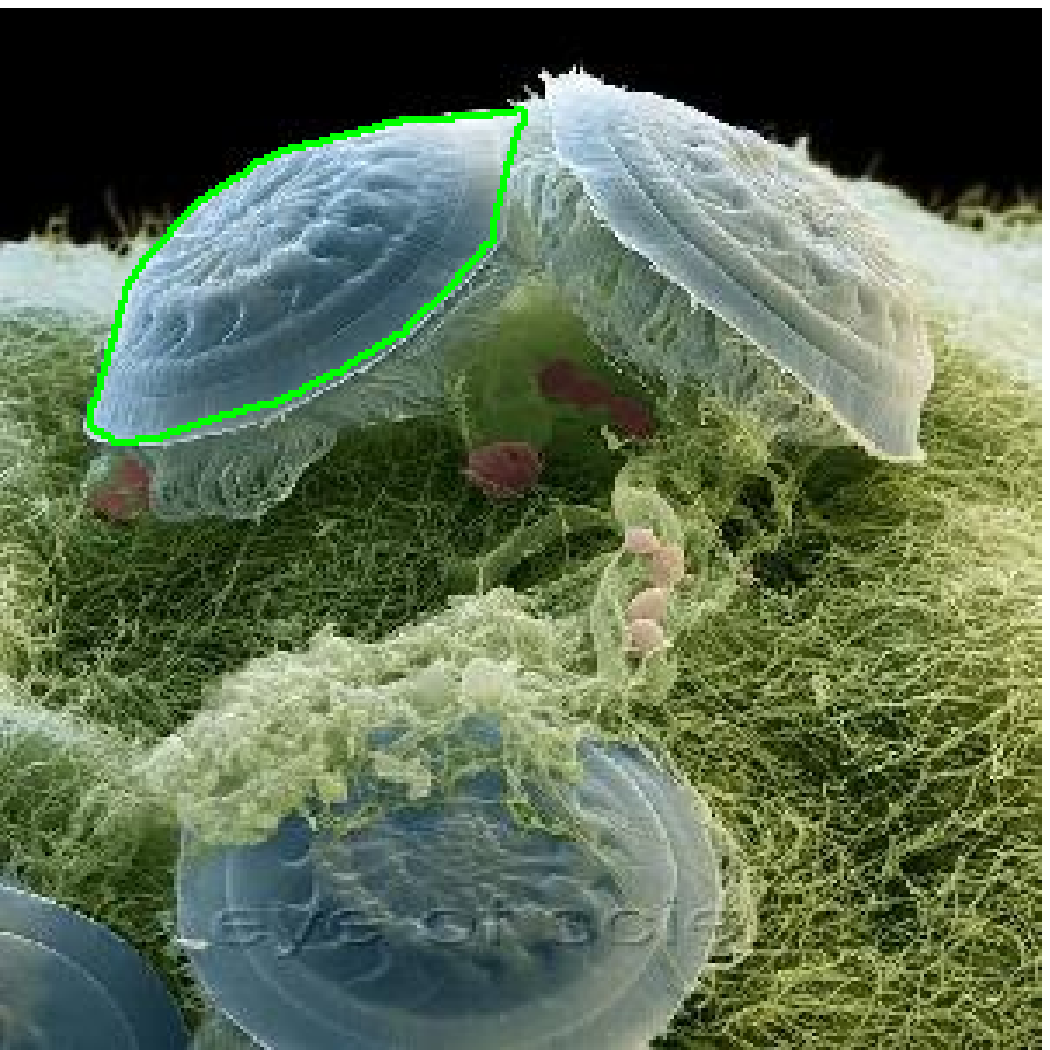}
	\includegraphics[height=1.6cm, width=1.6cm]{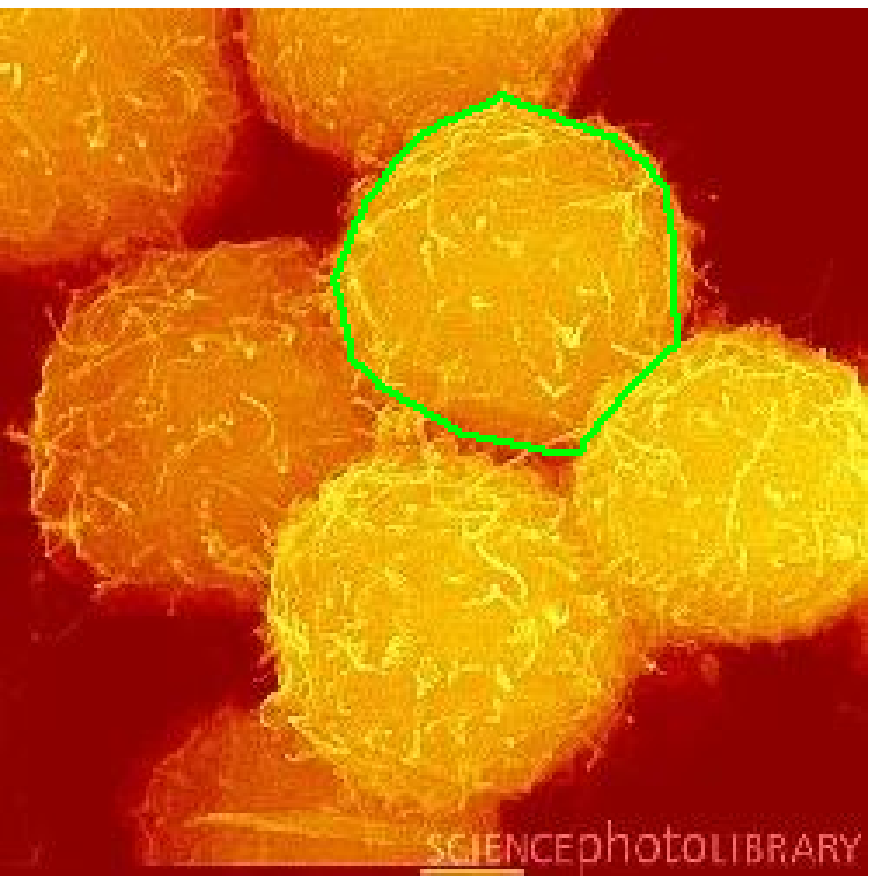}\\
	\caption{Results comparison:
		Results of 101(11), LS and the proposed methods using the downloaded (new subscribed) labels are
		shown in 1st, 2nd and 3rd rows (4th, 5th and 6th rows).}
	\label{fig:SegImg}
\end{figure}

\begin{table*}[!t]
	\caption{Average shape-distances of the results by the 1-0-1 method, LS method and our proposed method, where the best (minimal) values are bold.}
	\label{tab:shdist_single}
	\centering
	\begin{tabular}{c|cccc|ccc}%
		\hline
		&\multicolumn{4}{c}{Downloaded labels (\%)} &\multicolumn{3}{|c}{New labels (\%)}\\
		\hline
		Methods& 101(5) & 101(11) & LS & Ours &101(11)& LS & Ours\\
		avg&46.20 &12.76 &10.88&\pmb{10.26} &10.45&7.57 &\pmb{5.67} \\
		\hline
	\end{tabular}
\end{table*}
Compared the average shape-distances by different methods in Tab. \ref{tab:shdist_single} using the downloaded labels, the proposed method performs better than the 1-0-1 and LS methods in the sense of shape-distance.
In order to show the advantages of the proposed method, we subscribe new labels for the tested images (see Fig. \ref{fig:Label}) because the downloaded labels are not suitable for the proposed method.
Since the 101(11) is better than 101(5) (see the results using downloaded labels) for most of images, the average shape-distance of the results by 101(11) is presented for new labels in Tab. \ref{tab:shdist_single} only.
The performance of the proposed method is outstanding by comparing the shape-distances. In addition,
comparing the results in Fig. \ref{fig:SegImg} visually, we see that the proposed method can preserve the convexity of all the objects, and extract the object boundaries
exactly, but the 101(11) and LS methods fails for some objects (e.g. the last two images). 

As for the computational efficiency, the average computing time of the 101(11), LS and the proposed methods are
$52.8s$, $83.8s$ and $59.5s$, respectively.
In summary, these experiments illustrate that the proposed method outperforms the 1-0-1 \cite{Gorelick2016convexity} and LS \cite{2022Luo} methods in the sense of higher accurate results within almost the same or less time.

Besides the theoretical convergence of Algorithm  \ref{alg:PADMM} for the linearization problem (\ref{eq:modLL}), the convergence of Algorithm \ref{alg:proposed} for the proposed model (\ref{eq:mod_all}) is investigated numerically.
Taking the first three images as examples, the evolution the  of objective function values is plotted in Fig. \ref{fig:evL} (left), the total number of points that violate the convex constraint v.s. iterations is plotted in Fig. \ref{fig:evL} (right).
For all three images, the objective function values decrease monotonically until stability, and the number of violated points reduces quickly and then remains a relatively small values.

\begin{figure}[!t]
	\centering
	\includegraphics[height=3cm, width=4cm]{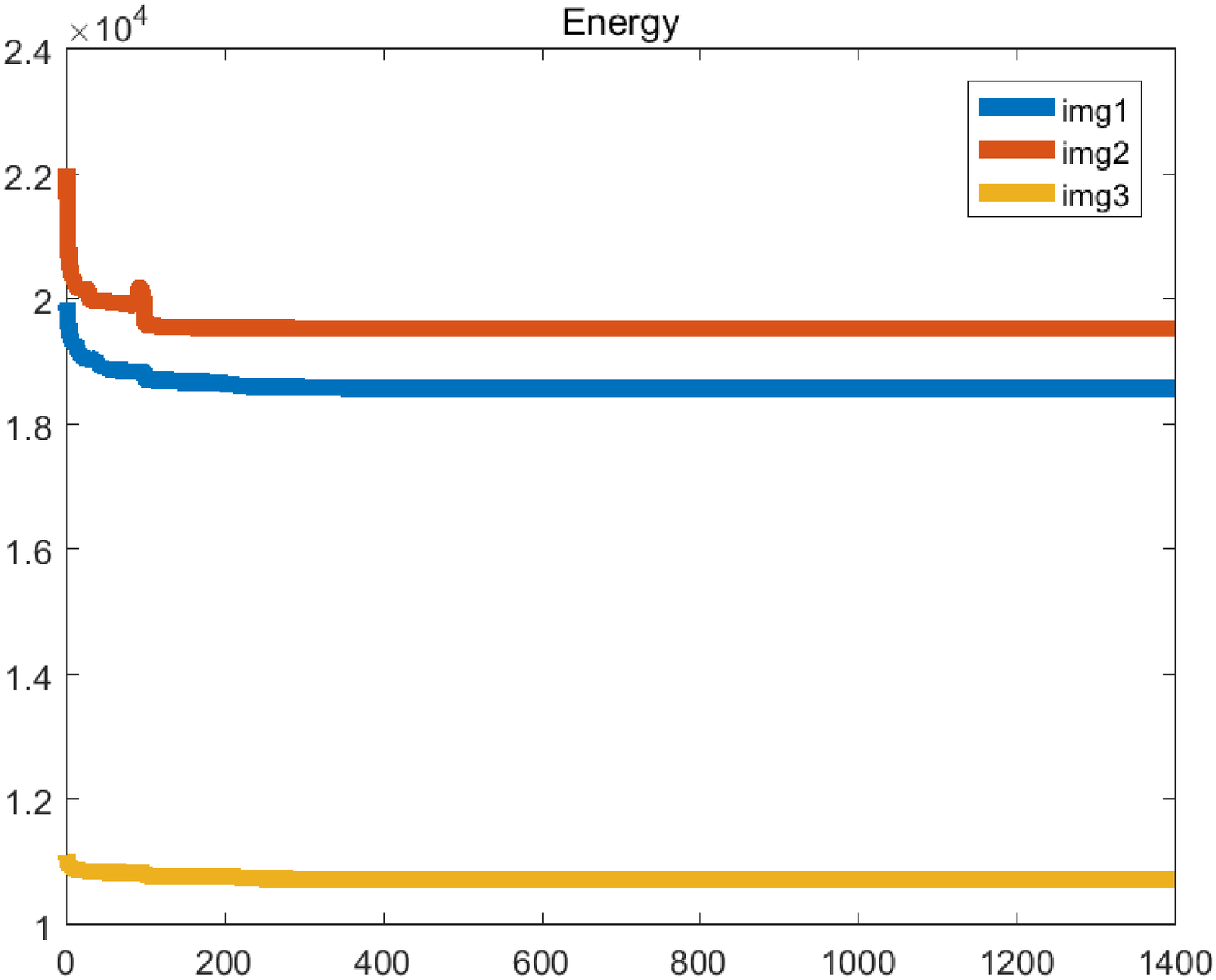}
	\includegraphics[height=3cm, width=4cm]{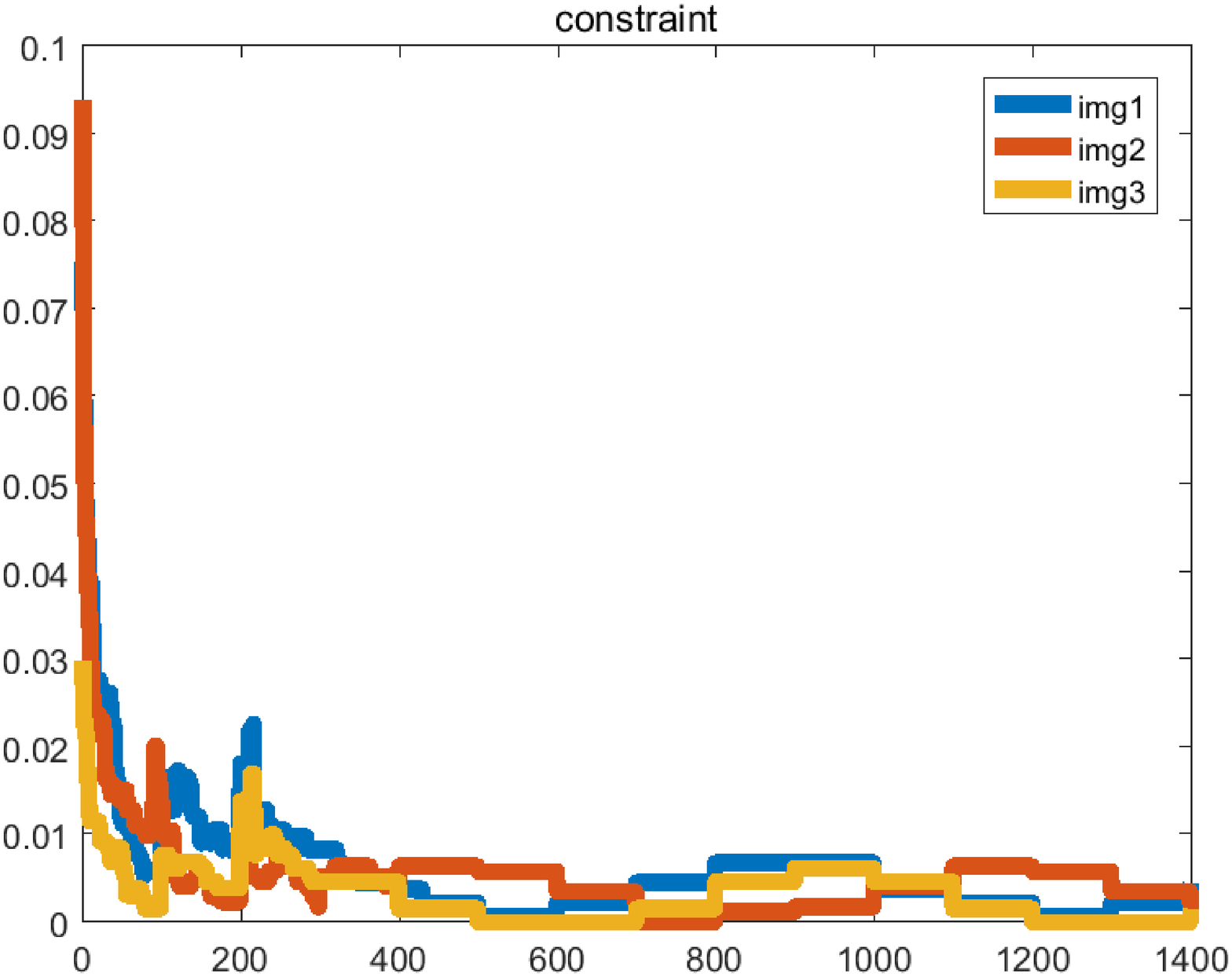}
	\caption{Numerical convergence of Algorithm \ref{alg:proposed} (left) and total numbers of points that violate the inequality constraint for convexity shape prior (right) of  the first three images in Fig. \ref{fig:Label}. }
	\label{fig:evL}
\end{figure}

\subsubsection{Multiple convex objects segmentation}
Experiments on 10 images (all the images can be seen in \cite{2022Luo})
containing multiple convex objects of interest are conducted, and the results are compared with 101(11) \cite{Gorelick2017multi_convexity} and LS methods \cite{2022Luo}.
Some segmentation results by the 101(11), LS and the proposed methods are demonstrated in Fig. \ref{fig:muti}. Besides the visual comparison, the average shape-distances of the results by different methods are tabulated in Tab. \ref{tab:mult_shdist_time}.

\begin{figure}[!t]
	\centering
	\includegraphics[height=1.6cm, width=1.6cm]{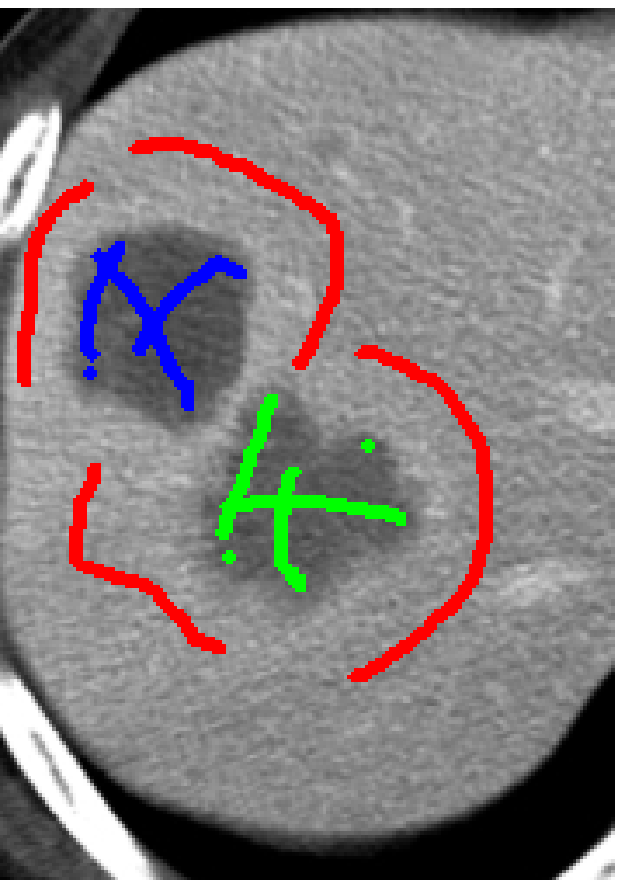}
	\includegraphics[height=1.6cm, width=1.6cm]{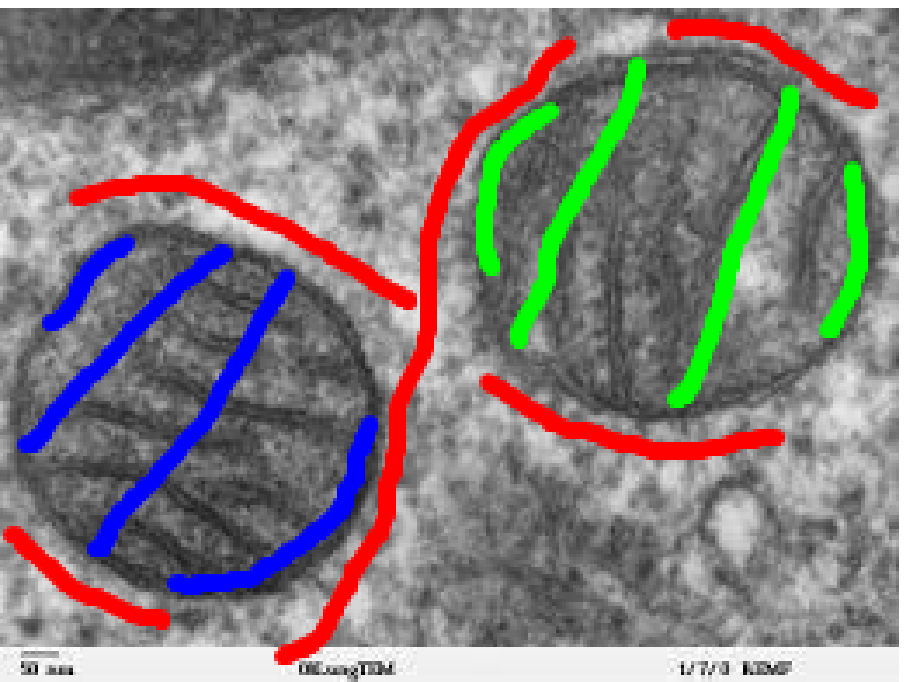}
	\includegraphics[height=1.6cm, width=1.6cm]{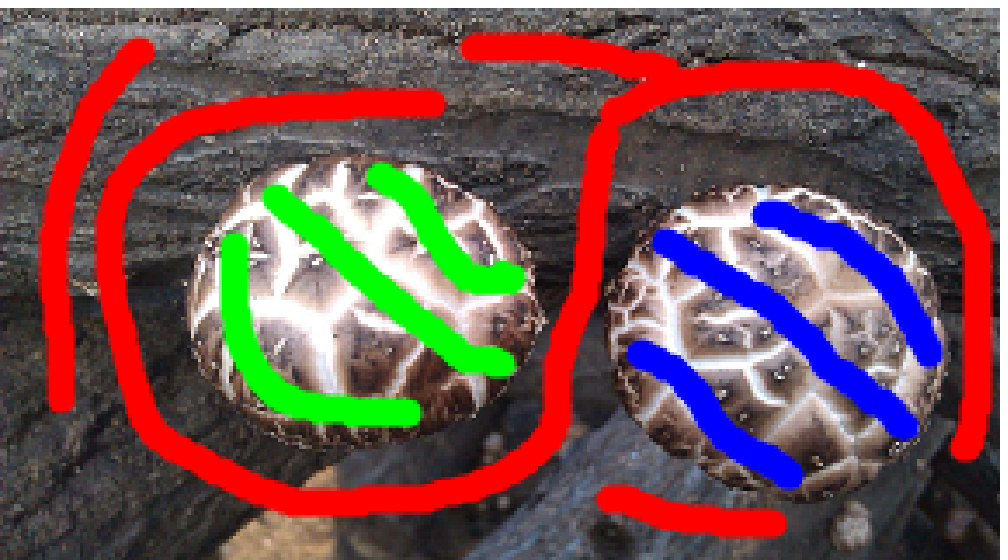}
	\includegraphics[height=1.6cm, width=1.6cm]{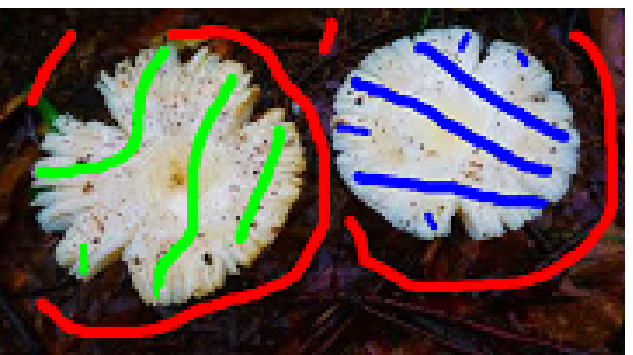}
	\includegraphics[height=1.6cm, width=1.6cm]{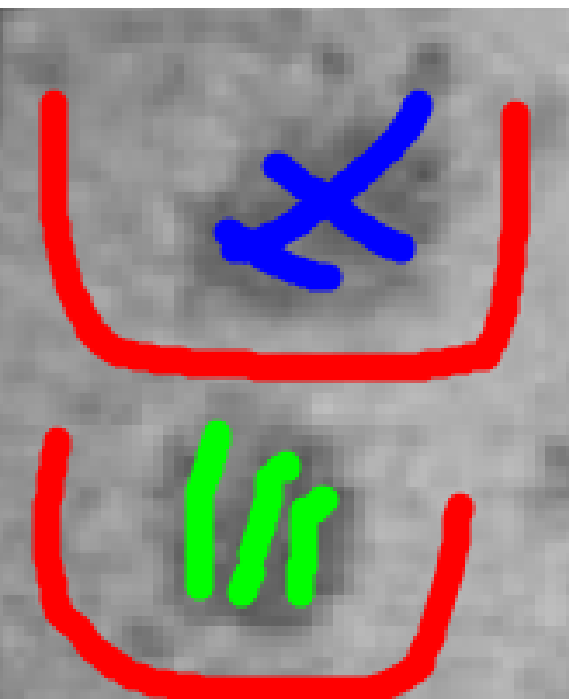}\\
	
	\includegraphics[height=1.6cm, width=1.6cm]{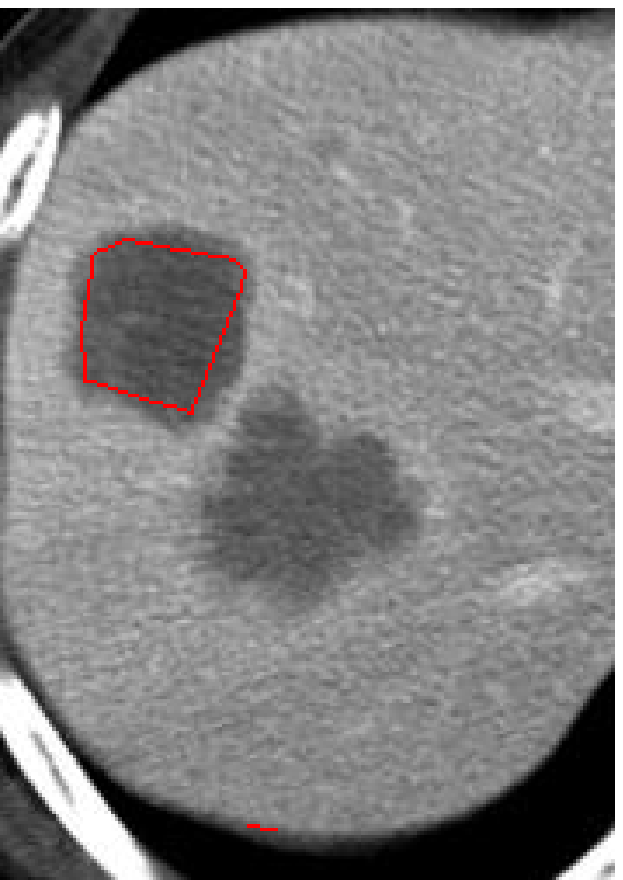}
	\includegraphics[height=1.6cm, width=1.6cm]{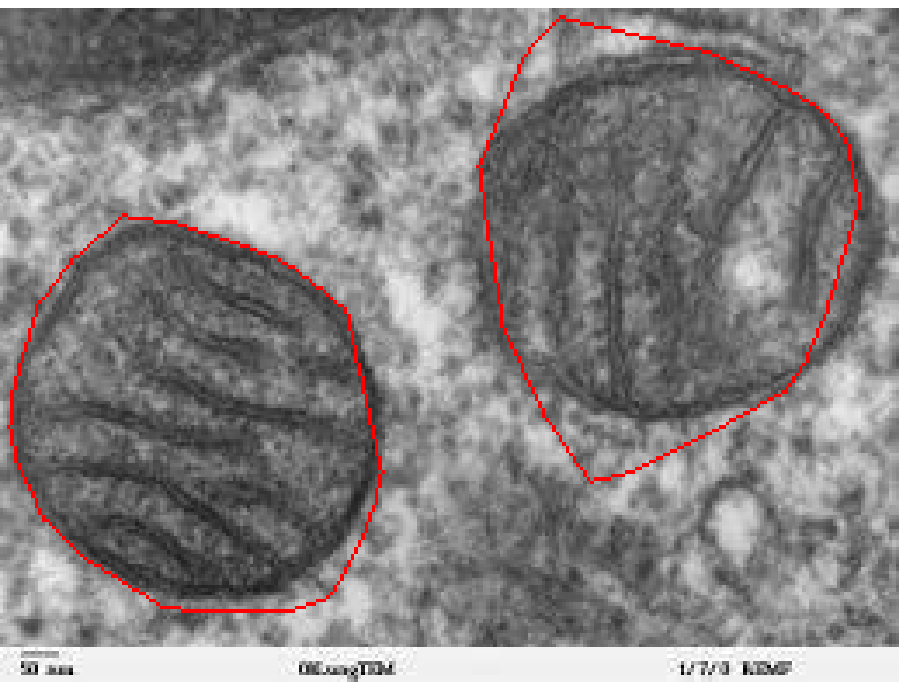}
	\includegraphics[height=1.6cm, width=1.6cm]{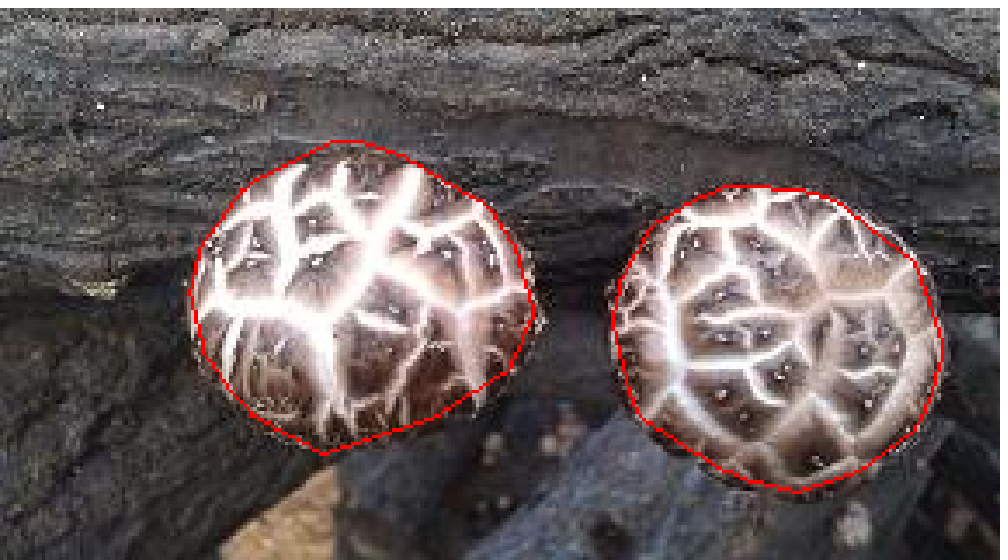}
	\includegraphics[height=1.6cm, width=1.6cm]{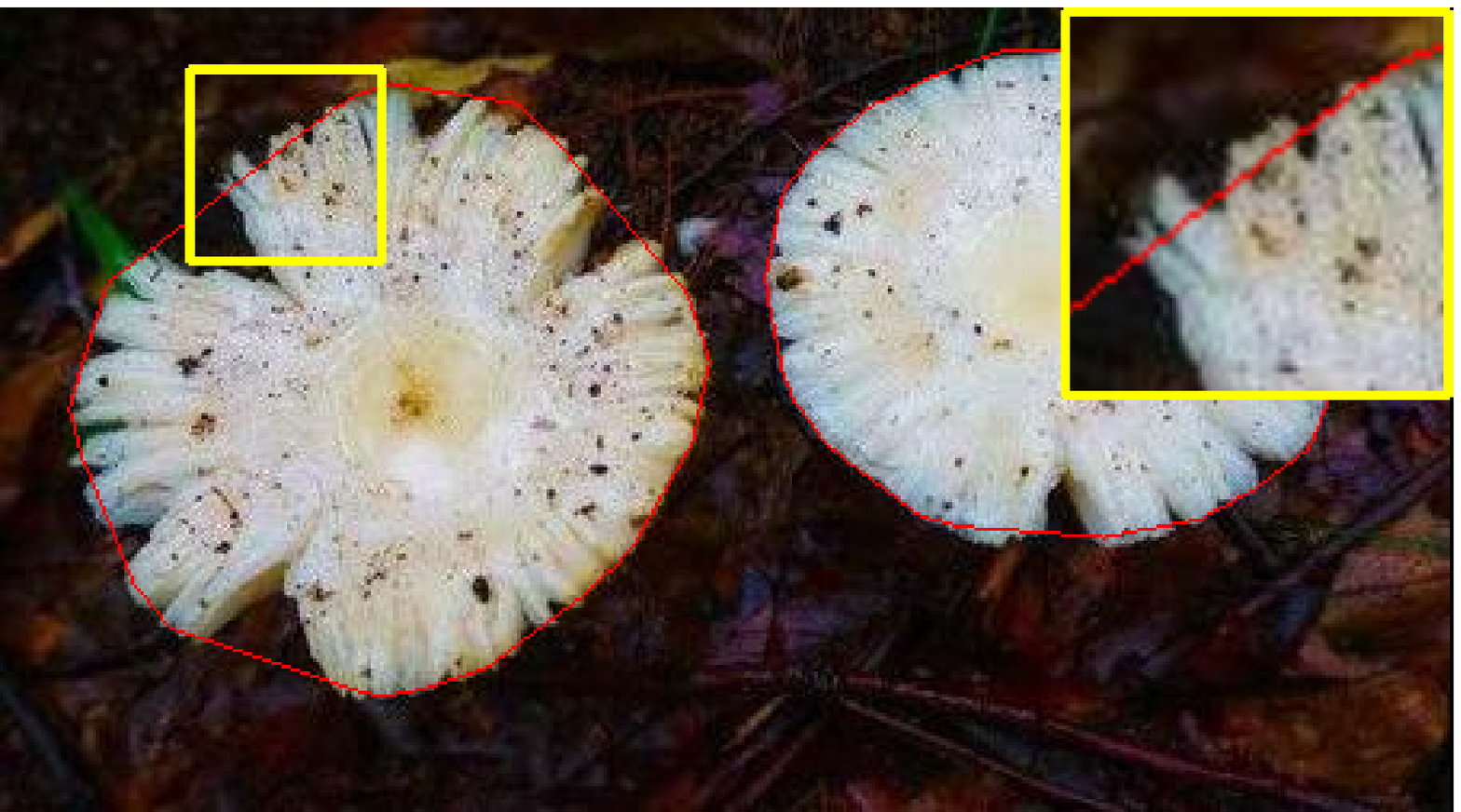}
	\includegraphics[height=1.6cm, width=1.6cm]{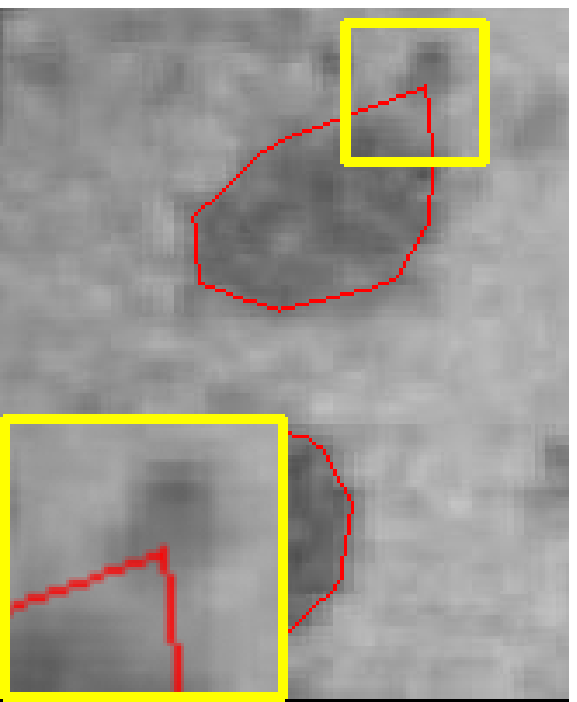}\\
	
	\includegraphics[height=1.6cm, width=1.6cm]{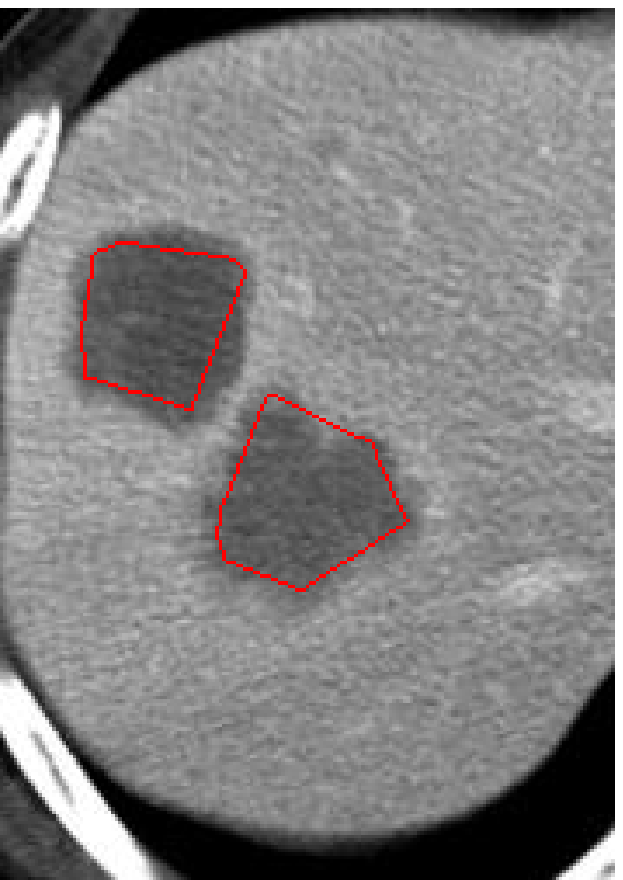}
	\includegraphics[height=1.6cm, width=1.6cm]{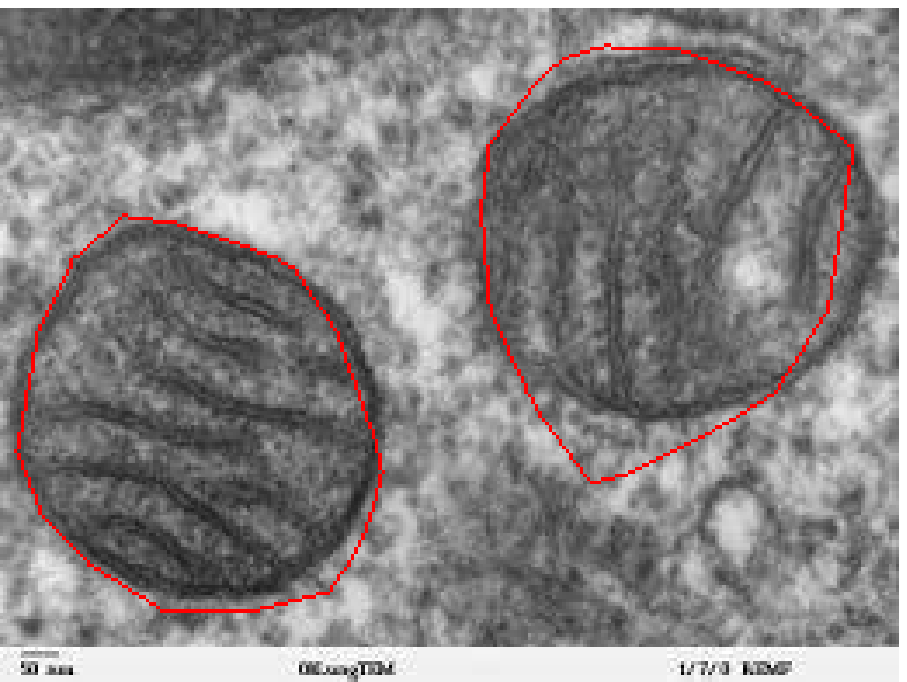}
	\includegraphics[height=1.6cm, width=1.6cm]{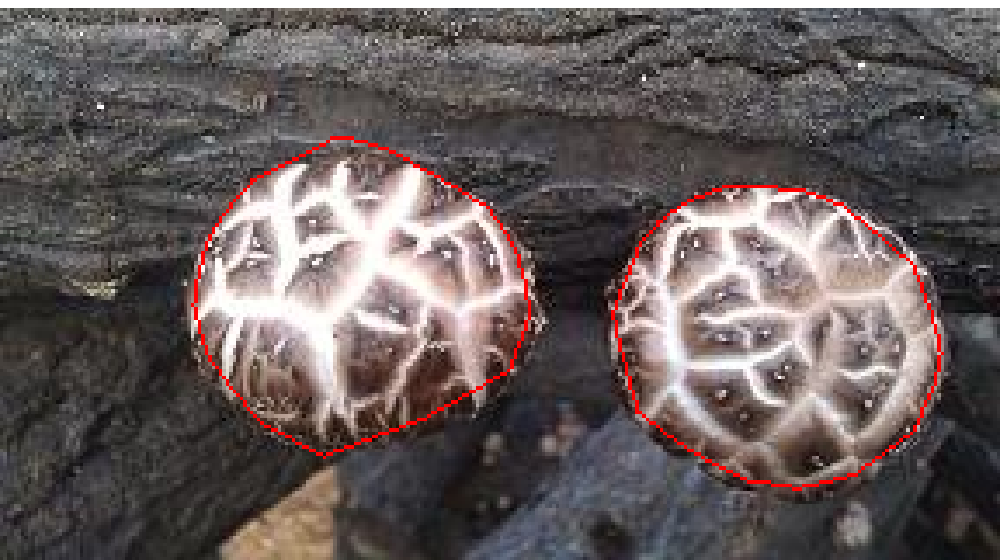}
	\includegraphics[height=1.6cm, width=1.6cm]{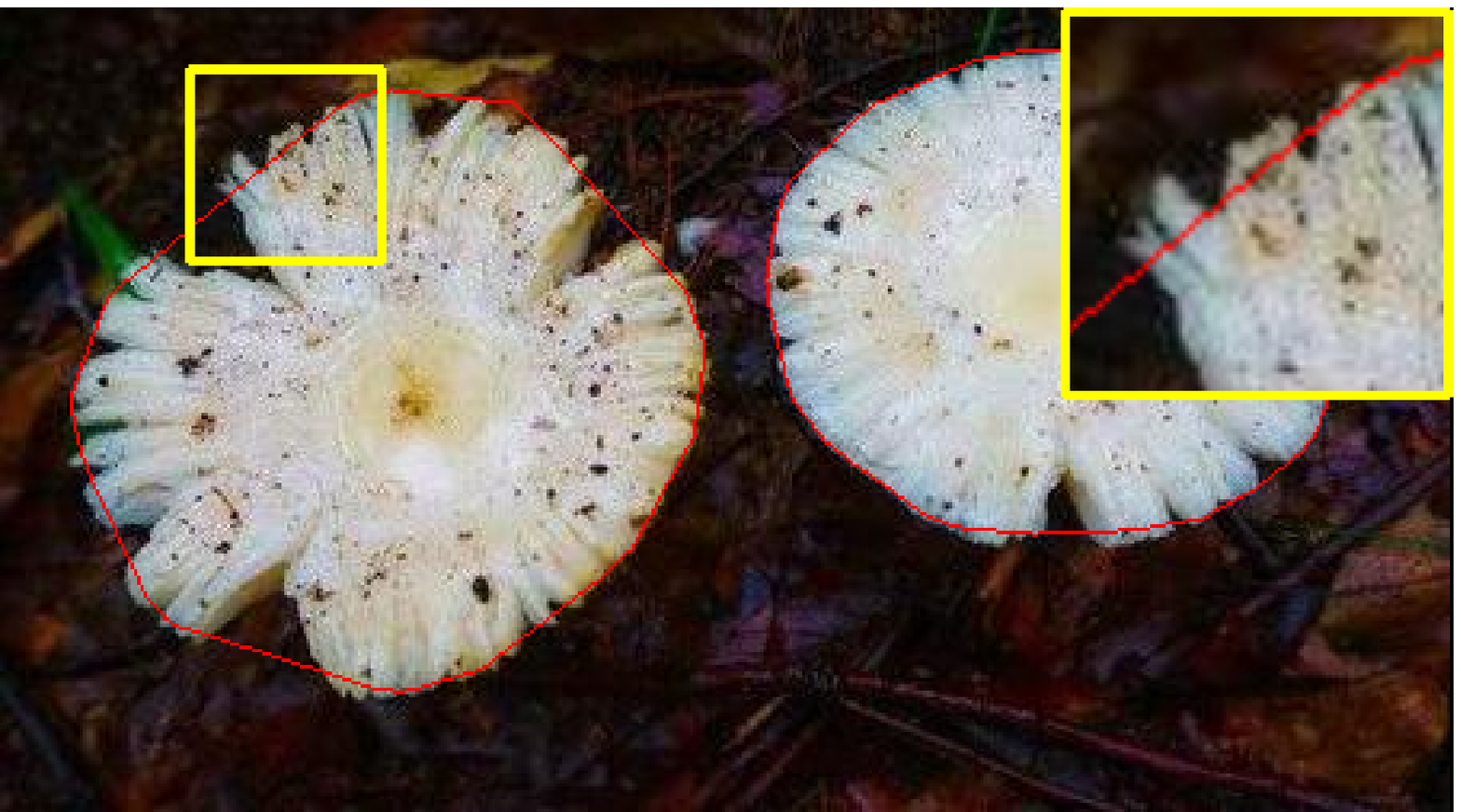}
	\includegraphics[height=1.6cm, width=1.6cm]{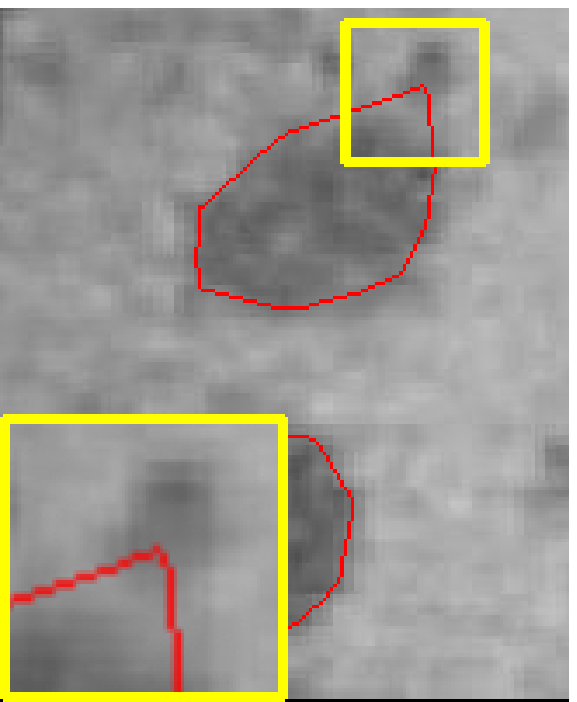}\\
	
	\includegraphics[height=1.6cm, width=1.6cm]{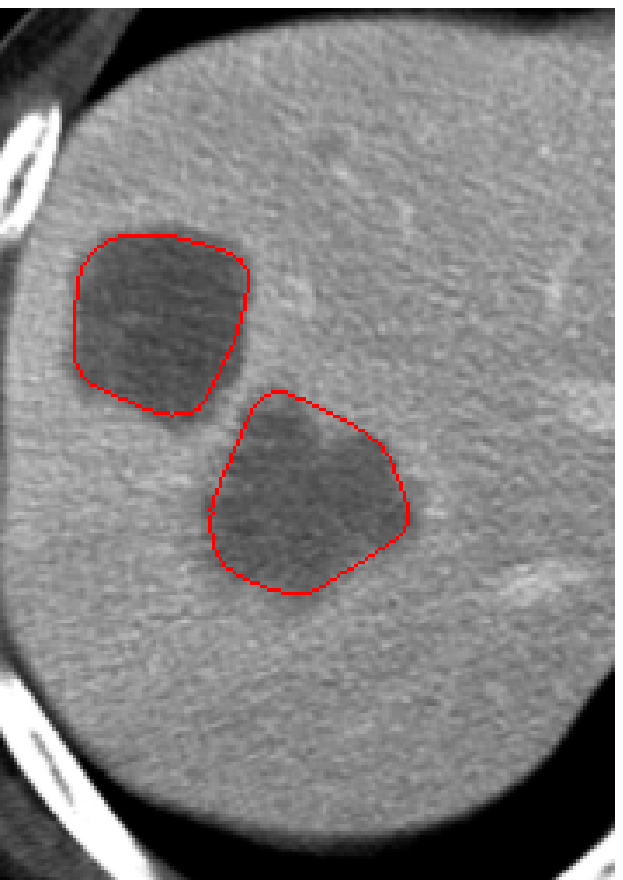}
	\includegraphics[height=1.6cm, width=1.6cm]{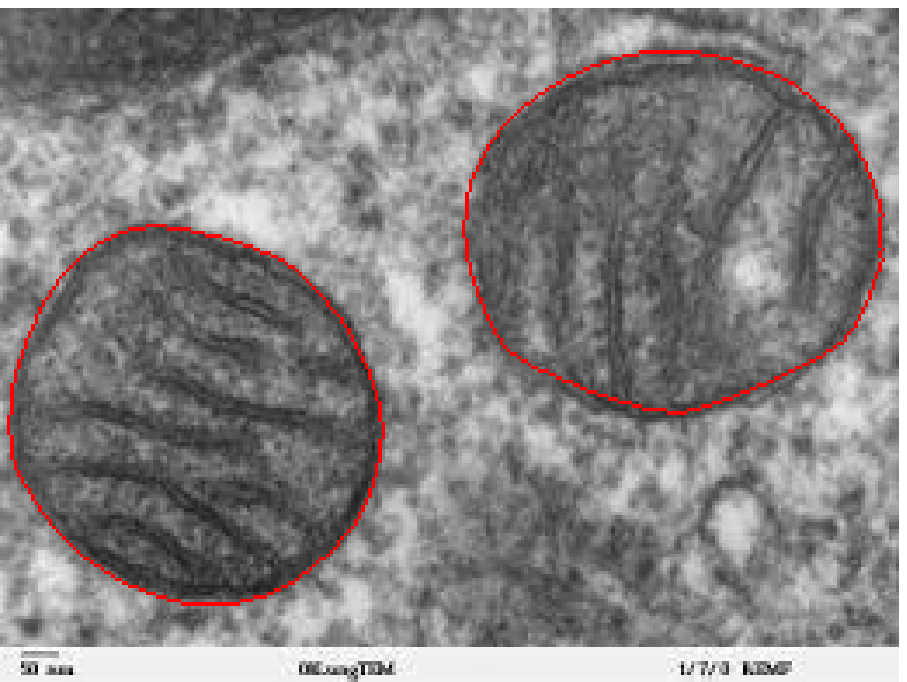}
	\includegraphics[height=1.6cm, width=1.6cm]{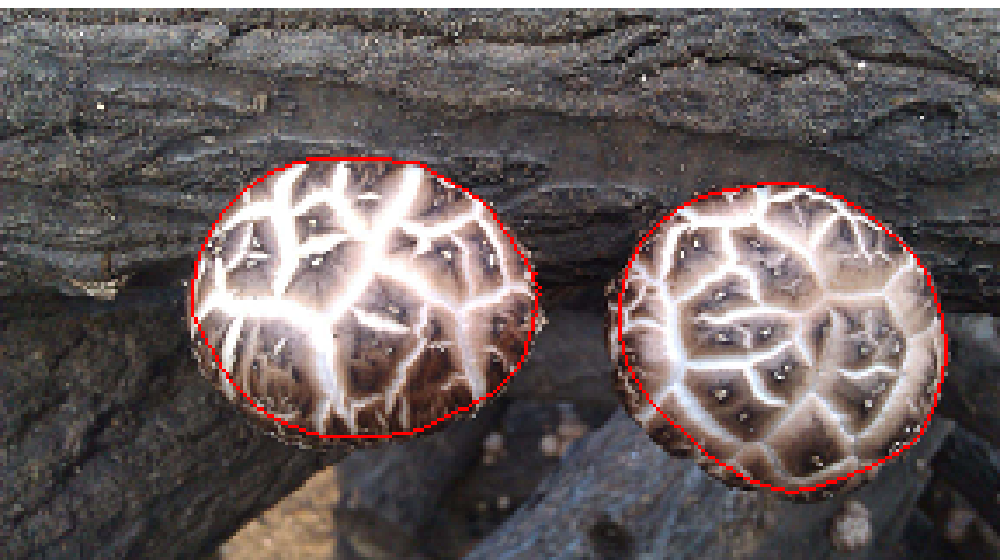}
	\includegraphics[height=1.6cm, width=1.6cm]{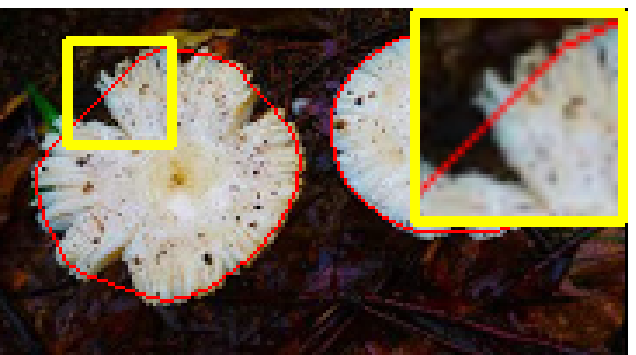}
	\includegraphics[height=1.6cm, width=1.6cm]{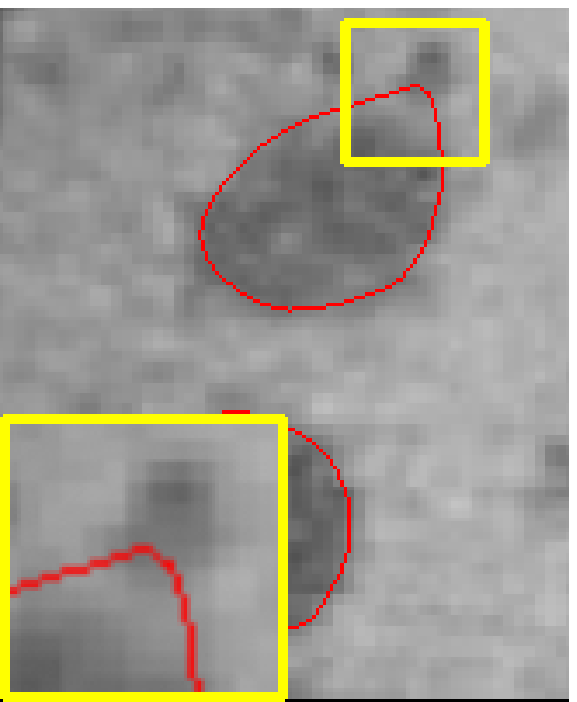}
	
	\includegraphics[height=1.6cm, width=1.6cm]{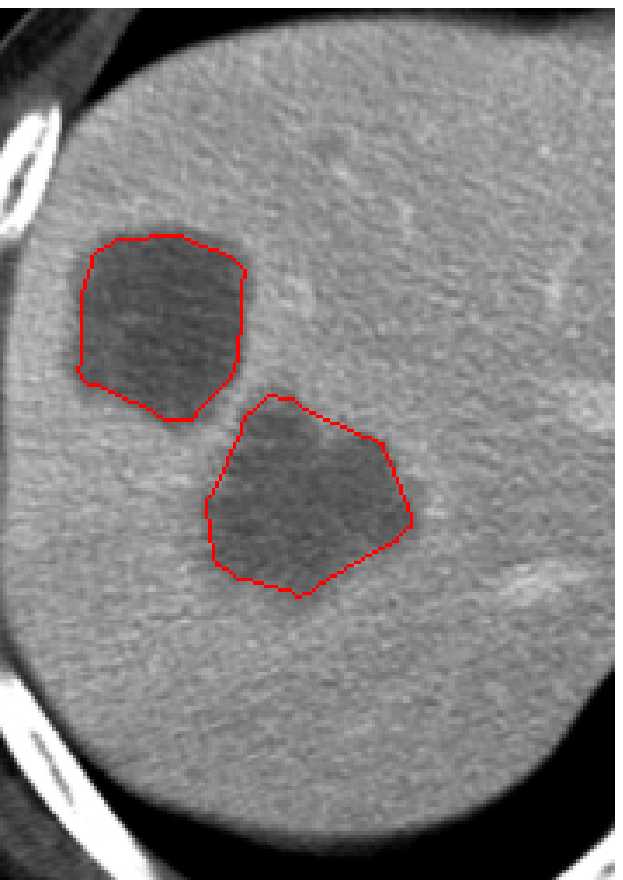}
	\includegraphics[height=1.6cm, width=1.6cm]{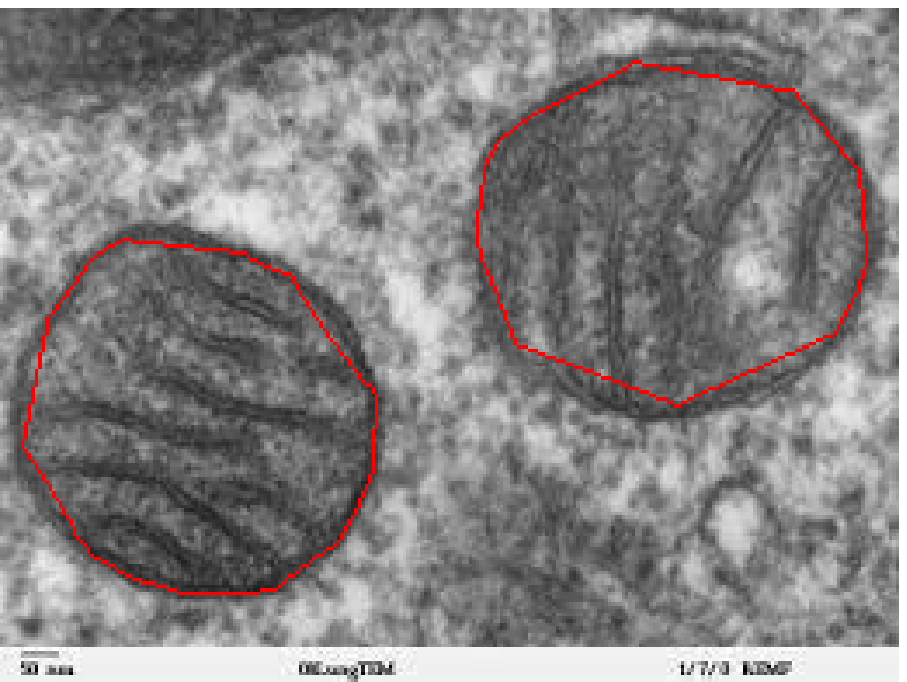}
	\includegraphics[height=1.6cm, width=1.6cm]{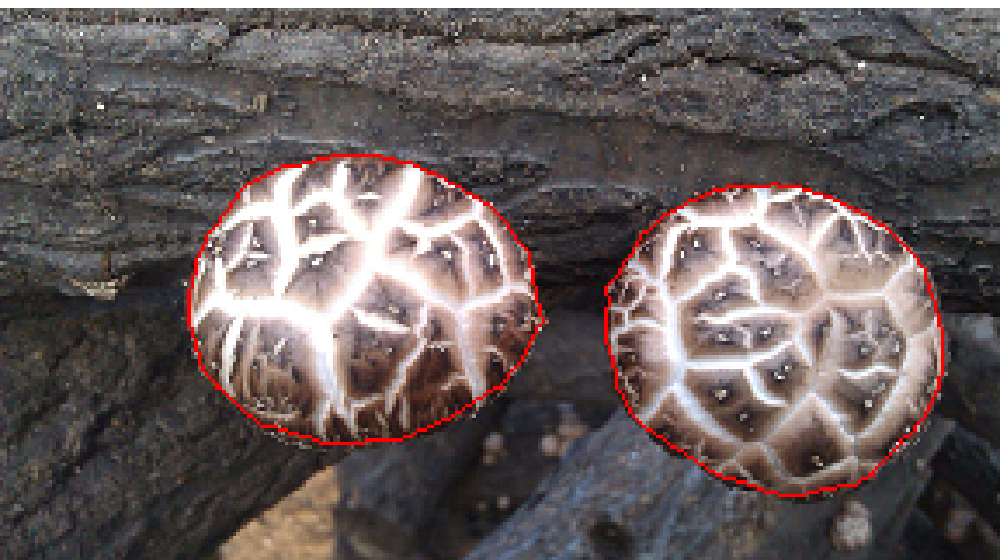}
	\includegraphics[height=1.6cm, width=1.6cm]{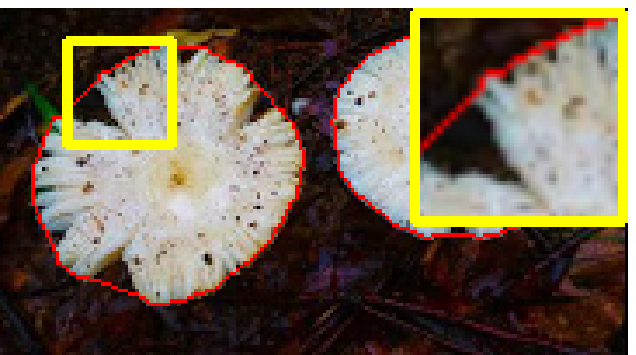}
	\includegraphics[height=1.6cm, width=1.6cm]{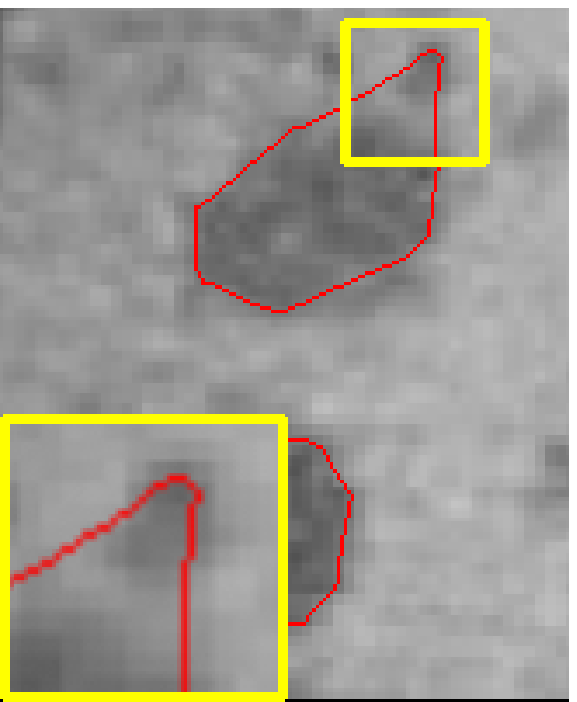}\\
	\caption{The images with labels and the results by 101(5), 101(11), LS and the proposed methods are shown from top to bottom. Parts of interest are zoomed. }
	\label{fig:muti}
\end{figure}

\begin{table*}[!t]
	\caption{Average shape-distances of the results and computing time for different methods.}
	\label{tab:mult_shdist_time}
	\centering
	\begin{tabular}{c|cccc|cccc}
		\hline
		&\multicolumn{4}{c}{Shape-distances (\%)} &\multicolumn{4}{|c}{Time (seconds)}\\
		\hline
		Methods&101(5) &101(11) &LS & Ours &101(5) &101(11) & LS & Ours\\
		
		avg& 19.961 &11.627 &5.939 &\pmb{5.809} &{183.3} &666.5 & 47.8&{154.6} \\
		\hline
	\end{tabular}
\end{table*}
Comparing the visual and quantitative accuracy in Fig. \ref{fig:muti} and Tab. \ref{tab:mult_shdist_time},
we can clearly observe that the proposed method can extract the objects and simultaneously preserve their convexity,
but the 1-0-1 and LS methods fail for some of them.
Taking the last two images in Fig. \ref{fig:muti} as examples,
the 1-0-1 method can't
catch the objects completely (e.g. the zooming of the left mushroom)
the both convex objects of interest.
The faults of the 1-0-1 method can be observed for other images as well.
For the LS method, the superiority of the proposed method can be observed by comparing the results of the last two images in Fig. \ref{fig:muti}, where
the LS method can not catch the object completely (e.g. the zooming of the left mushroom).
In addition, the average shape-distance for the proposed method is $5.809\%$, which is one half of the one obtained by the 101(5) and 101(11) methods and less than the one by LS method. 

As for the computational time,
our proposed method is more efficient than the 1-0-1 method, and it is at least $4$ times faster than 101(11) from Tab. \ref{tab:mult_shdist_time}.
This experiment illustrates that our proposed method is superior to the 1-0-1 method in the sense of segmentation accuracy and computation efficiency further. It is observed that the
computational time for the LS method is much less than the proposed method. The reason for this is that the LS method use only one level set function for multiple objects segmentation, i.e. the computation time does not increase with the number of objects, while the computational time of the proposed method will increase dramatically with  the increasing of the number of objects because the
multiple binary functions are need. Factually, the proposed method can use only one binary function to represent multiple objects (see the experiments for convex hull computation), which will not increase the computational time. For this aspect, we will discuss in the future.

\subsubsection{Segmentation of convex ring object}
Convex ring, i.e. the curvature of the  outer  (inner) boundary
is positive (negative), is common in real life (see Fig. \ref{fig:ring_label}).
Although one can derive a more effective representation method for the convex ring shape using the result in Theorem \ref{th:single},
we apply the proposed method to convex ring extraction directly: segment the
inner background (convex) firstly to catch the inner boundary, and
then segment the combination of ring shape and the inner background to catch the outer boundary.

The test images are shown in the first row of Fig. \ref{fig:ring_label}, where the labels on the
inner and outer background and the object are subscribed in different colors.
The segmentation results are shown in the second row of Fig. \ref{fig:ring_label}.
We can see that boundaries of the convex ring shape are extracted accurately, and the convexity of the boundaries is preserved.
\begin{figure}[!t]
	\centering
	\includegraphics[height=1.5cm, width=1.5cm]{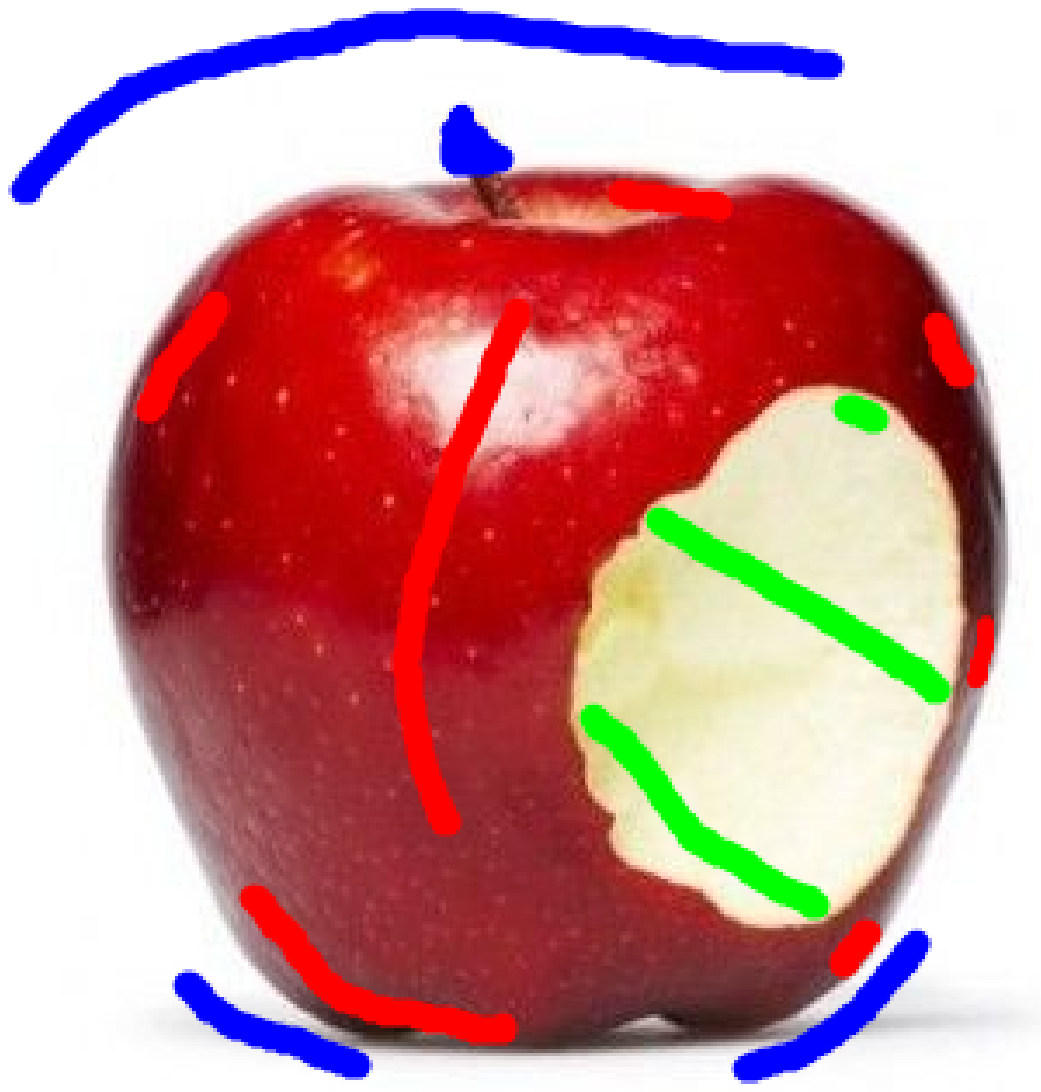}
	\includegraphics[height=1.5cm, width=1.5cm]{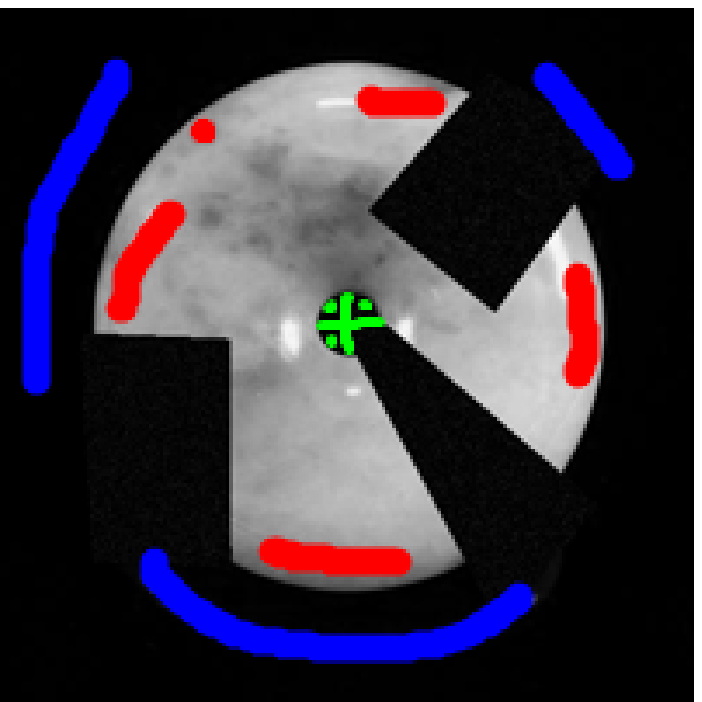}
	\includegraphics[height=1.5cm, width=1.5cm]{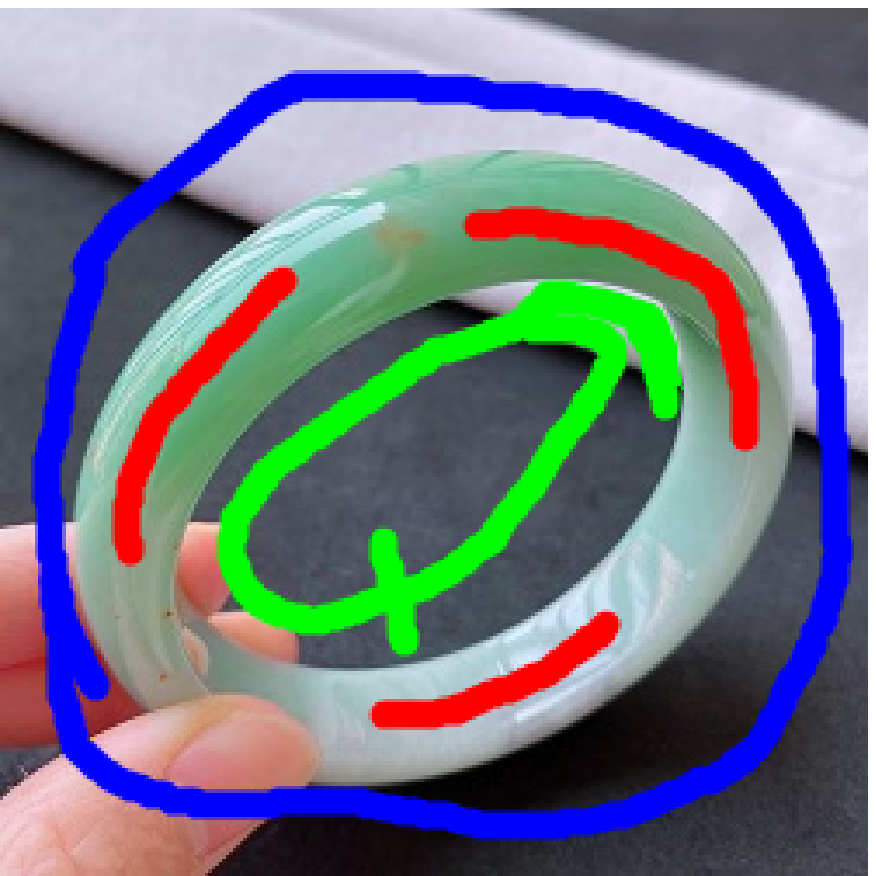}
	\includegraphics[height=1.5cm, width=1.5cm]{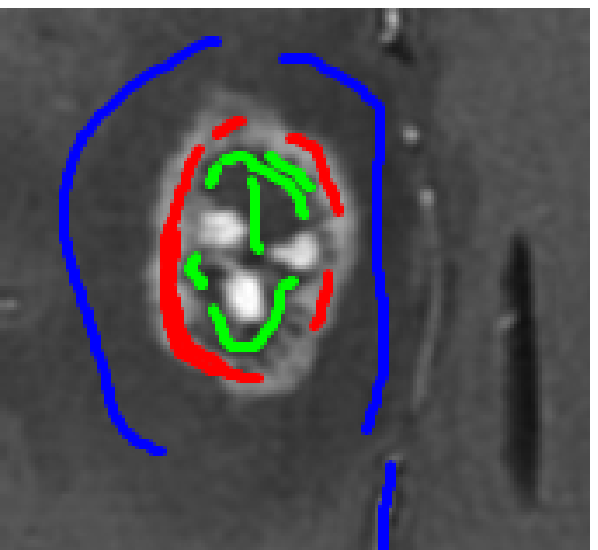}
	\includegraphics[height=1.5cm, width=1.5cm]{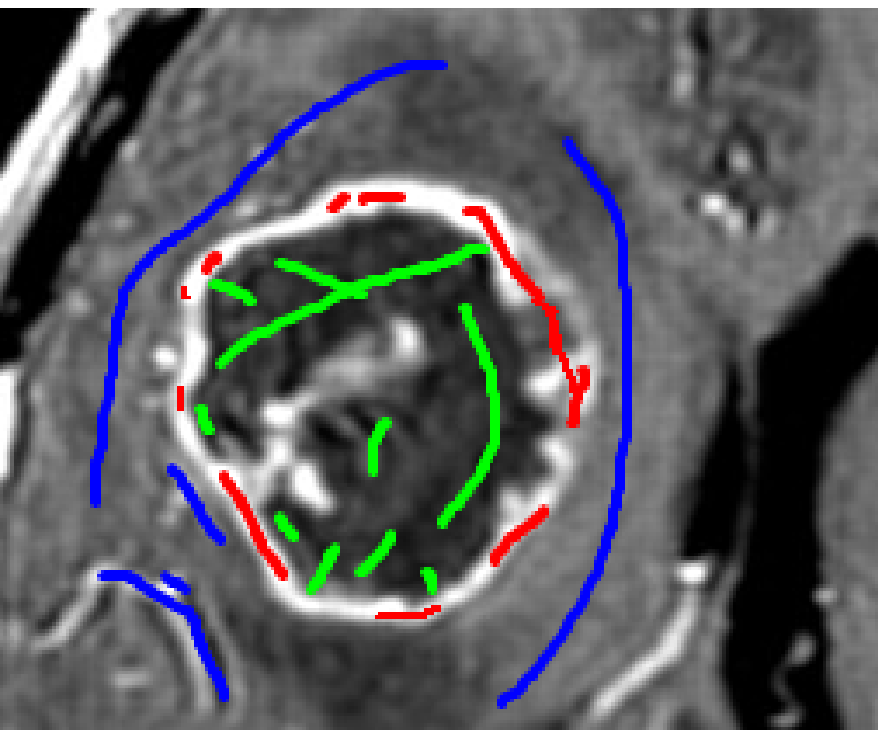}
	\includegraphics[height=1.5cm, width=1.5cm]{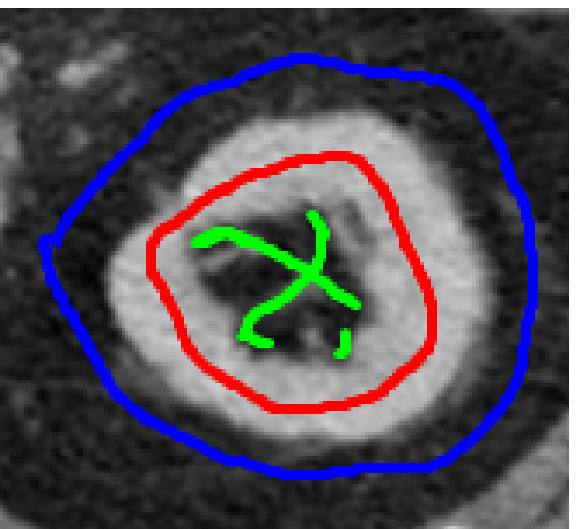}
	\includegraphics[height=1.5cm, width=1.5cm]{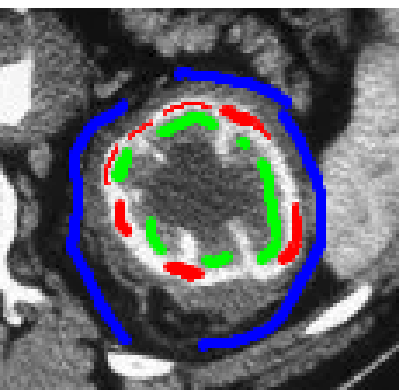}
	\includegraphics[height=1.5cm, width=1.5cm]{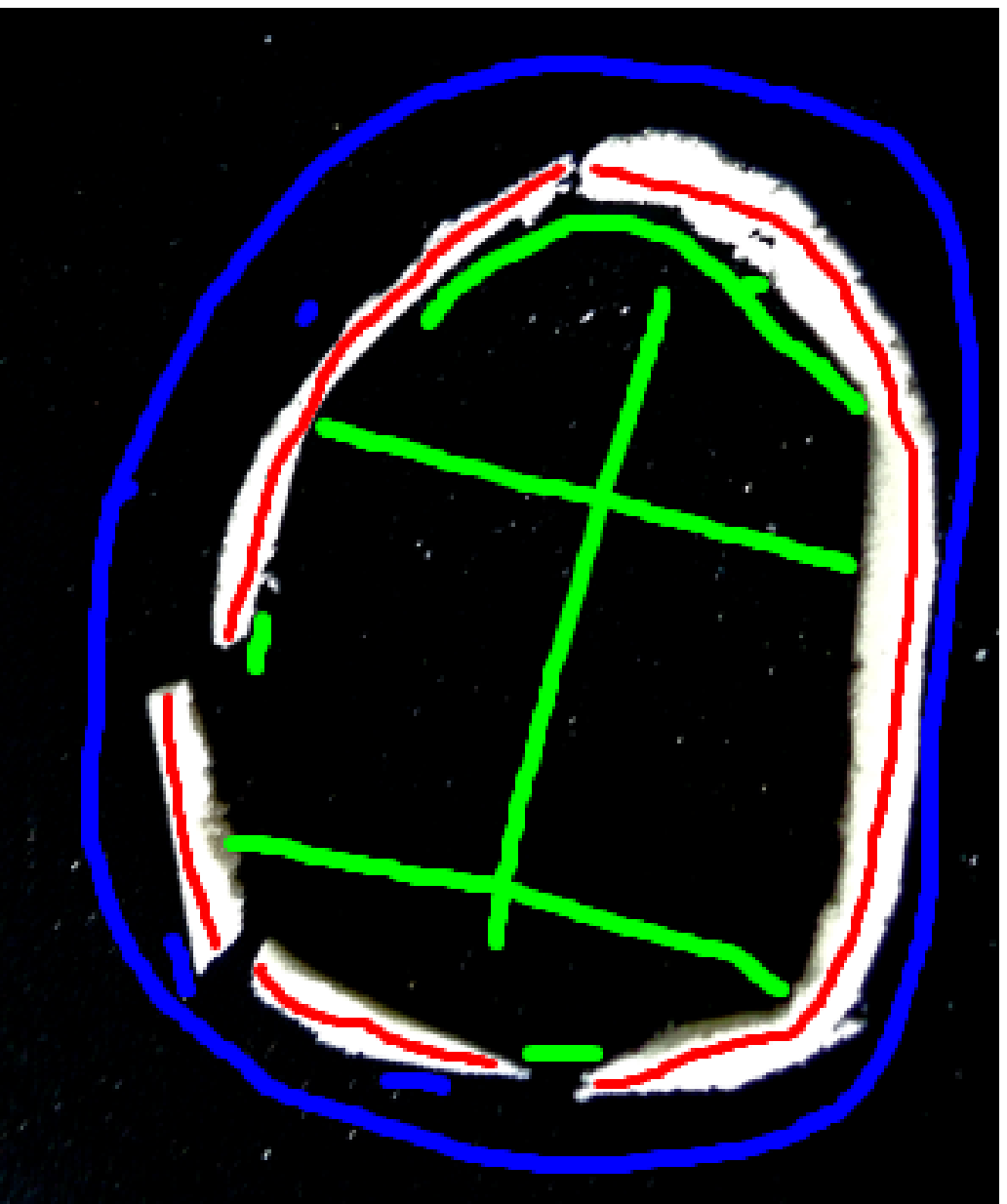}
	\includegraphics[height=1.5cm, width=1.5cm]{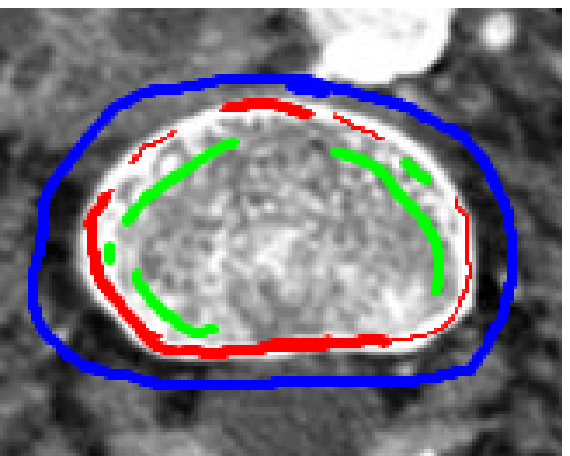}
	\includegraphics[height=1.5cm, width=1.5cm]{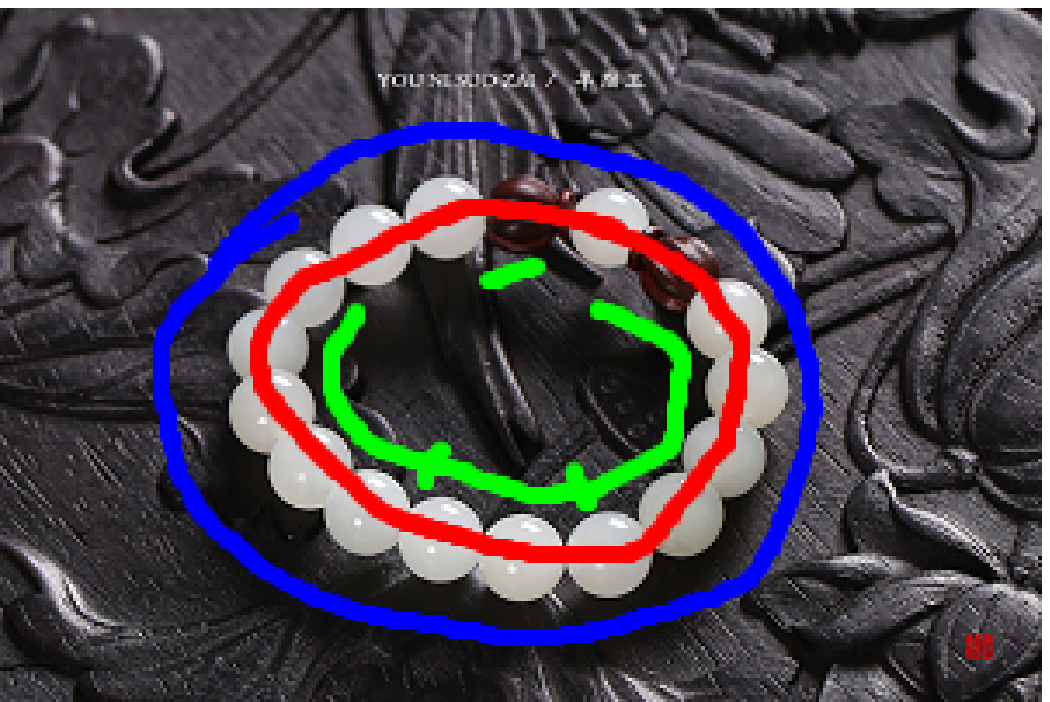}\\
	\includegraphics[height=1.5cm, width=1.5cm]{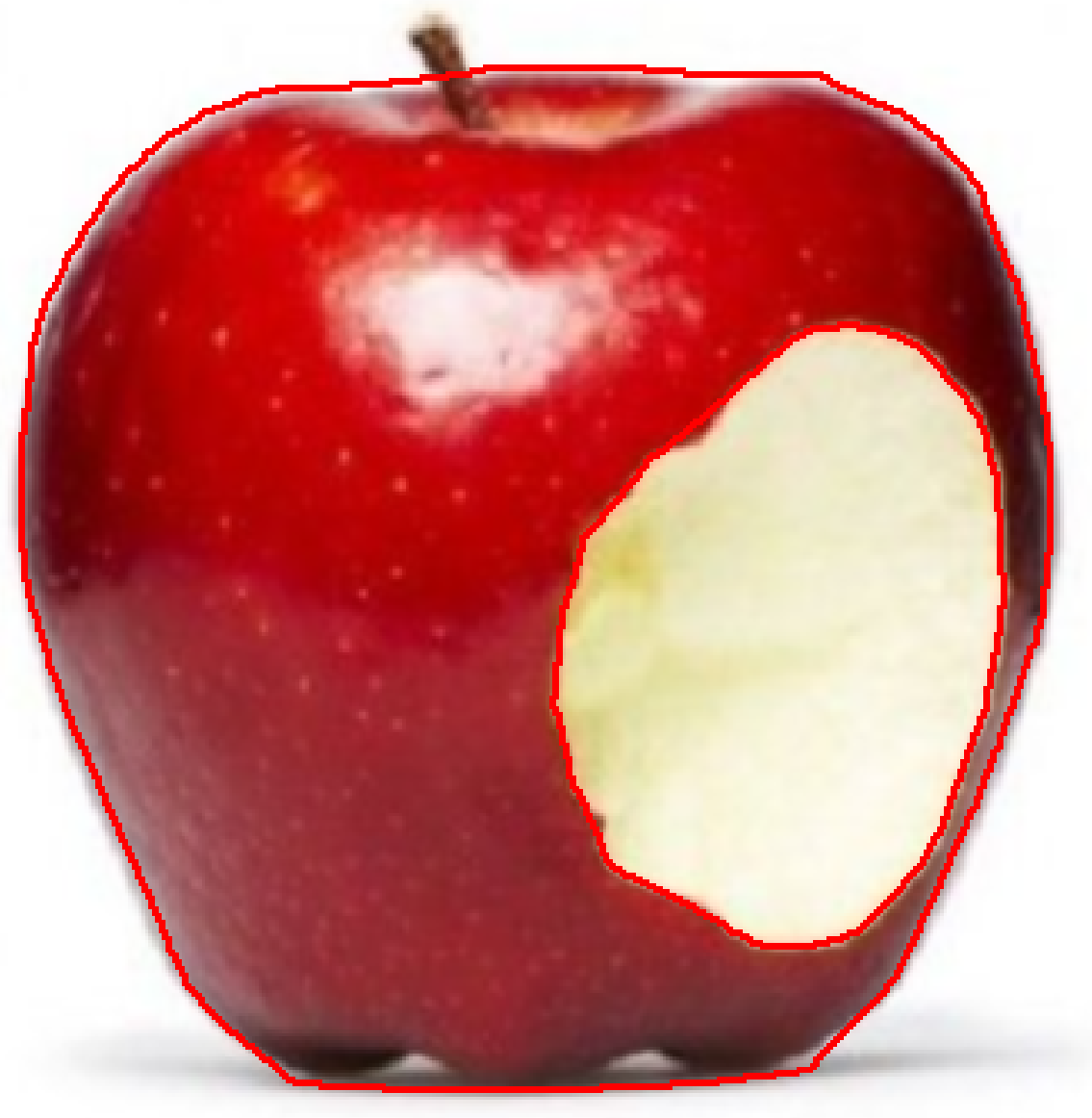}
	\includegraphics[height=1.5cm, width=1.5cm]{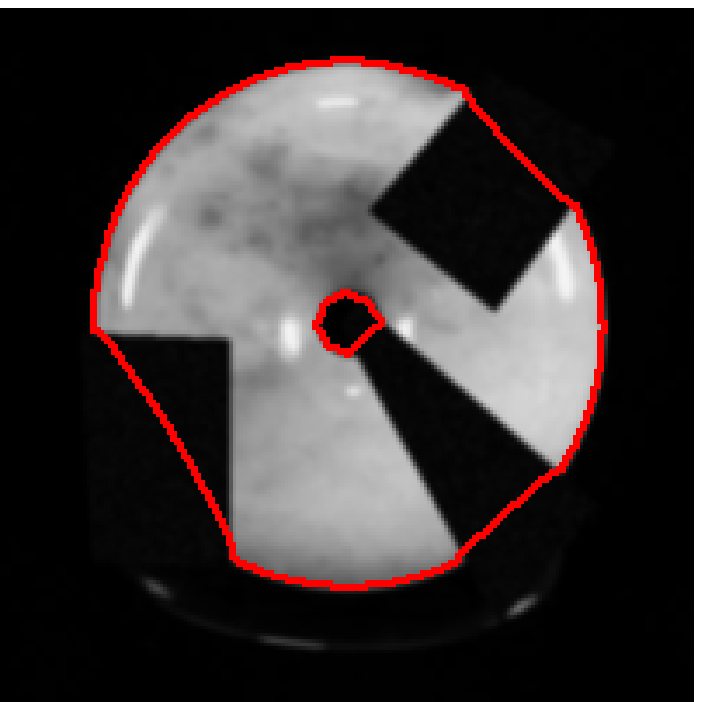}
	\includegraphics[height=1.5cm, width=1.5cm]{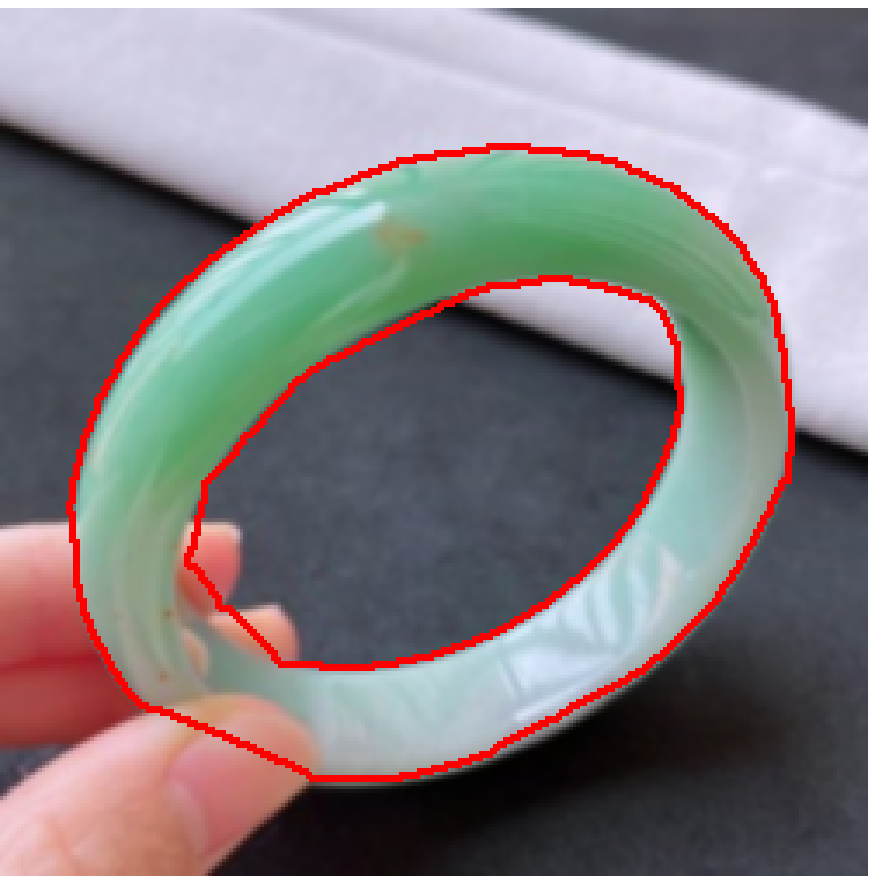}
	\includegraphics[height=1.5cm, width=1.5cm]{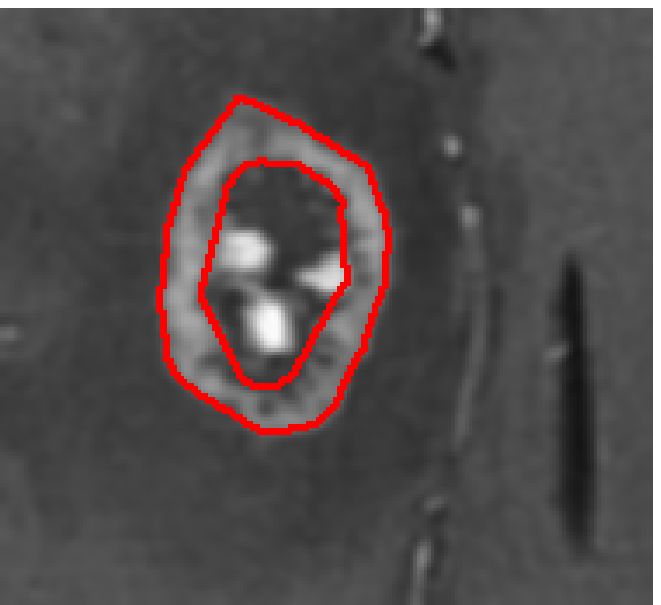}
	\includegraphics[height=1.5cm, width=1.5cm]{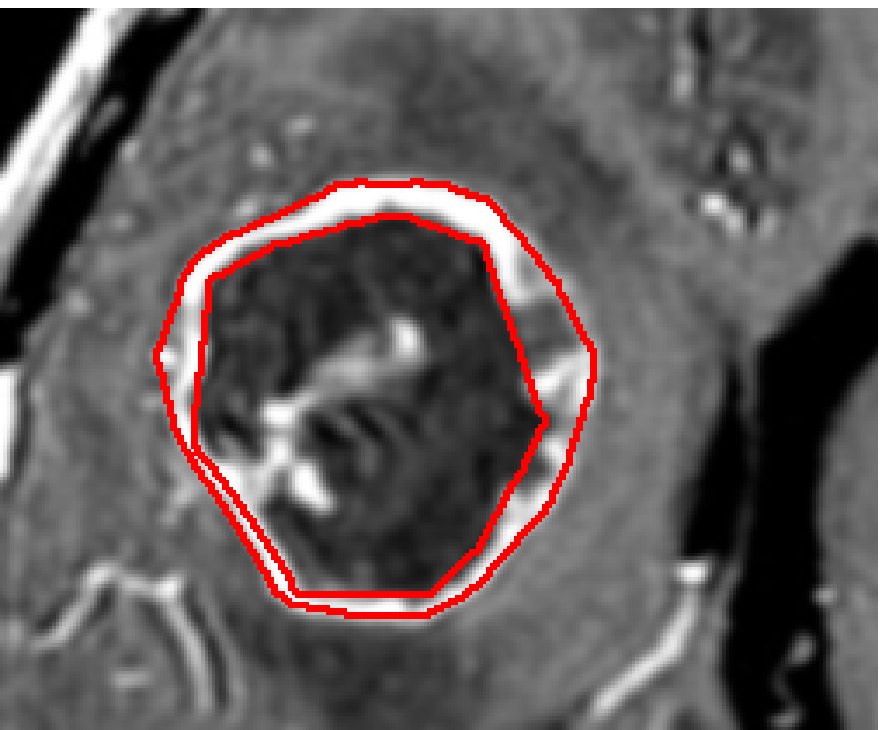}
	\includegraphics[height=1.5cm, width=1.5cm]{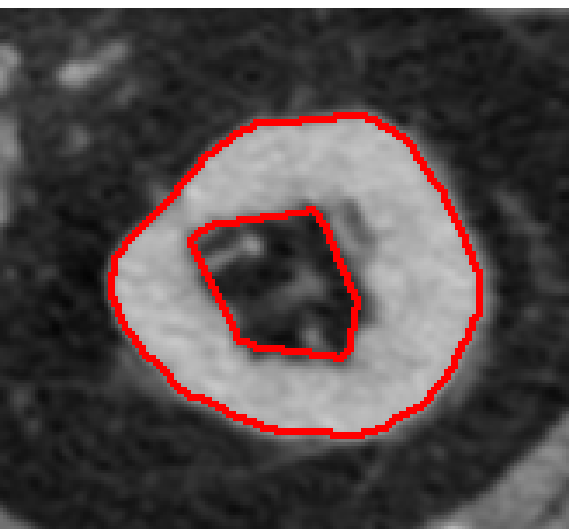}
	\includegraphics[height=1.5cm, width=1.5cm]{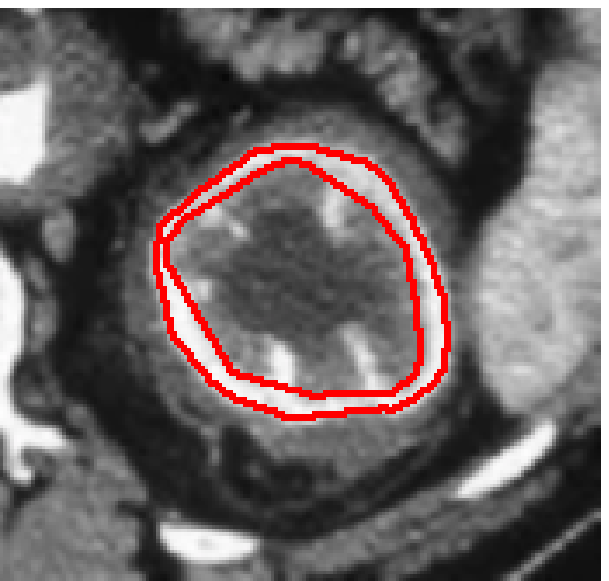}
	\includegraphics[height=1.5cm, width=1.5cm]{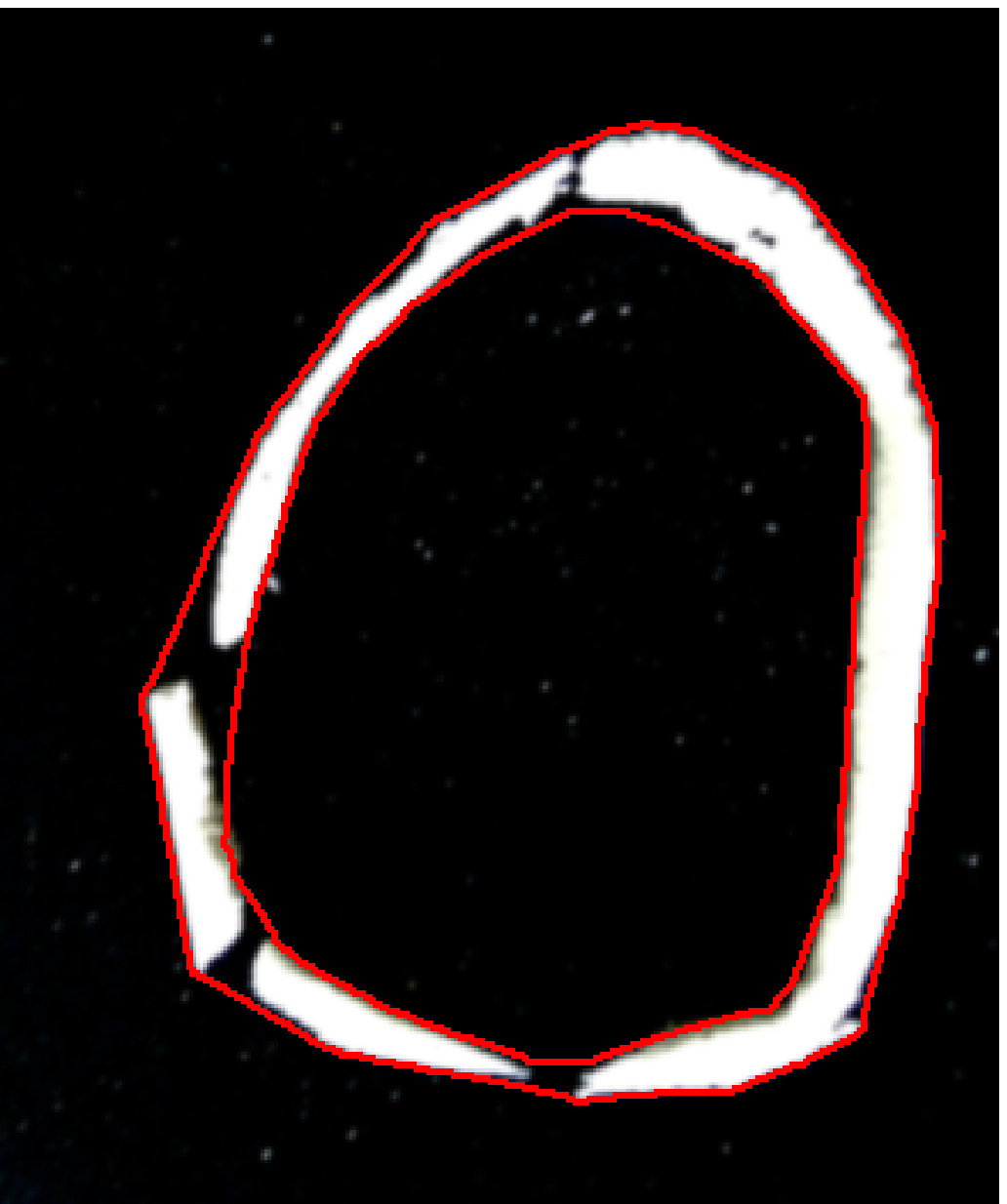}
	\includegraphics[height=1.5cm, width=1.5cm]{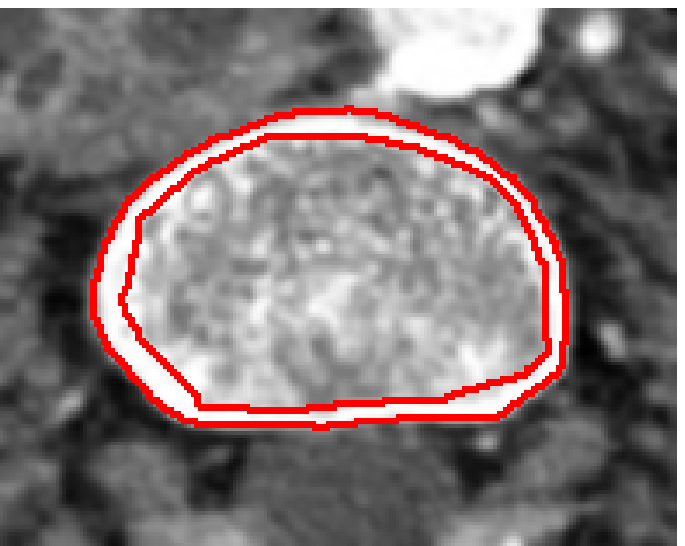}
	\includegraphics[height=1.5cm, width=1.5cm]{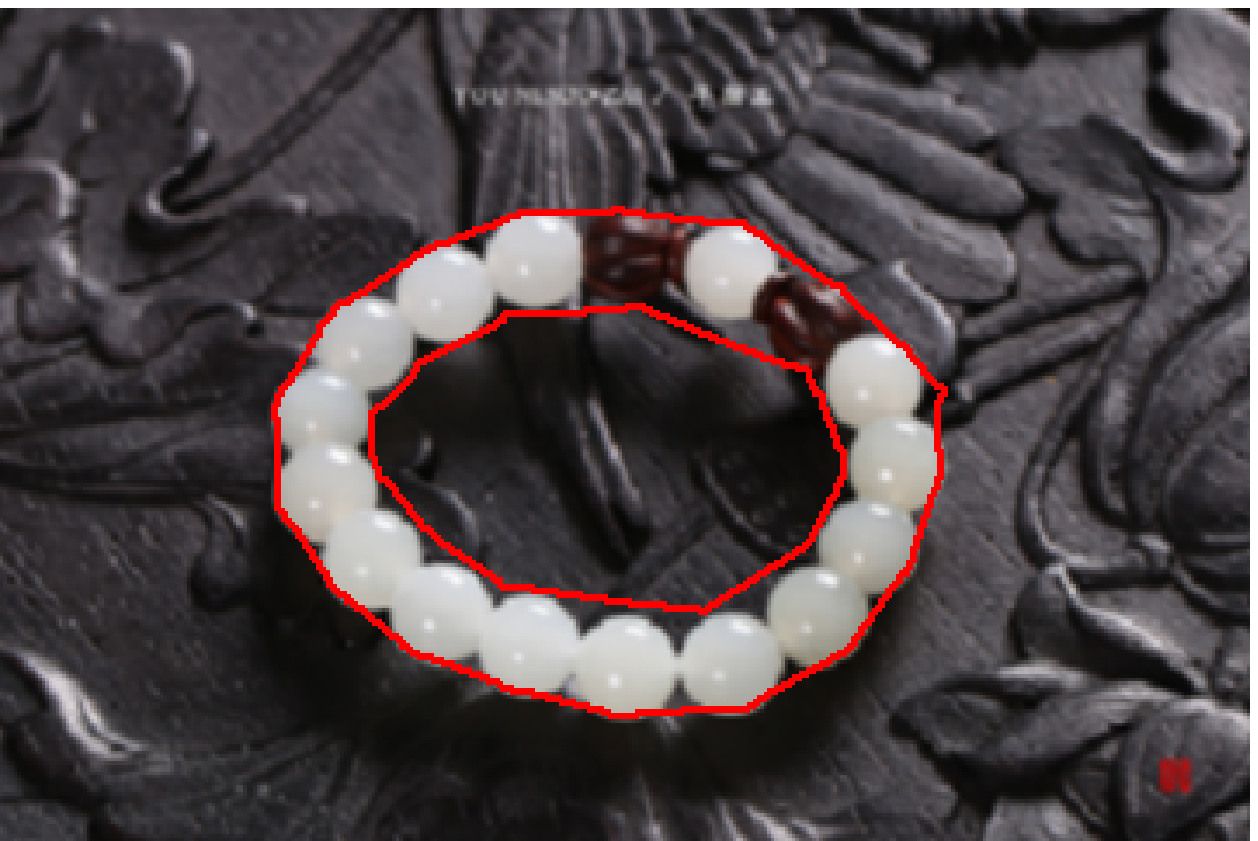}
	\caption{Tested images with labels and segmentation results correspoding.}
	\label{fig:ring_label}
\end{figure}

\subsubsection{Experiments using the interactive procedure}
The proposed interactive procedure is performed on some challenging images.  Because of weak and cluttered edges or nonuniform intensities of concerned objects (see Fig. \ref{fig:Interactive}),
it is necessary to apply the interactive procedure by
adding proper labels on the object and background to yield desirable result.
The labels and results are shown in Fig. \ref{fig:Interactive}, where the labels given at different steps are marked in different colors.

\begin{figure}[!t]
	\centering
	\includegraphics[height=1.6cm, width=2cm]{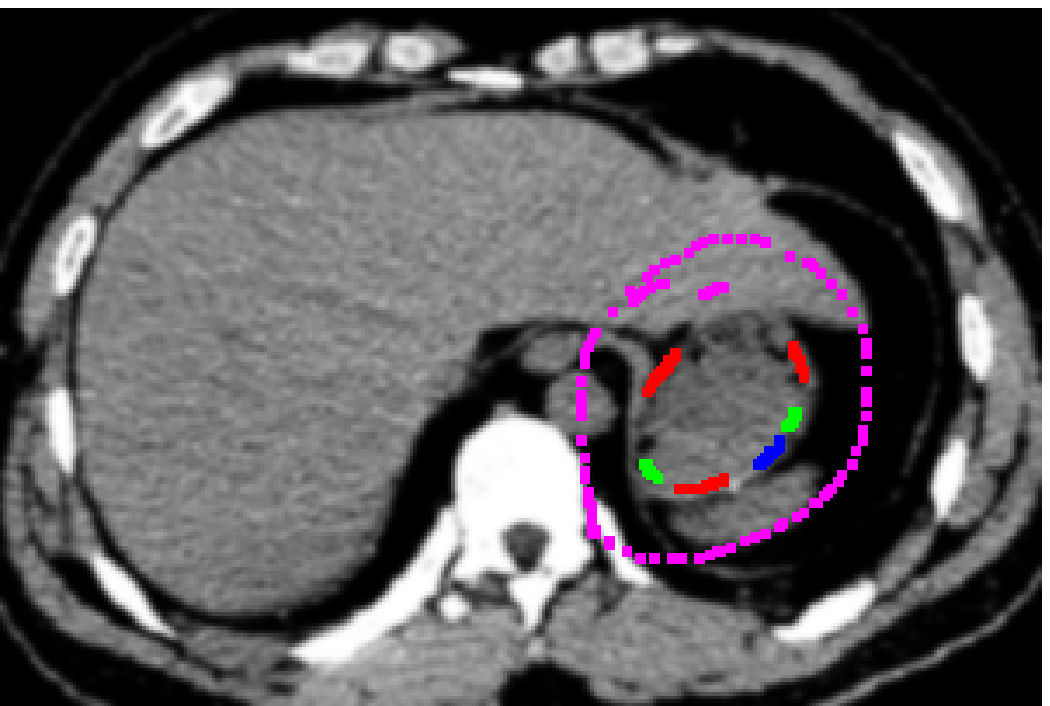}
	\includegraphics[height=1.6cm, width=2cm]{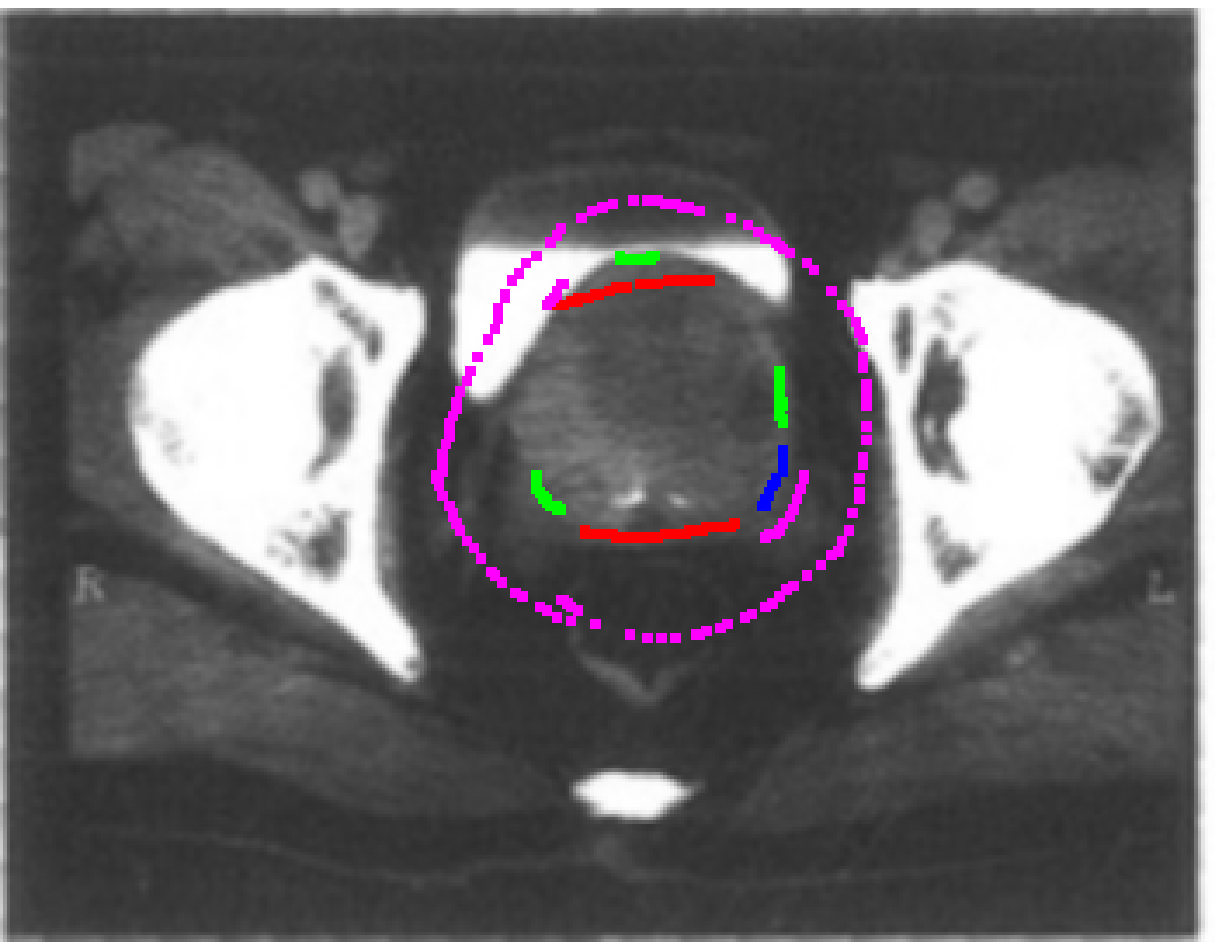}
	\includegraphics[height=1.6cm, width=2cm]{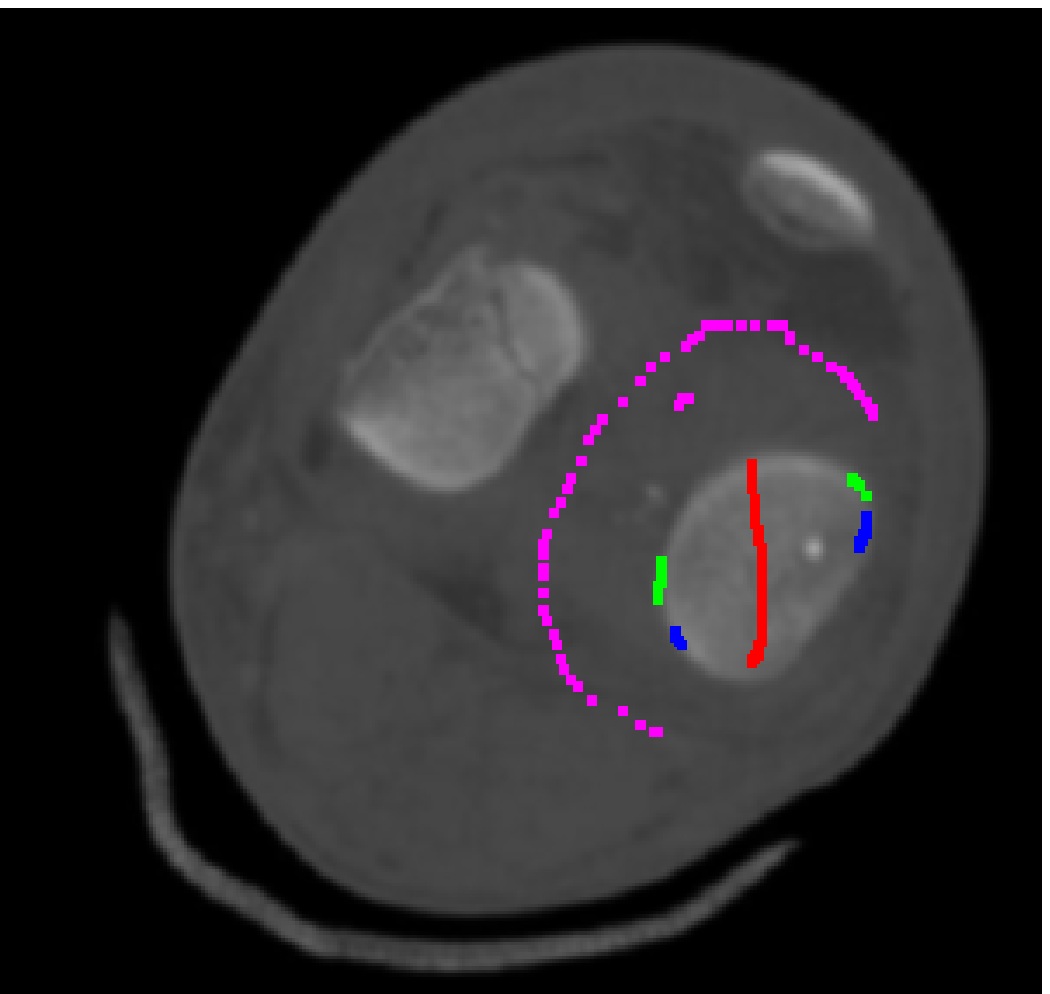}
	\includegraphics[height=1.6cm, width=2cm]{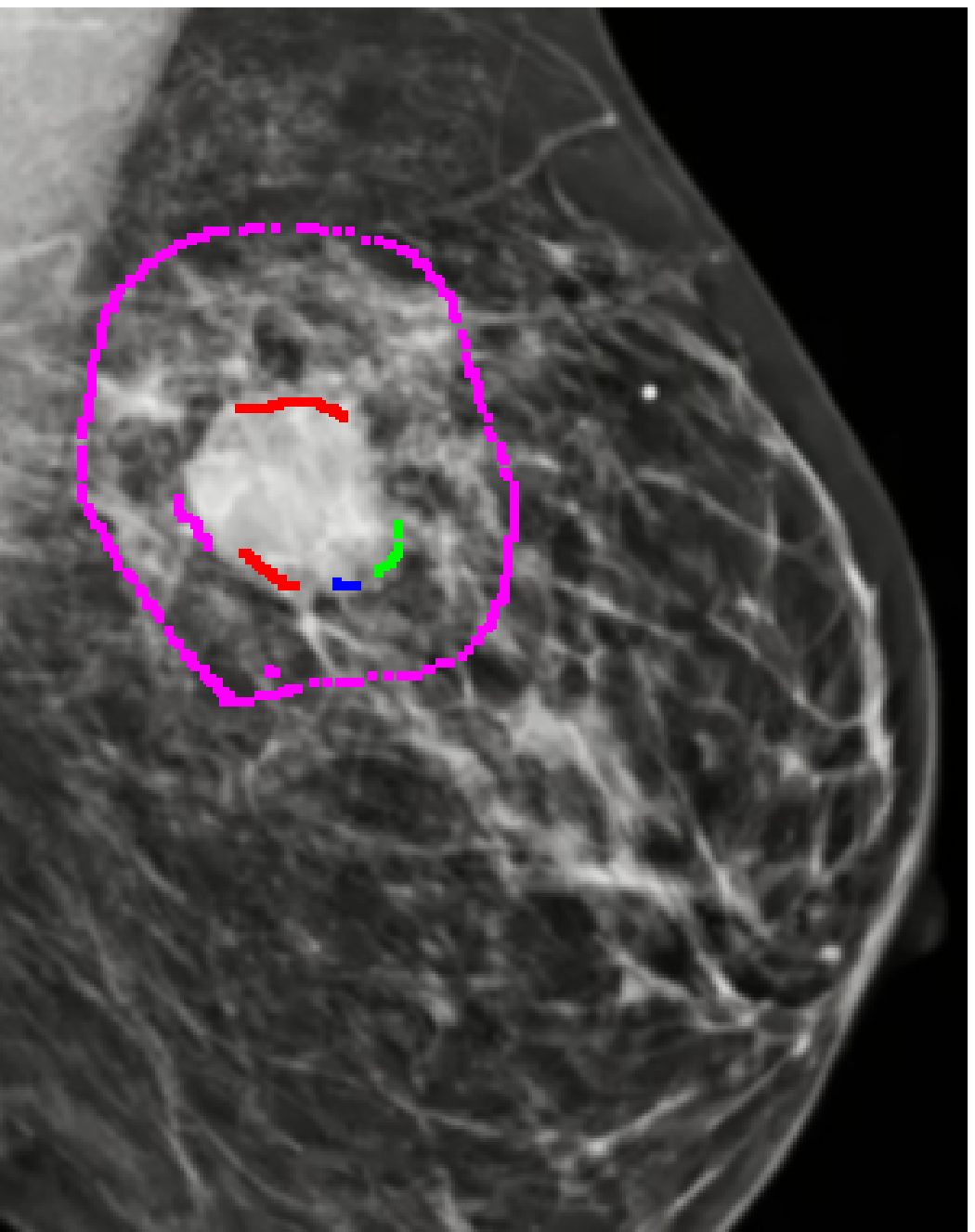}\\
	\includegraphics[height=1.6cm, width=2cm]{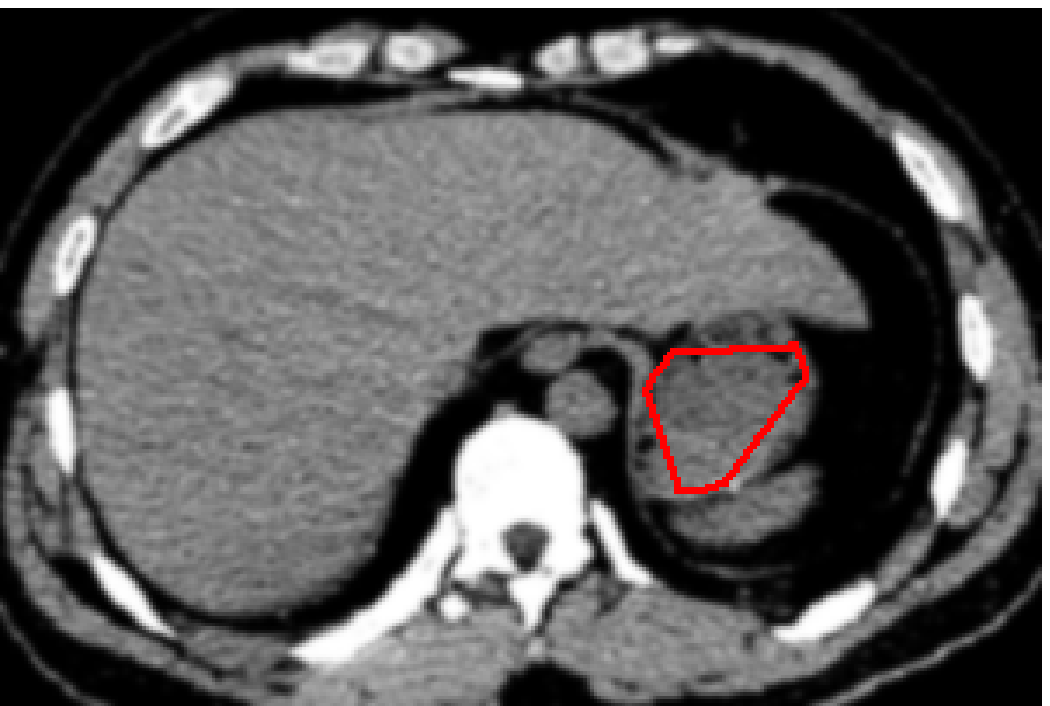}
	\includegraphics[height=1.6cm, width=2cm]{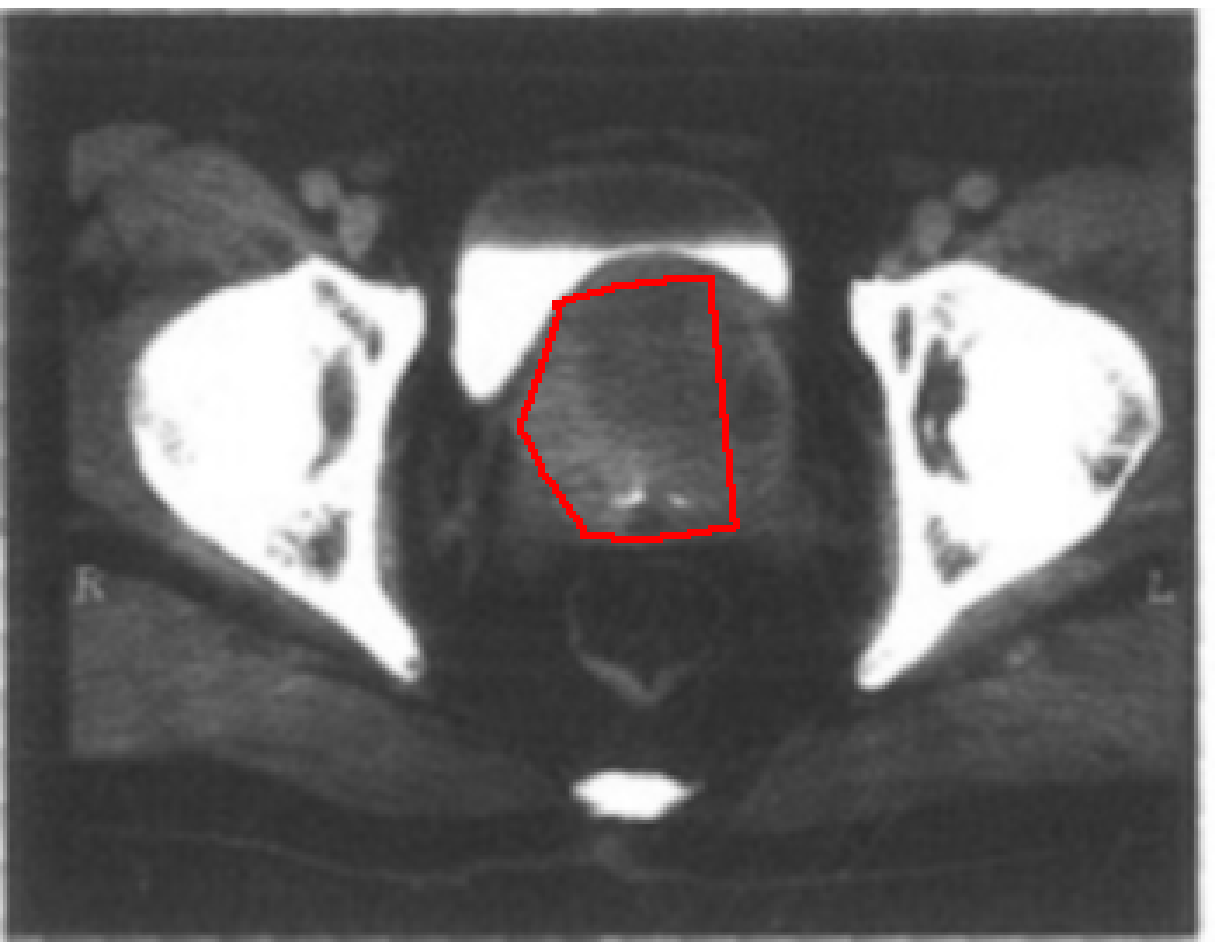}
	\includegraphics[height=1.6cm, width=2cm]{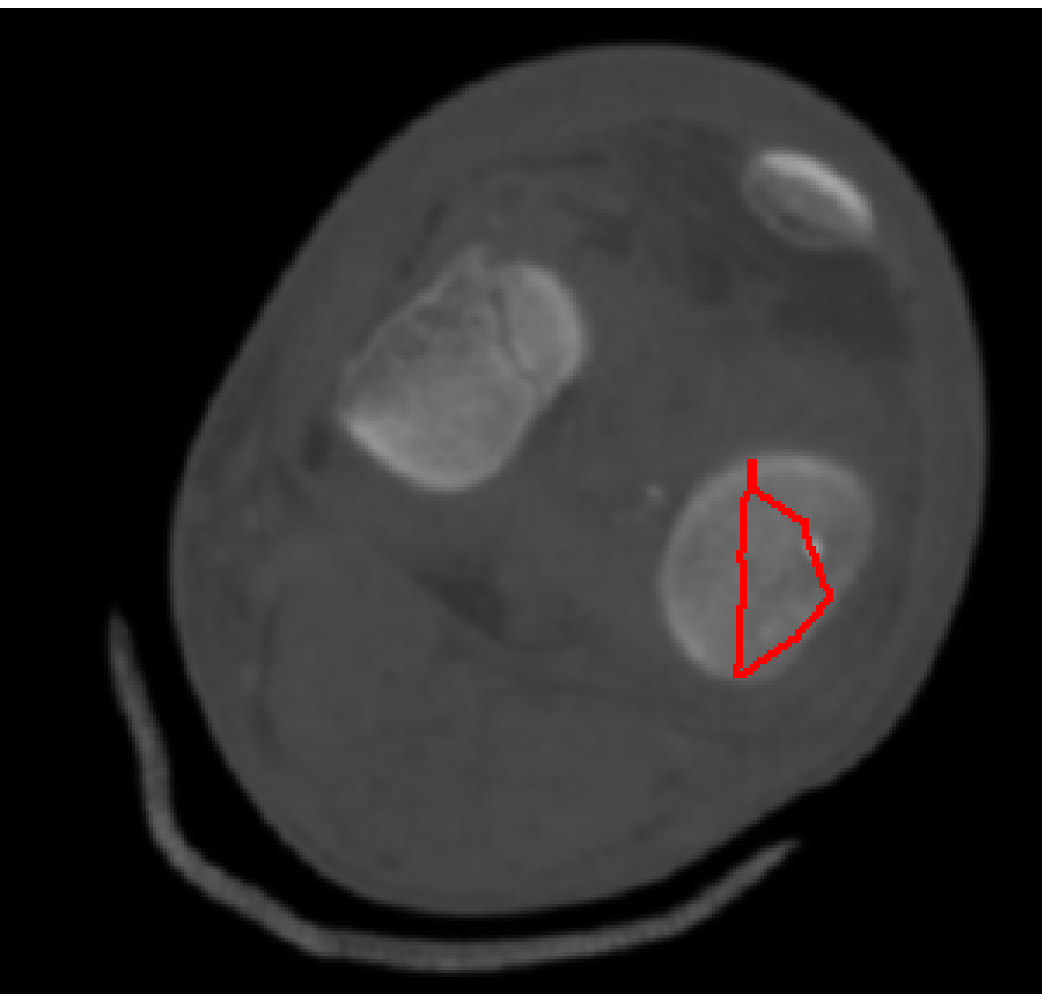}
	\includegraphics[height=1.6cm, width=2cm]{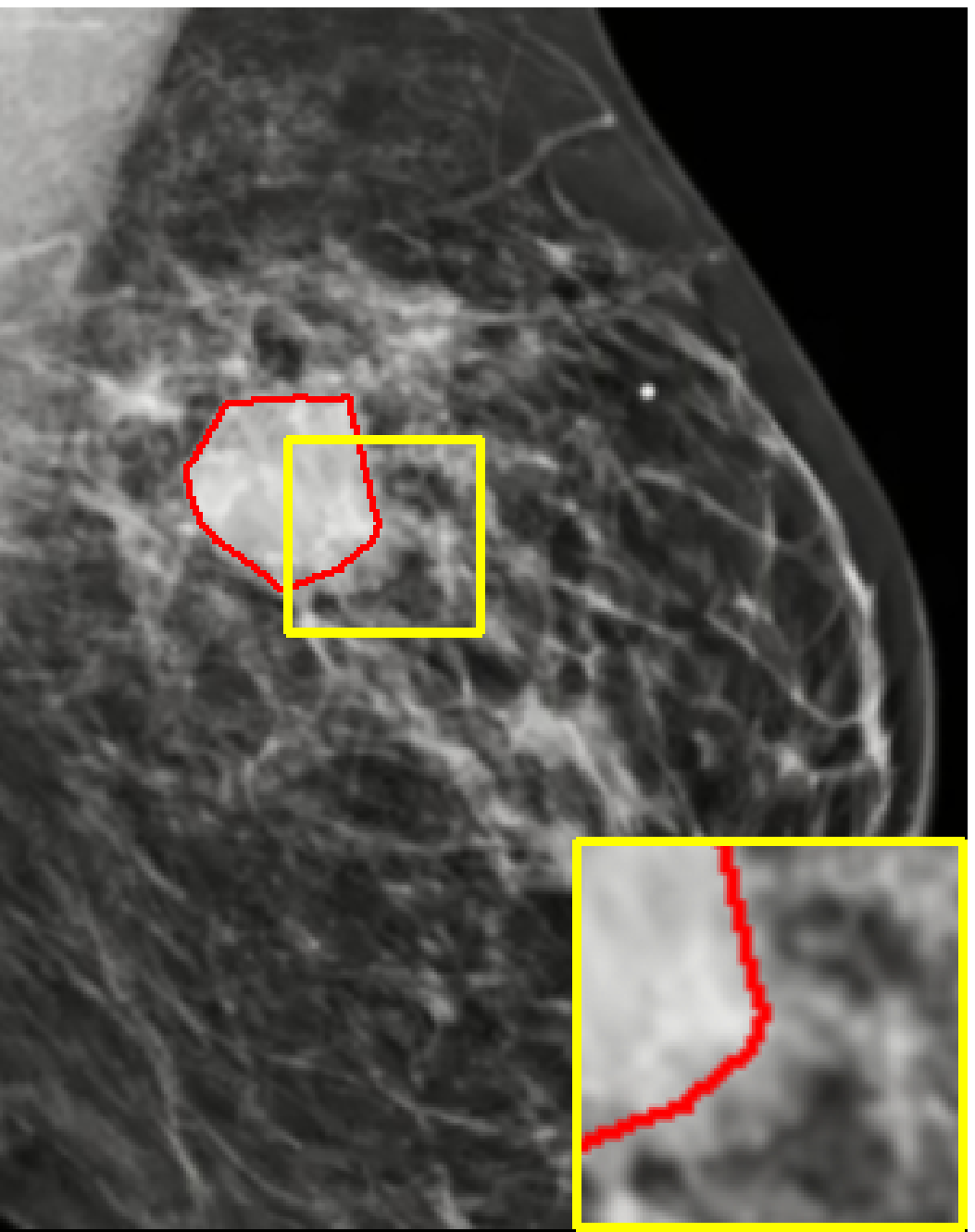}\\
	\includegraphics[height=1.6cm, width=2cm]{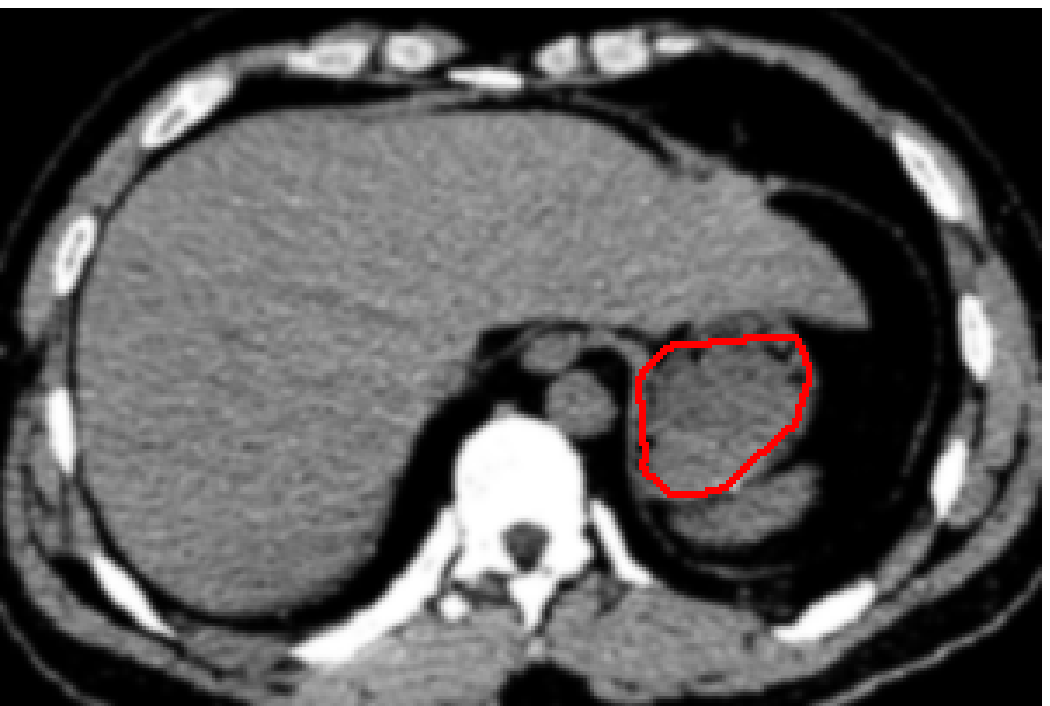}
	\includegraphics[height=1.6cm, width=2cm]{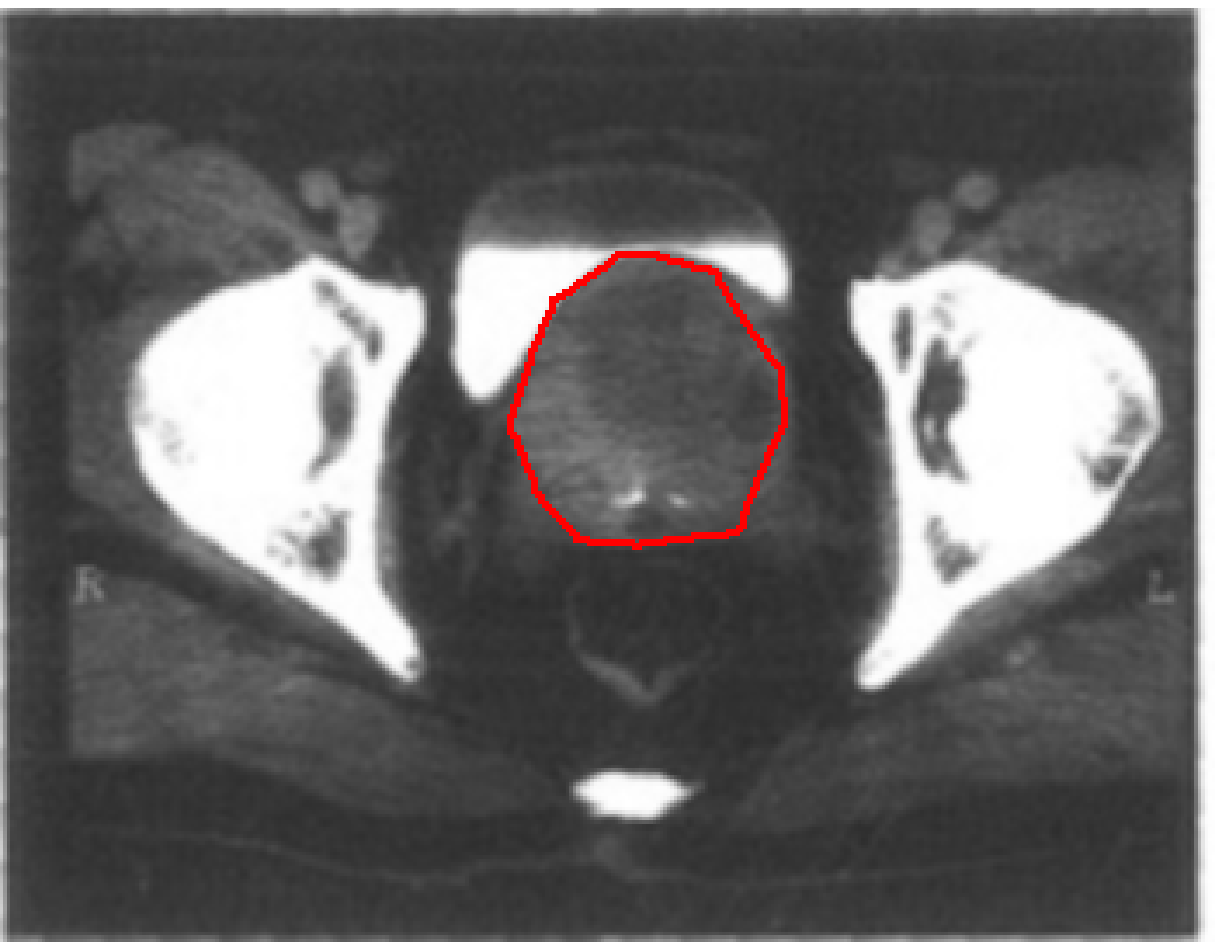}
	\includegraphics[height=1.6cm, width=2cm]{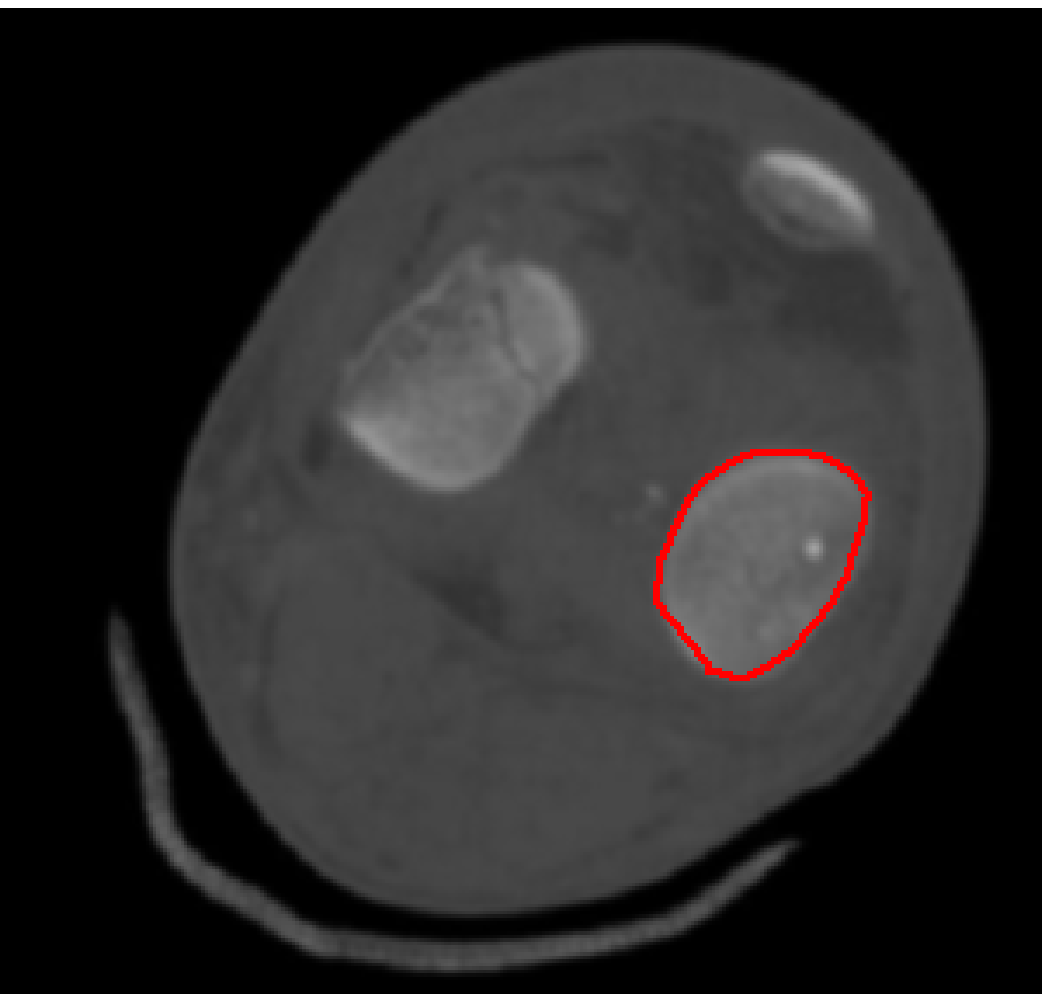}
	\includegraphics[height=1.6cm, width=2cm]{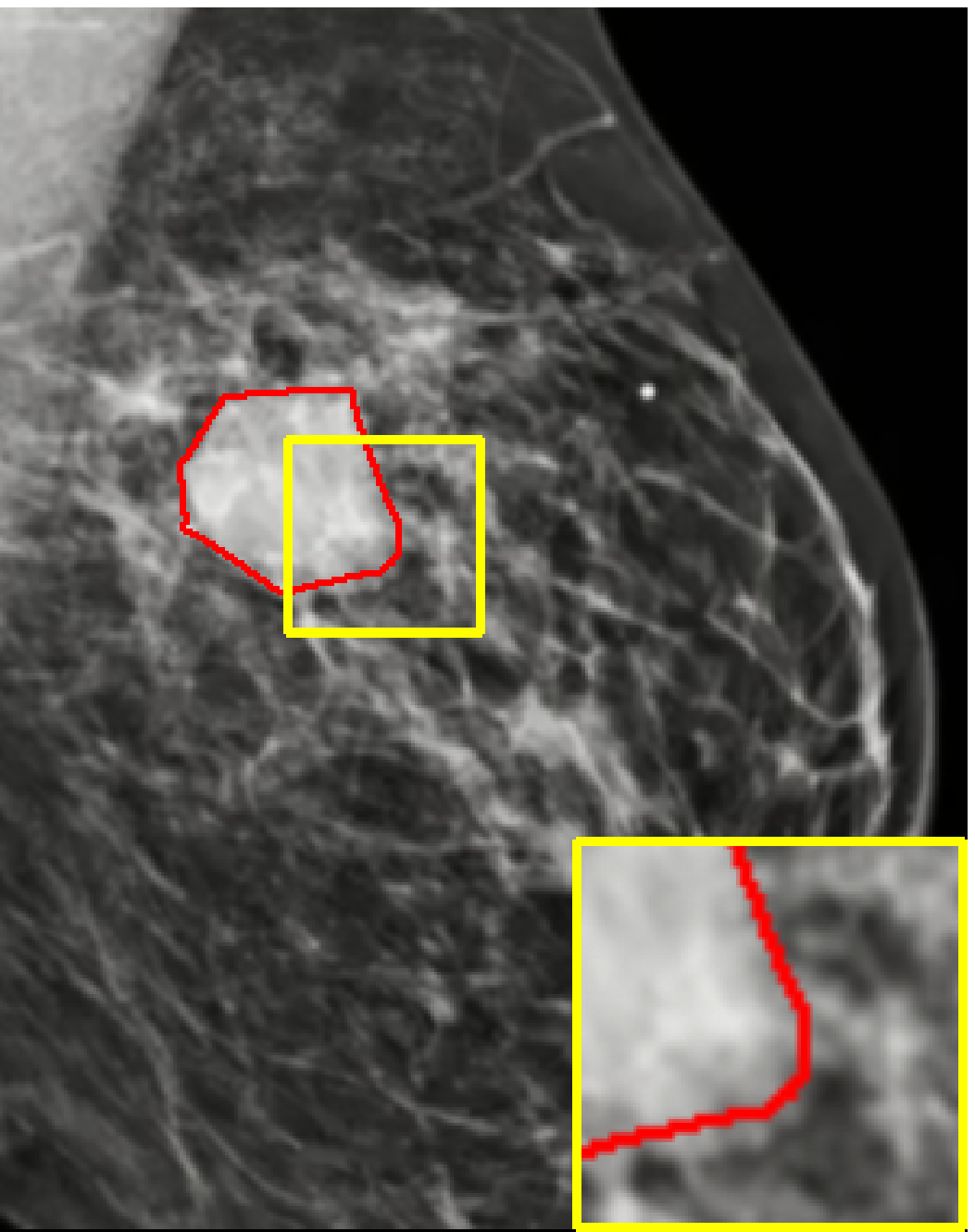}
	\\
	\includegraphics[height=1.6cm, width=2cm]{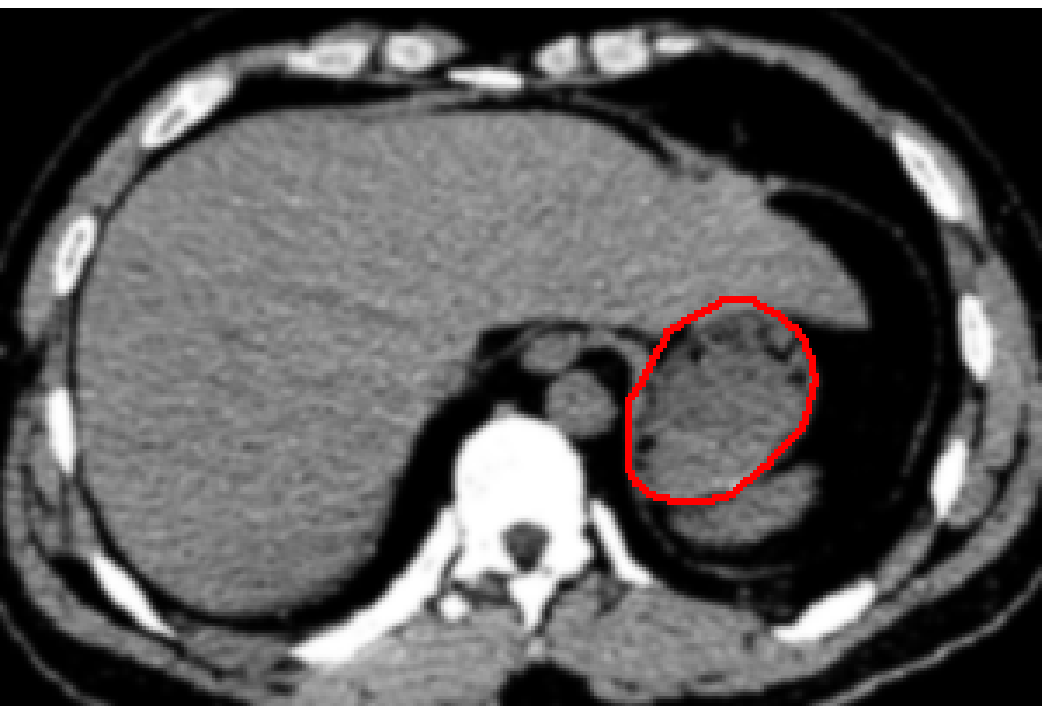}
	\includegraphics[height=1.6cm, width=2cm]{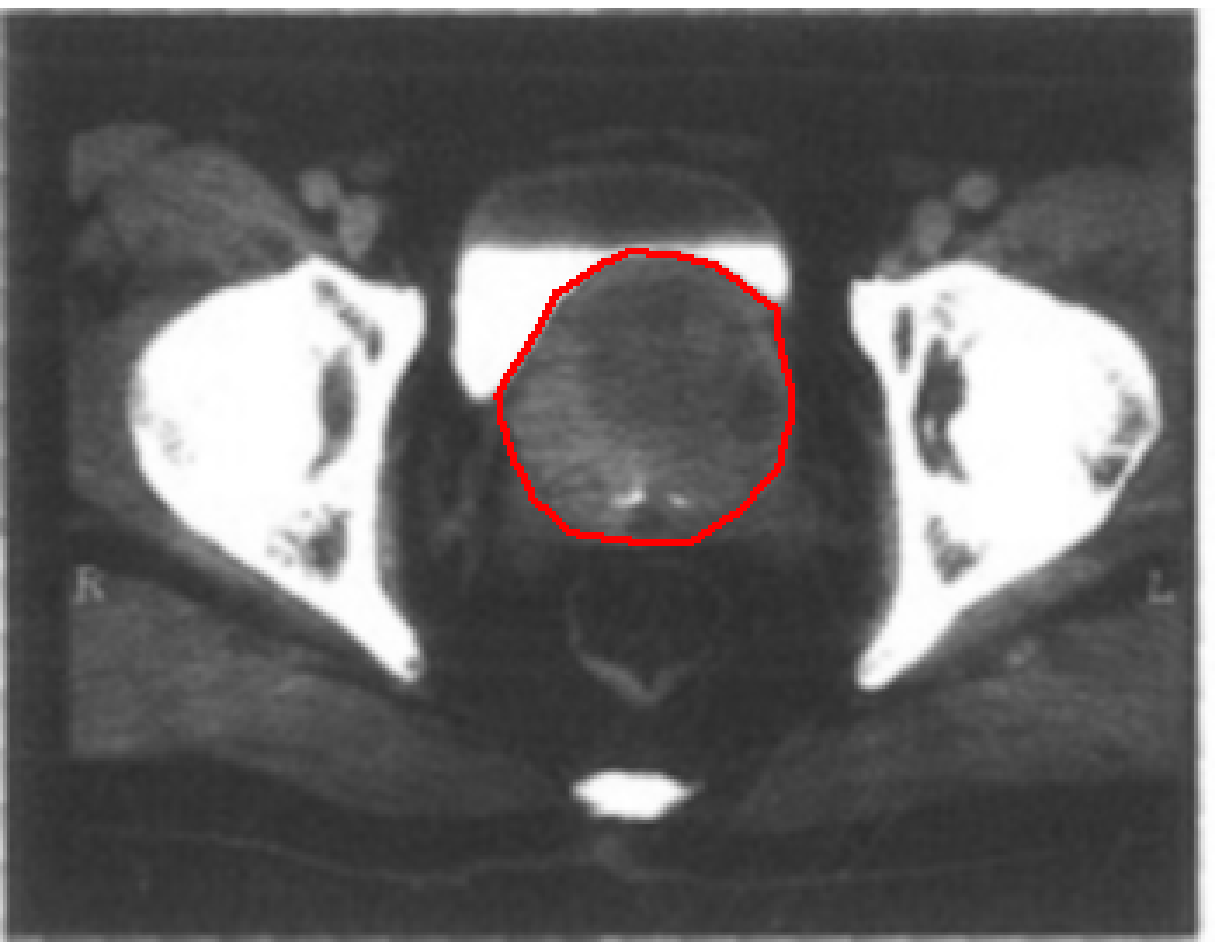}
	\includegraphics[height=1.6cm, width=2cm]{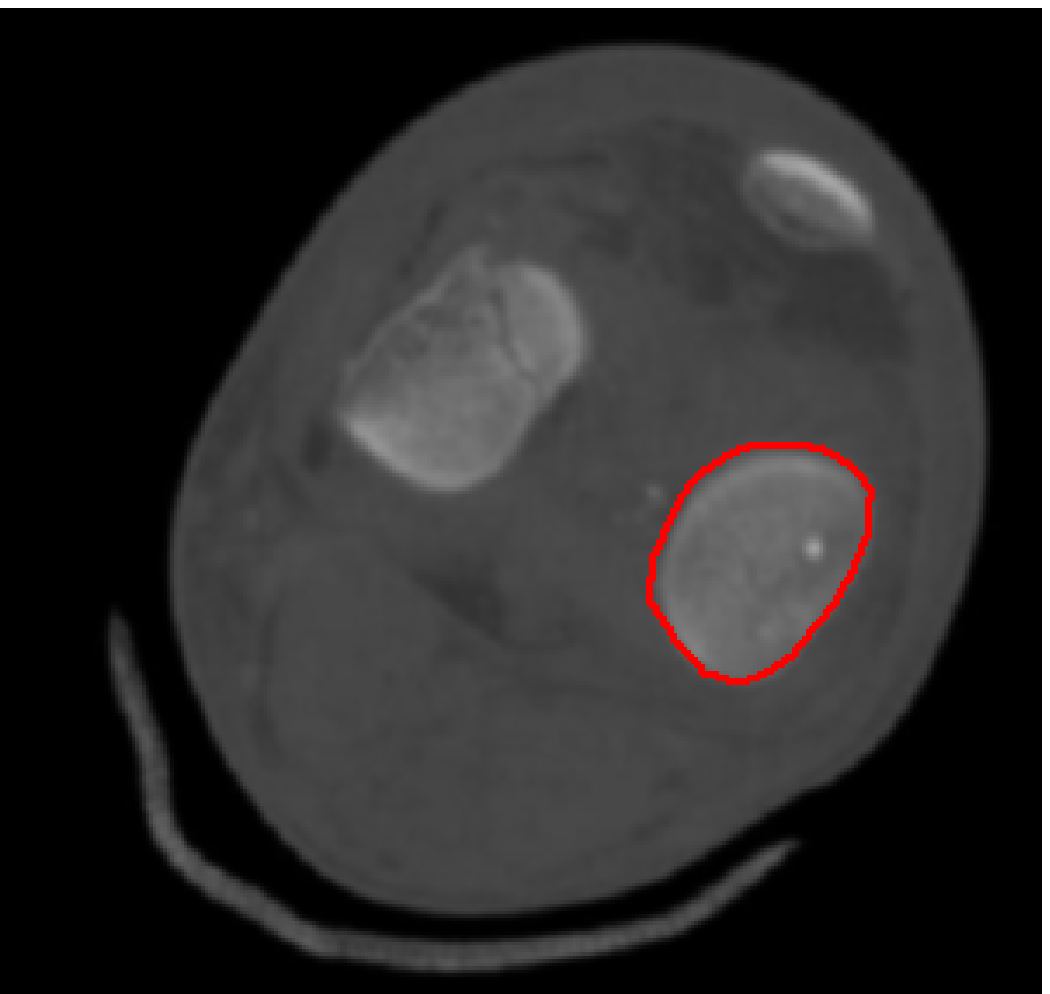}
	\includegraphics[height=1.6cm, width=2cm]{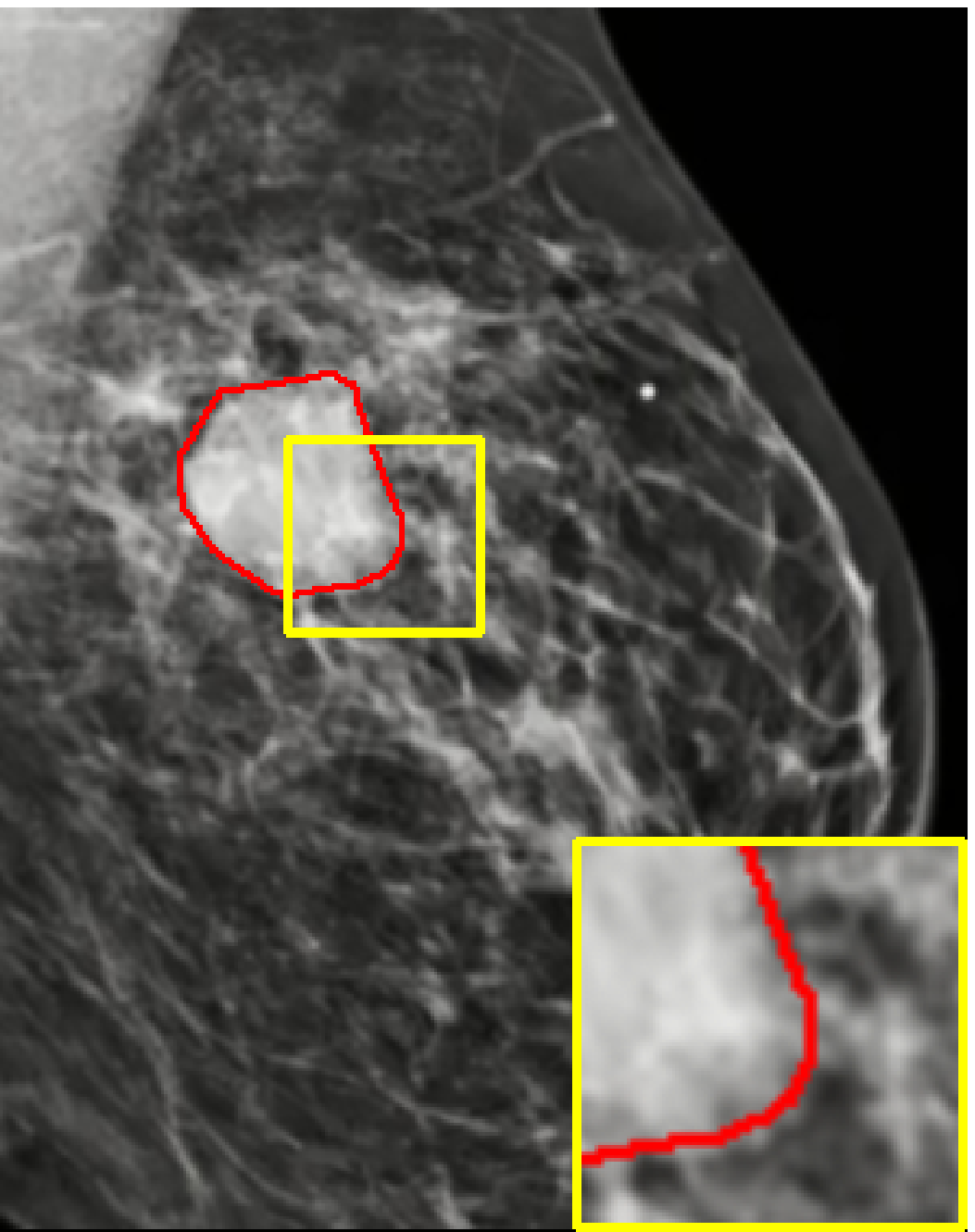}
	\caption{Labeled images (1st row) and results (2nd - 4th rows) by the interactive procedure:
		The labels for background are marked in purple, and
		the labels subscribed for foreground at the 1st, 2nd, and 3rd iterations are marked in red, green, and blue, respectively.  Parts of interest are zoomed.}
	\label{fig:Interactive}
\end{figure}

Comparing the results,
we can observe that the segmentation accuracy is improved gradually with some more labels are subscribed.
Let's take the left abdomen image as an instance. Boundary of the concerned object is very weak and the intensities on its background are nonuniform severely.
By using the initial labels (red), the proposed method cannot extract the object accurately.
With more and more labels  added,
the segmentation result is improved step by step until an accurate segmentation results are obtained.
In addition, the breast image listed in the fourth column is more challenging because of the cluttered textures and edges.
Therefore, the initial segmentation result is far from satisfactory.
When some additional labels are subscribed properly, our proposed method is able to segment it successfully.
In summary, these results demonstrate that the interactive procedure is capable of dealing with complicated image segmentation by adding proper labels gradually.

\subsection{Convex hull computation}\label{exp:ch}
In this subsection we present some numerical examples for convex hull
computation of given sets, which are used in \cite{li2021new}.
For clean data set (resp. noisy data set),
the proposed model (\ref{eq:mod_convexhull}) (resp. model (\ref{eq:convexhull_noise})) is applied to compute exact (approximate) convex hull.
We compare the results by our method and the LS method \cite{li2021new} quantitatively in the sense of shape-distance, where the results of clean data sets by the
quick-hull algorithm \cite{barber1996quickhull} are viewed as ground truth.
We chose the same parameters with the segmentation issue of convex objects except the radii of radial functions in the convexity constraint term, which are  $\text{mod}(l+1,22),\text{mod}(l+3,22),\text{mod}(l+5,22),\text{mod}(l+7,22),\text{mod}(l+9,22)$ and  $l:=2\text{mod}(\lfloor{k/100}\rfloor,10)$.
\begin{figure*}[!t]
	\centering
	\includegraphics[height=1.6cm, width=1.6cm]{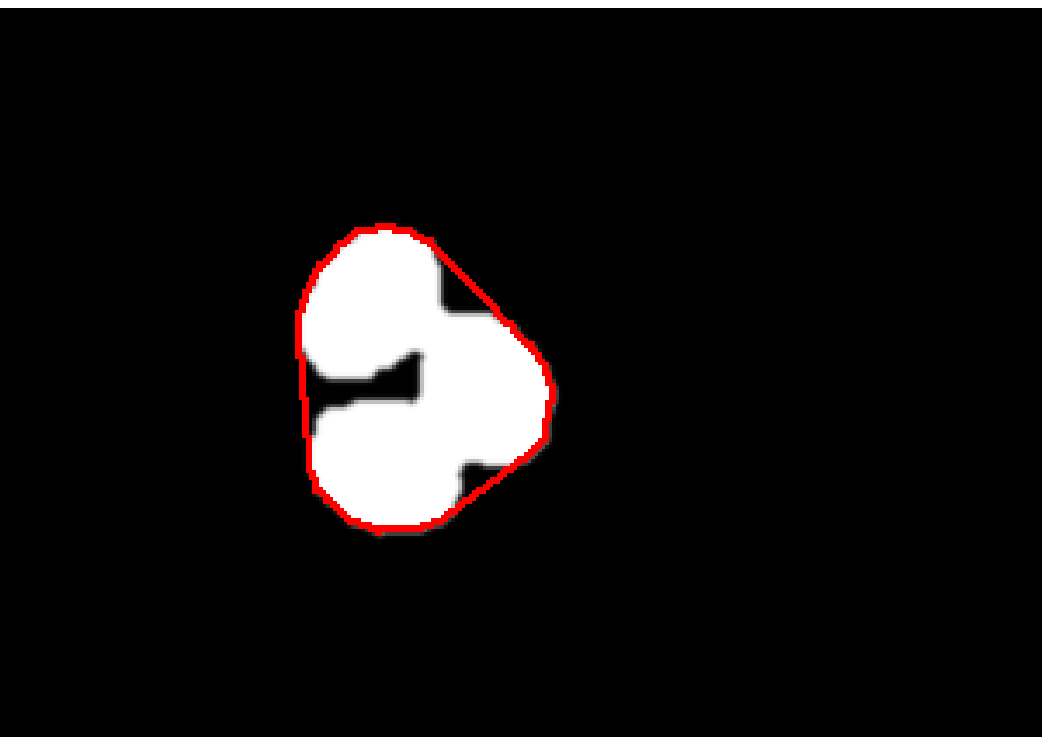}
	\includegraphics[height=1.6cm, width=1.6cm]{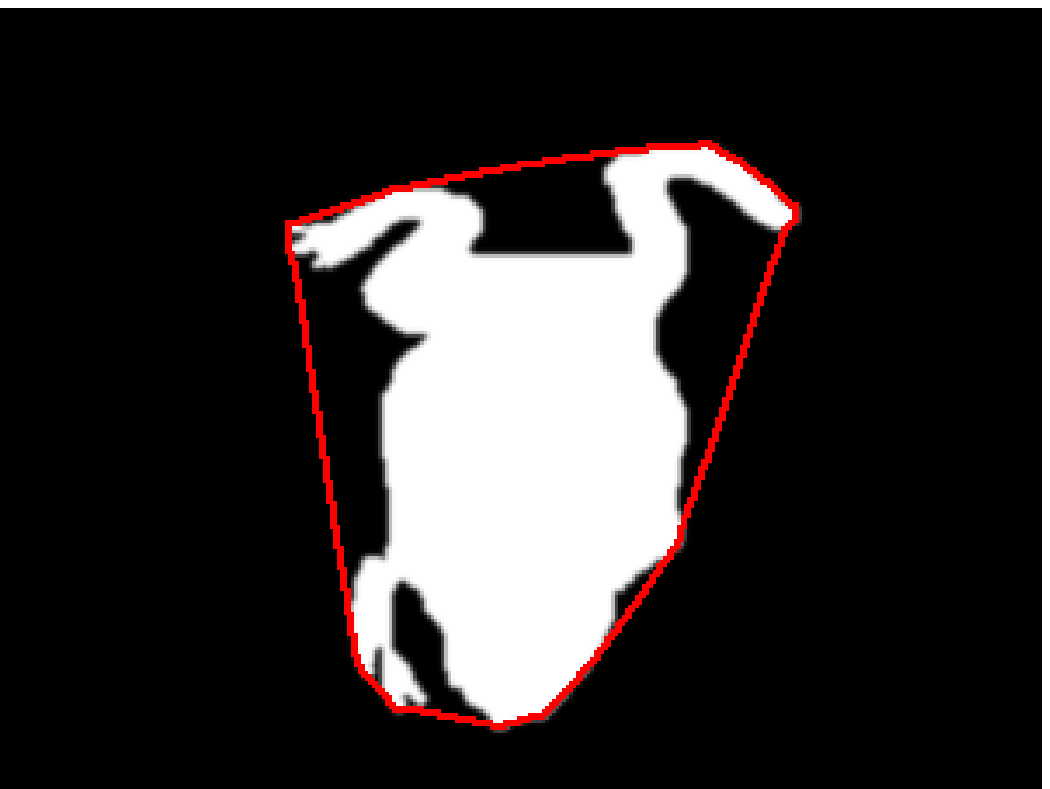}
	\includegraphics[height=1.6cm, width=1.6cm]{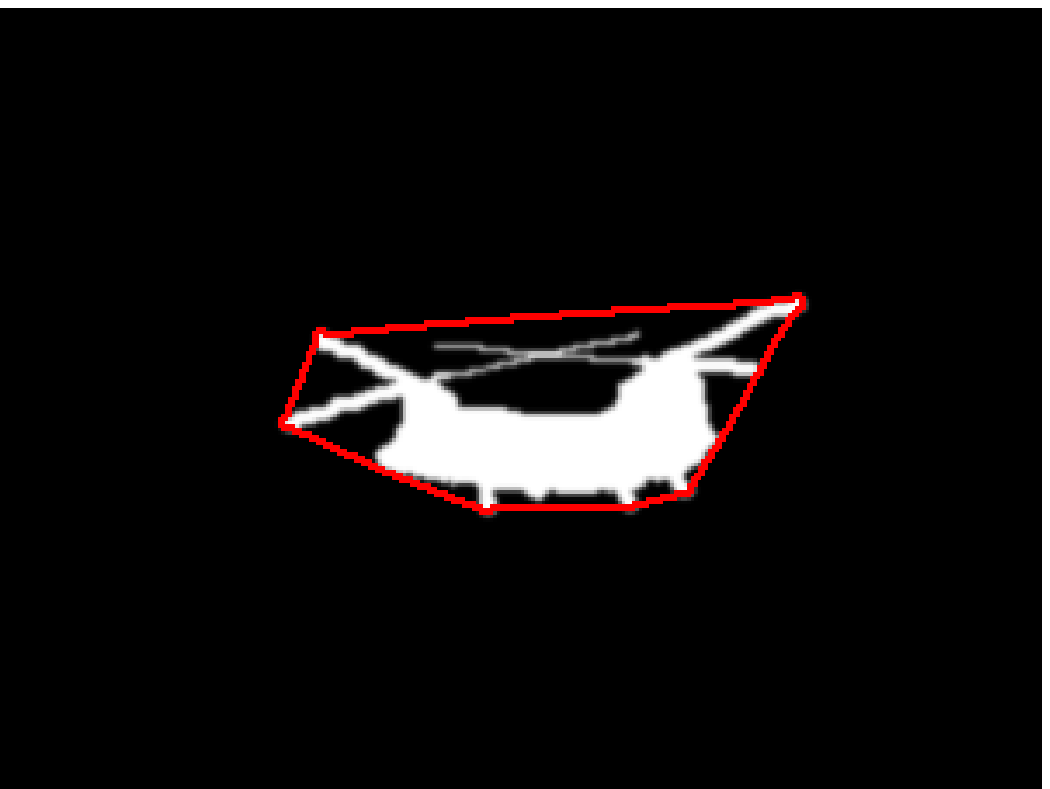}
	\includegraphics[height=1.6cm, width=1.6cm]{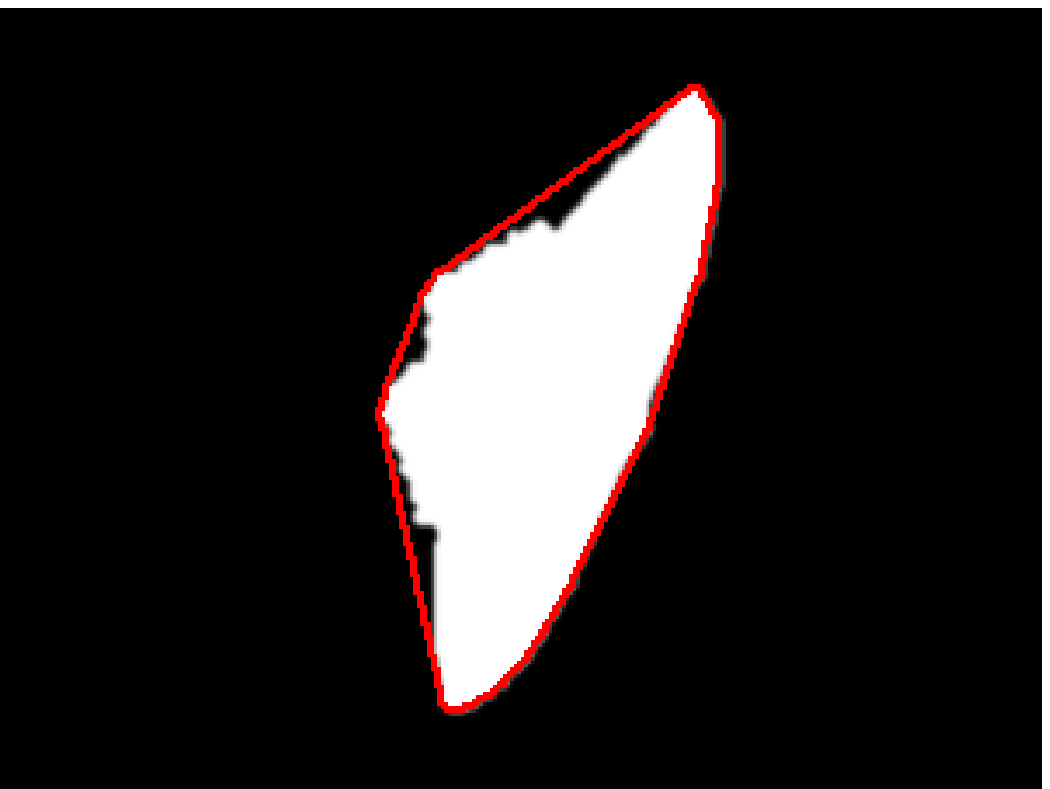}
	\includegraphics[height=1.6cm, width=1.6cm]{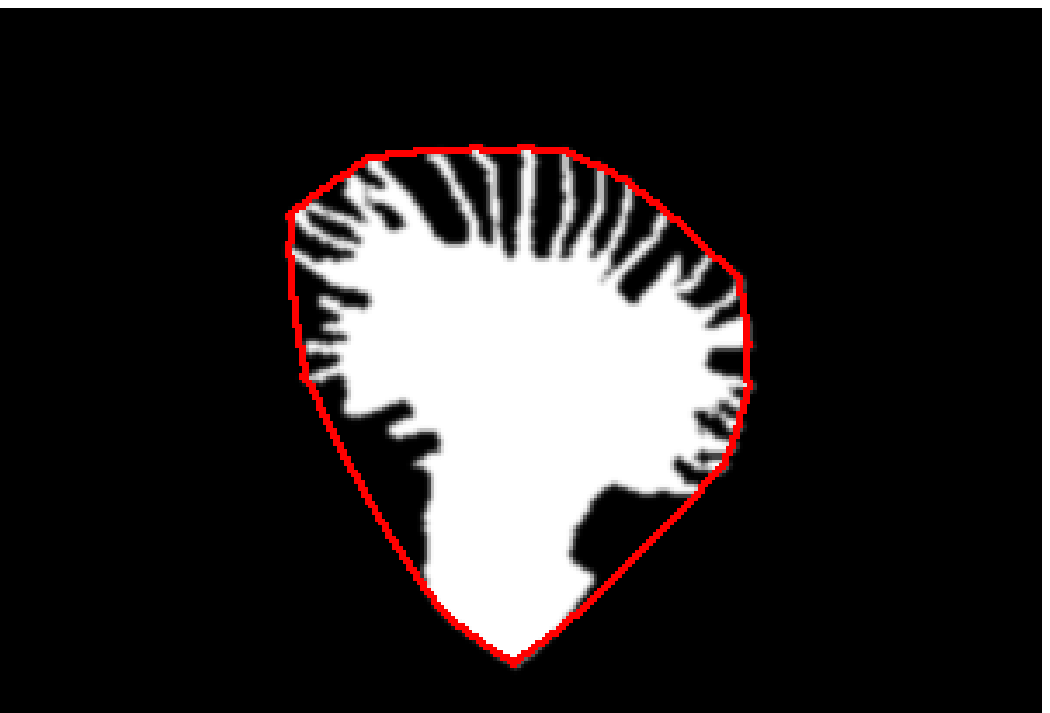}
	\includegraphics[height=1.6cm, width=1.6cm]{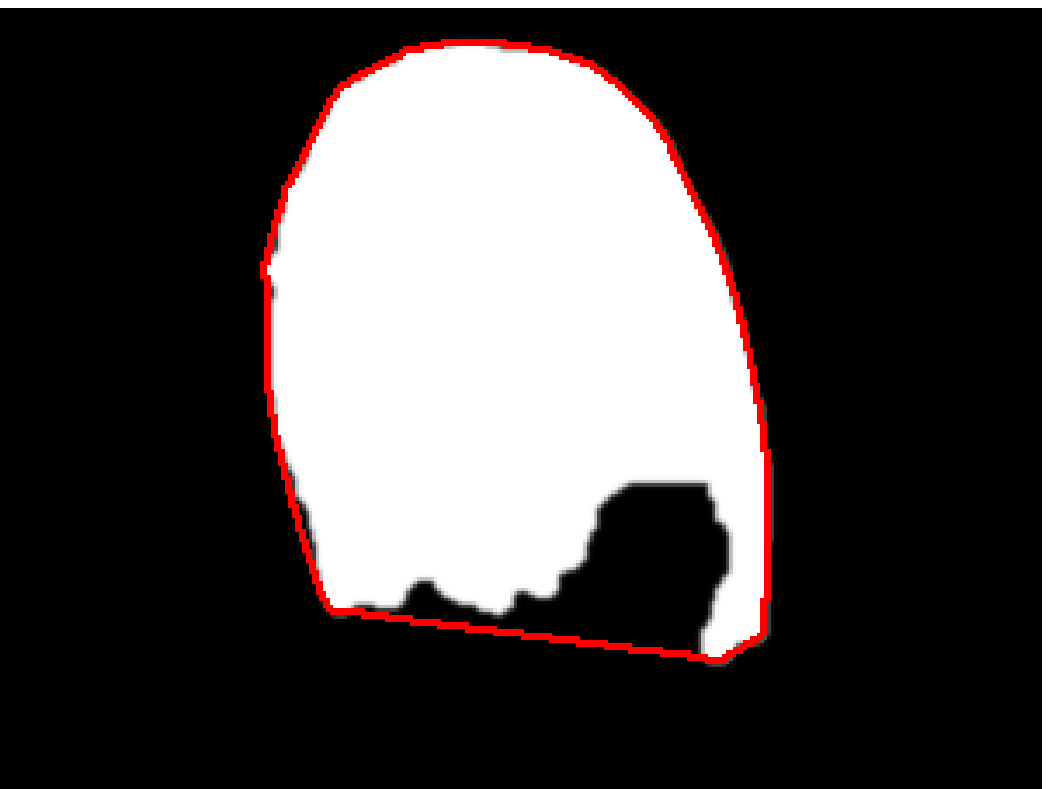}
	\includegraphics[height=1.6cm, width=1.6cm]{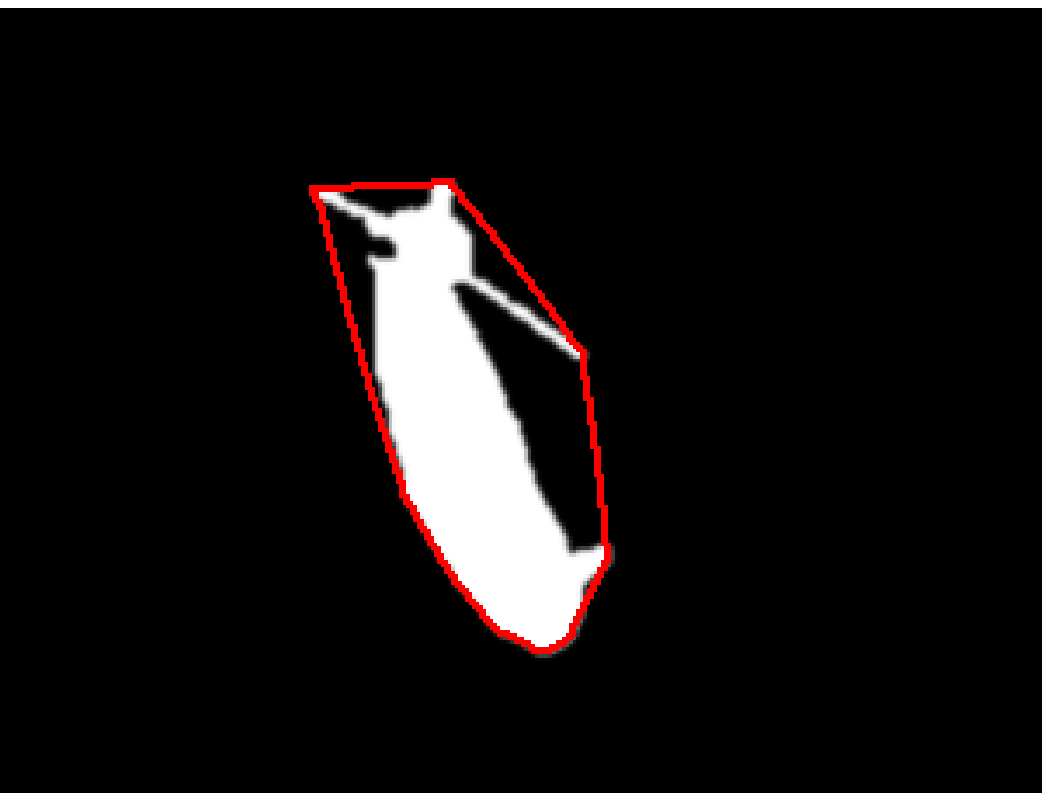}
	\includegraphics[height=1.6cm, width=1.6cm]{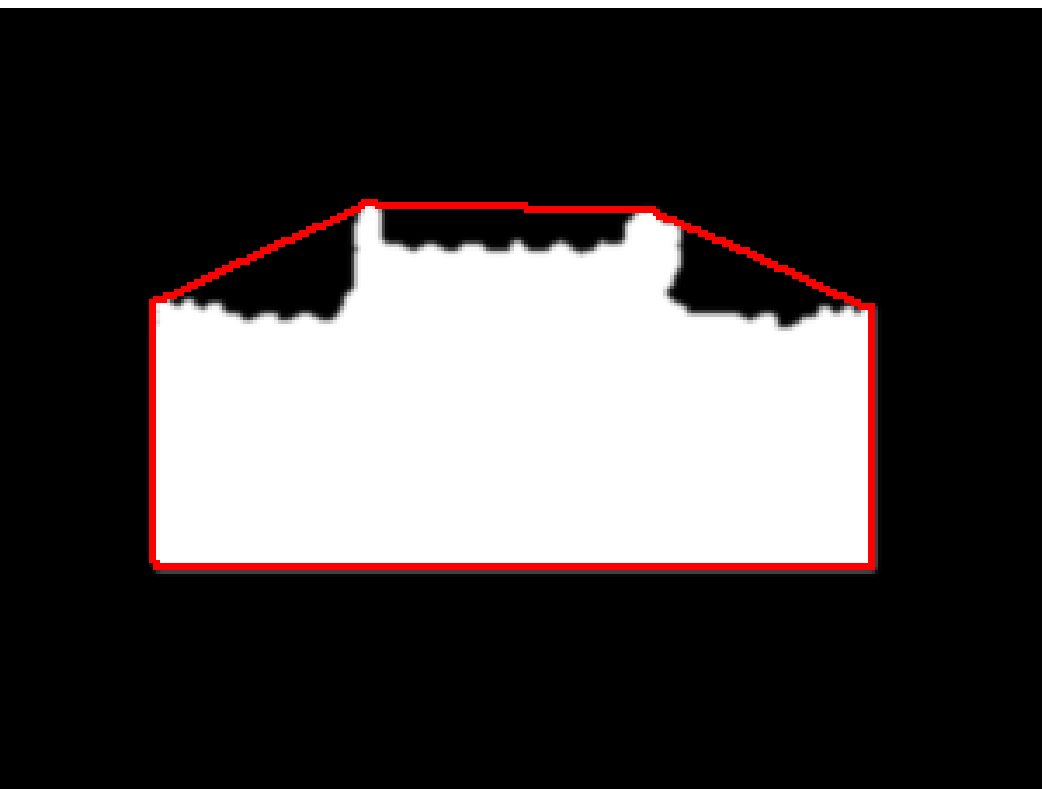}
	\includegraphics[height=1.6cm, width=1.6cm]{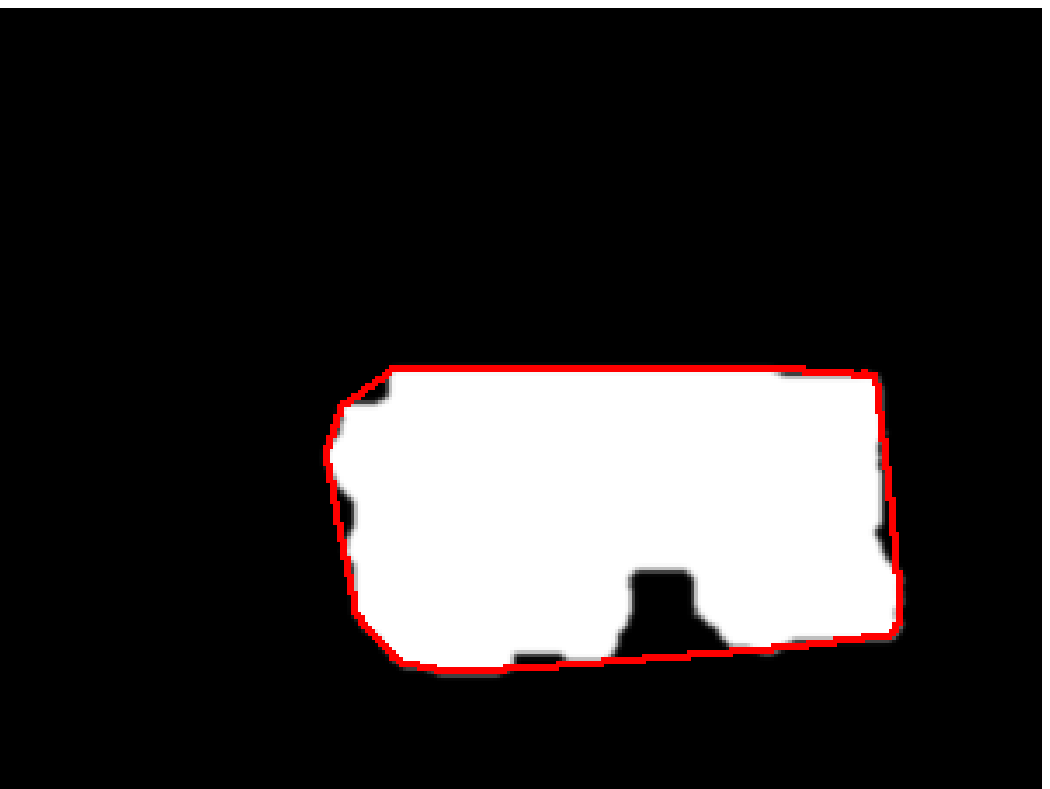}\\
	\caption{The convex hulls of clean data set and noisy data sets. They are named img1 to img9 from left to right.}
	\label{fig:convexhull_our}
\end{figure*}

\begin{table*}[!t]
	\caption{The shape-distances of the results of clean data set by the LS method and ours.}
	\label{tab:error_compare}
	\centering
	\begin{tabular}{c|ccccccccc}
		\hline
		name & img1 & img2 & img3  &img4 & img5 &img6 &img7 &img8 & img9\\
		\hline
		LS(\%) & 1.28 & 1.51 & 1.13 & 0.92 & 0.73 &0.61 &1.08 &0.63 &0.79\\
		
		our(\%) & 1.29& 0.96 & 1.09 & 0.90 & 0.75 &0.61 & 1.07 &0.63 &0.77\\
		\hline
	\end{tabular}
\end{table*}

In Fig. \ref{fig:convexhull_our}, some results
of the given data without noises are demonstrated.
The relative shape-distances are shown in Tab. \ref{tab:error_compare}, the average shape-distances of LS and the proposed methods are $0.964\%$ and $0.897\%$, from which we can see our proposed algorithm can yield
more accurate results than the LS method (For the
result of the LS method, one can refer to \cite{li2021new} in details).
\begin{figure}[htbp]
	\centering
	\includegraphics[height=1.6cm, width=1.6cm]{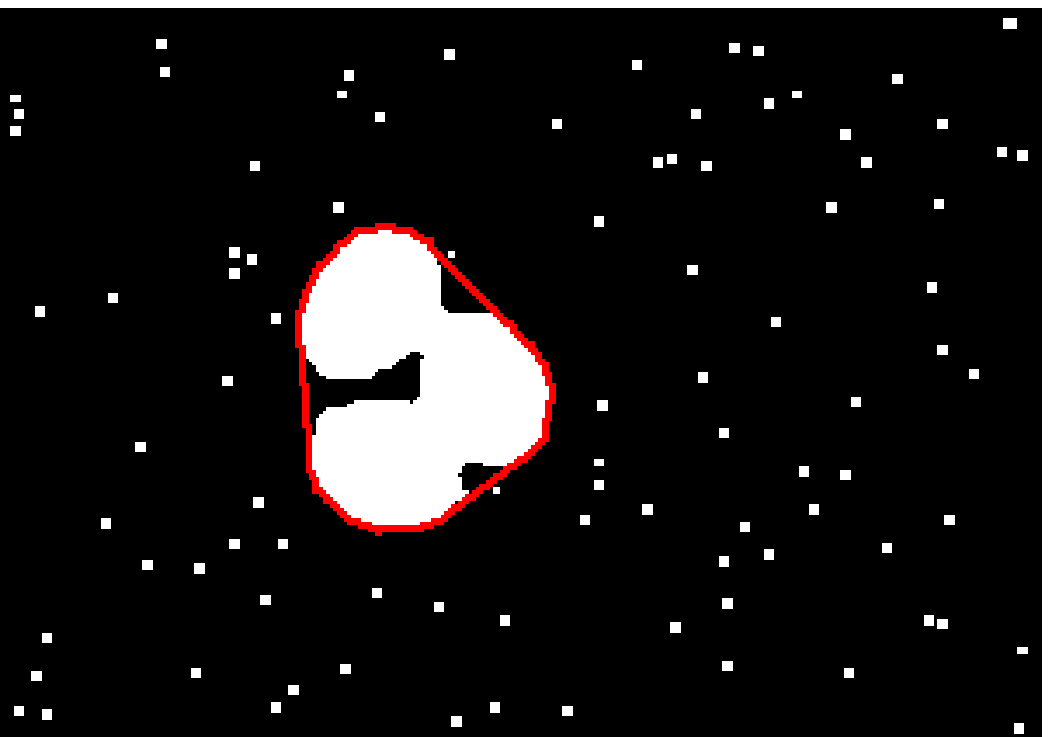}
	\includegraphics[height=1.6cm, width=1.6cm]{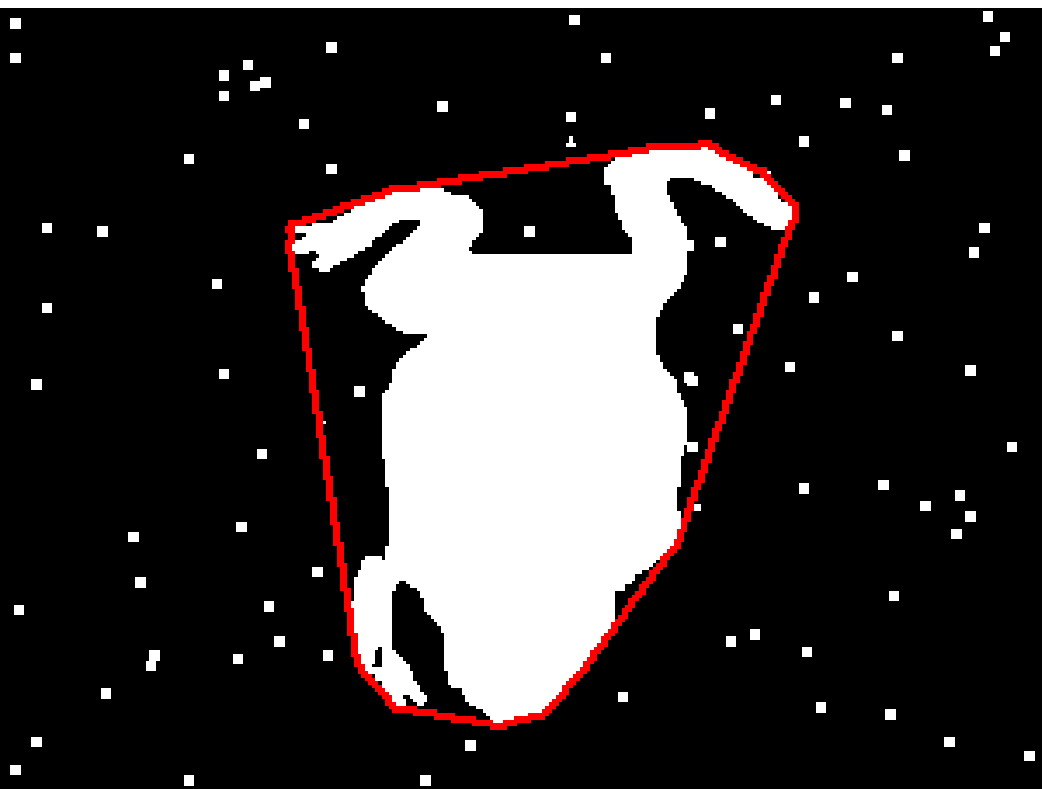}
	\includegraphics[height=1.6cm, width=1.6cm]{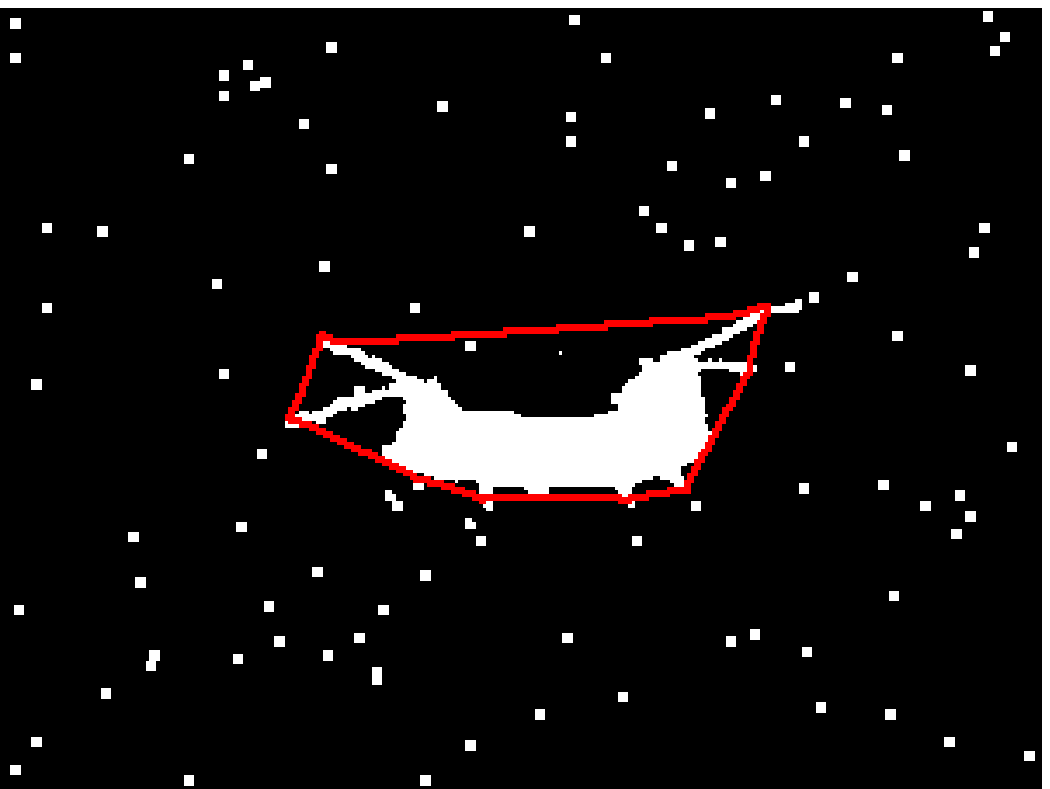}
	\includegraphics[height=1.6cm, width=1.6cm]{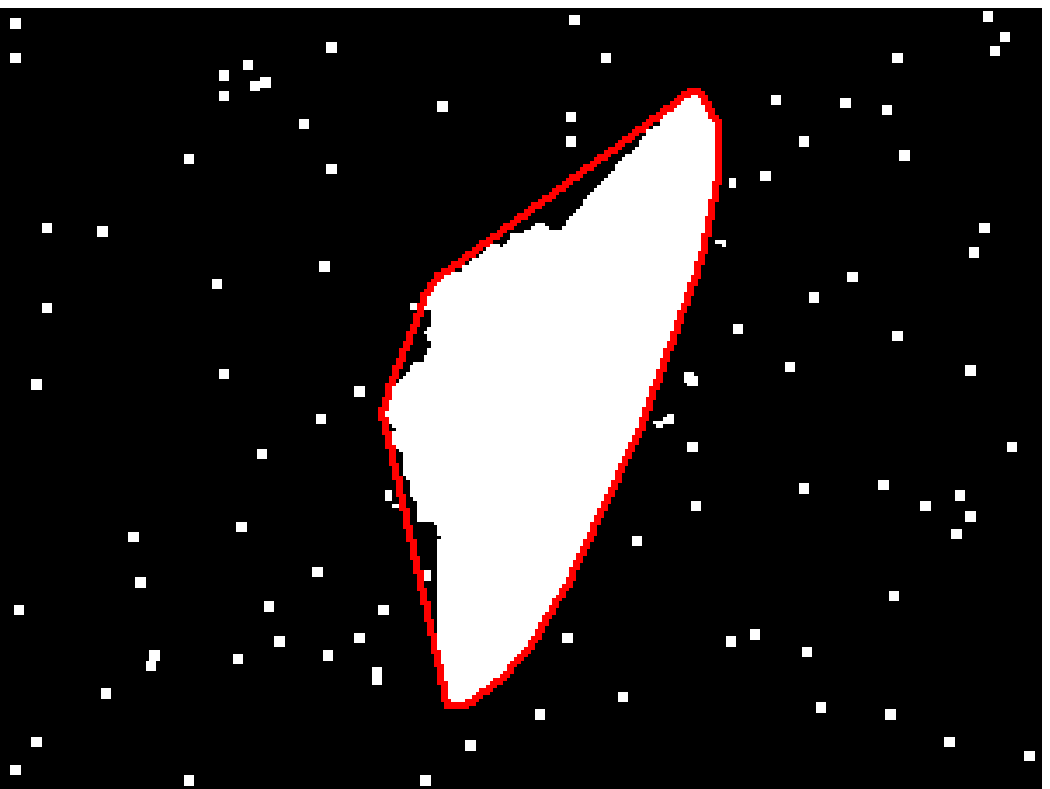}
	\includegraphics[height=1.6cm, width=1.6cm]{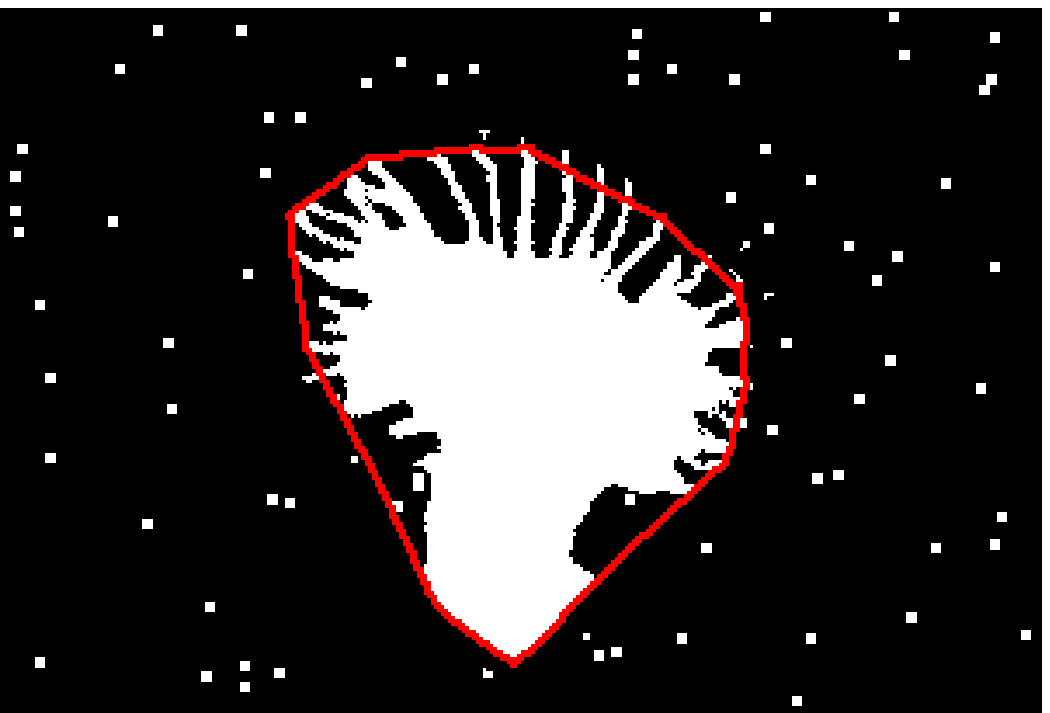}
	\includegraphics[height=1.6cm, width=1.6cm]{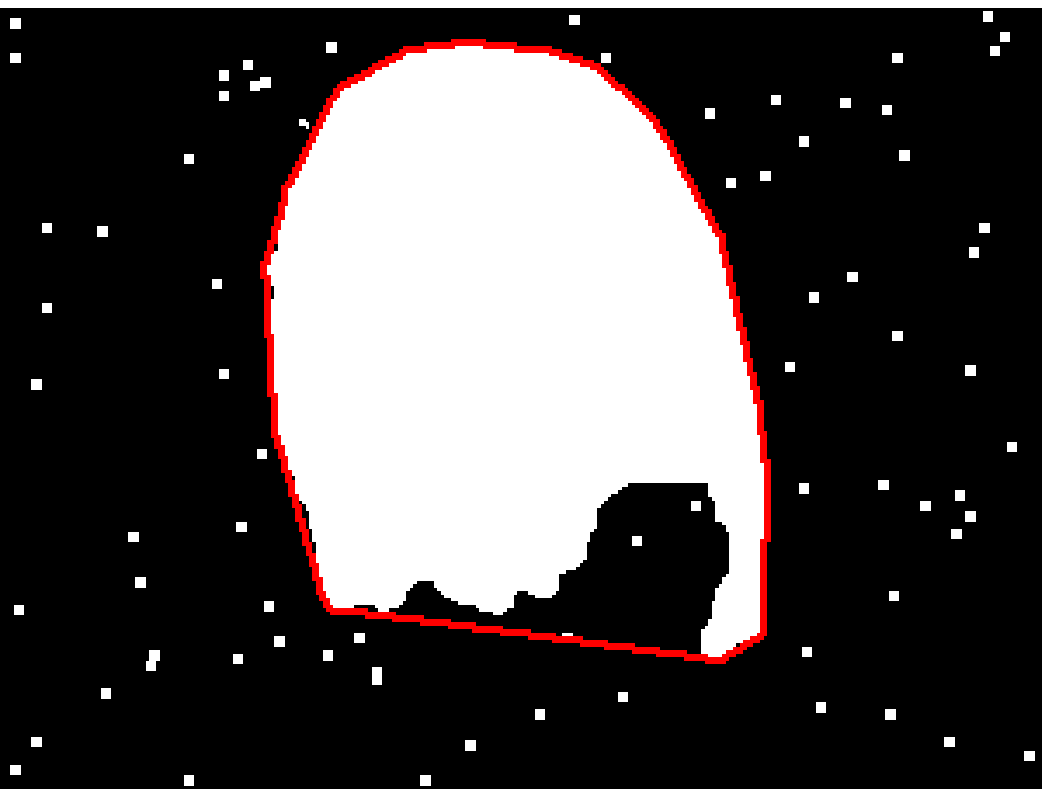}
	\includegraphics[height=1.6cm, width=1.6cm]{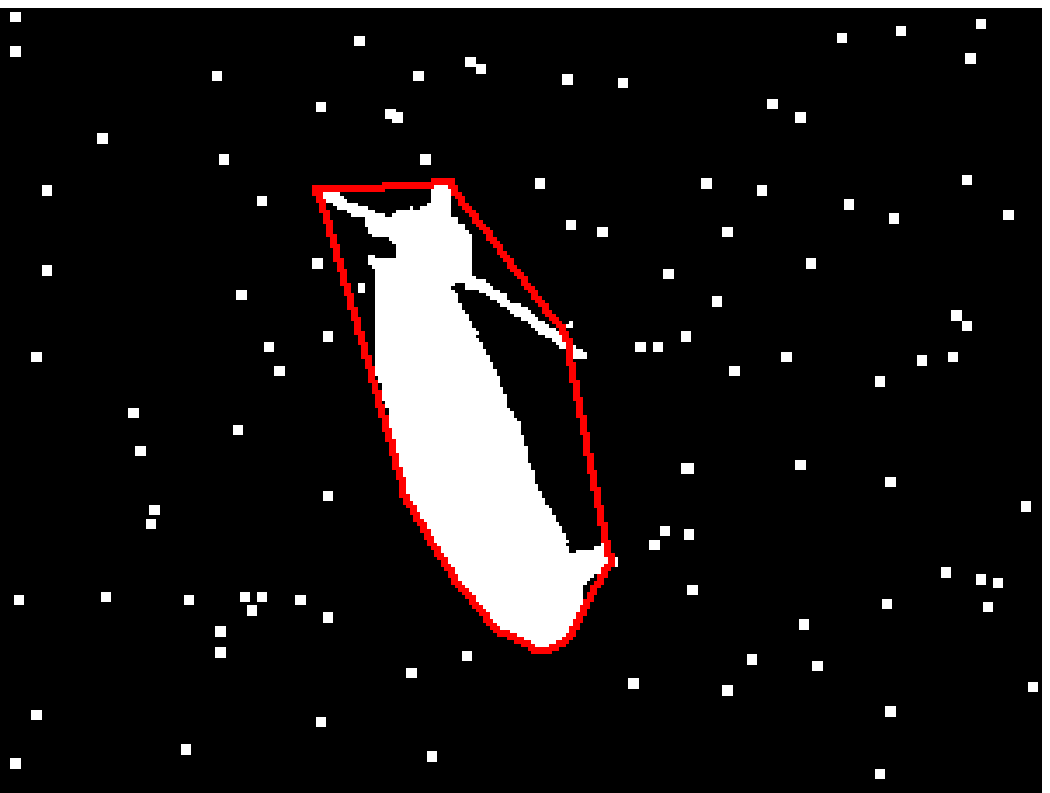}
	\includegraphics[height=1.6cm, width=1.6cm]{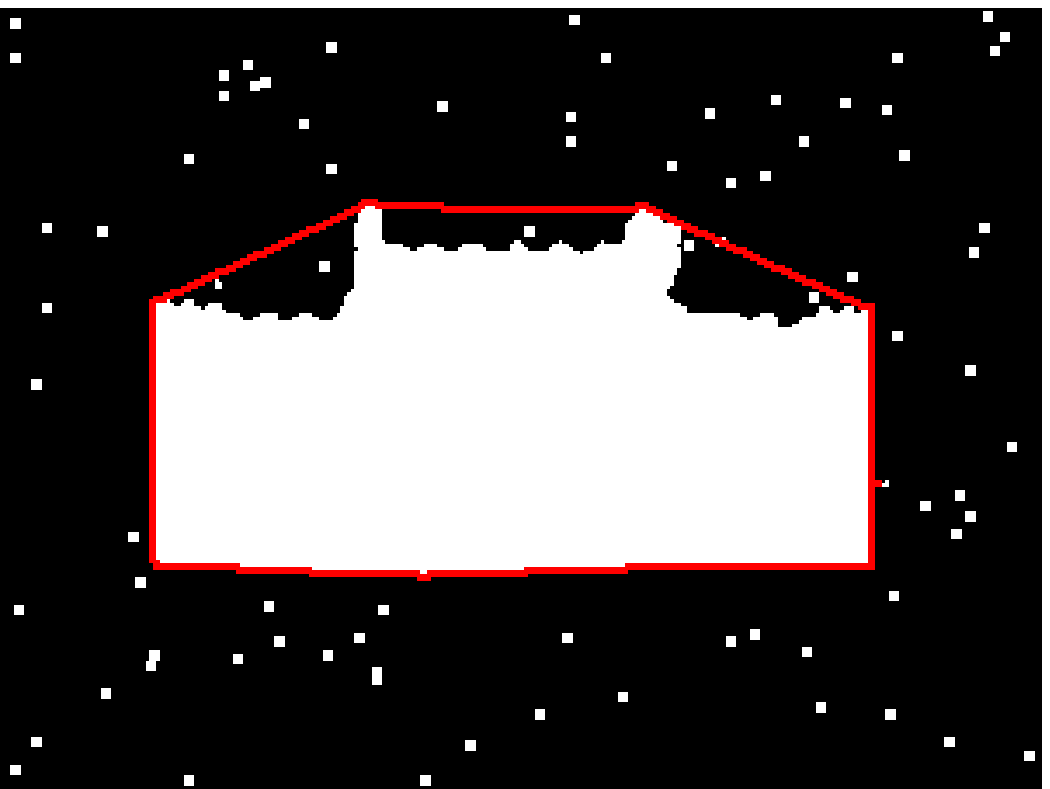}
	\includegraphics[height=1.6cm, width=1.6cm]{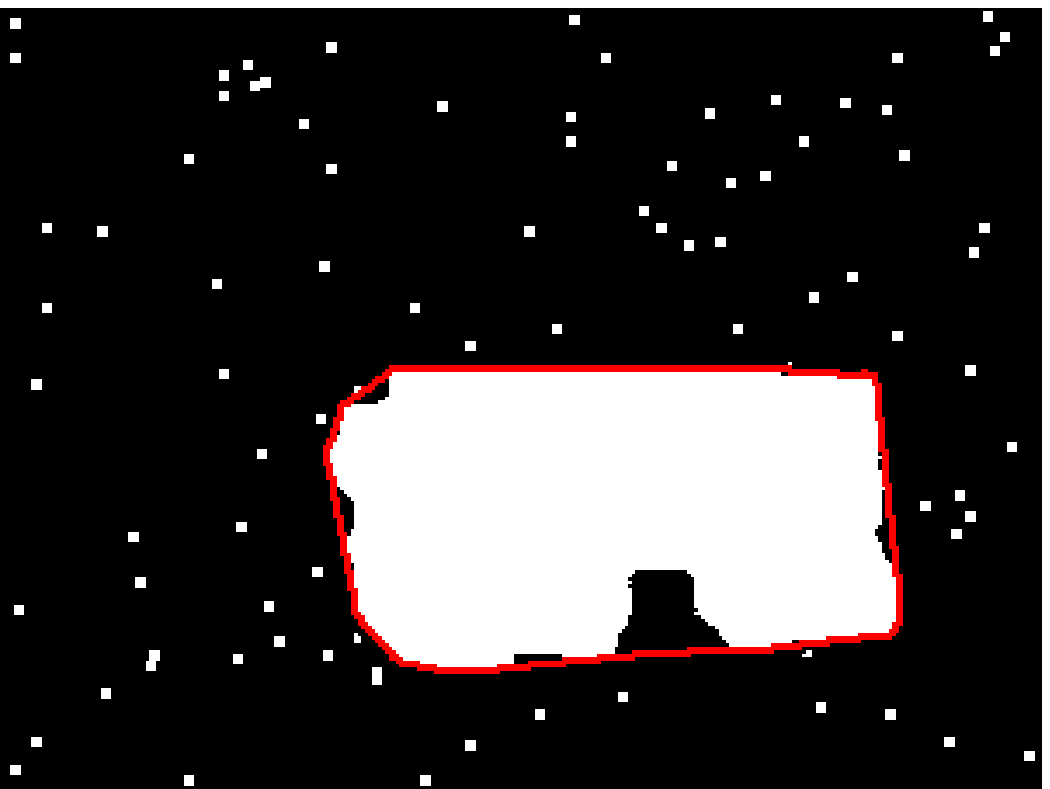}
	\caption{The convex hulls of noisy data sets. They are named img1 to img9 from left to right and top to bottom. }
	\label{fig:convexhull_new}
\end{figure}

The approximated convex hulls of the noisy data set are demonstrated in Fig. \ref{fig:convexhull_new}
by the model (\ref{eq:convexhull_noise}), where the regularization parameter $\lambda=1$. The quantitative comparisons of shape-distances are tabulated in Tab. \ref{tab:error_new}, the average shape-distances of the LS and the proposed methods are $4.524\%$ and $2.037\%$, which show the superiority of our method to the LS method.

\begin{table}[htbp]
	\caption{The shape-distances of the results of noisy  data set by the LS method and ours.}
	\label{tab:error_new}
	\centering
	\begin{tabular}{c|ccccccccc}
		\hline name & img1 & img2 & img3  &img4 & img5 &img6 &img7 &img8 & img9 \\
		\hline
		LS(\%) & 1.28 &4.79 &9.63 &3.33 &4.39 &2.73 &6.85 &3.79 &3.93\\
		our(\%) & 1.29& 1.51 & 3.28 &1.80 &3.37 &1.22 &3.21 &1.88 &0.77\\
		\hline
	\end{tabular}
\end{table}

Lastly, we want to show an interesting performance of the proposed method for the issue involved multiple convex objects, which
is illustrated by numerical experiments on the convex hull computation of multiple disconnected sets.

Intuitively, we need to design a model with multiple indicator functions for multiple convex hulls computation similar to the issue of multiple objects segmentation. Factually,
if the distance(s) between the given sets is large enough with respective the radii of the radial functions for convexity constraint, the model (\ref{eq:convexhull_noisec}) can yield multiple convex hulls numerically  using only one indicator function, while the model will yield a single convex hull otherwise (see Fig. \ref{fig:convexhull_multi}).
For the given two data sets from two real images in
Fig. \ref{fig:convexhull_multi}, we use
radii $\{1,3,5,7,9,11,13,15\}$ for left and $\{1,3,5,7,9,11,13,15,20,25,30\}$ for right.

There are advantages and disadvantages for this property.
On one hand, we should increase the radii of the radial functions to yield a single convex hull if the given set is disconnected,
which will
result in higher computational cost. On the other hand, we can utilize this property and use only one rather than multiple indicator functions to compute multiple convex hulls if the distance(s) between them is large enough, which will reduce computation cost.
This property can also be applied to multiple objects segmentation.

\begin{figure}[htbp]
	\centering
	\includegraphics[width=2.5cm]{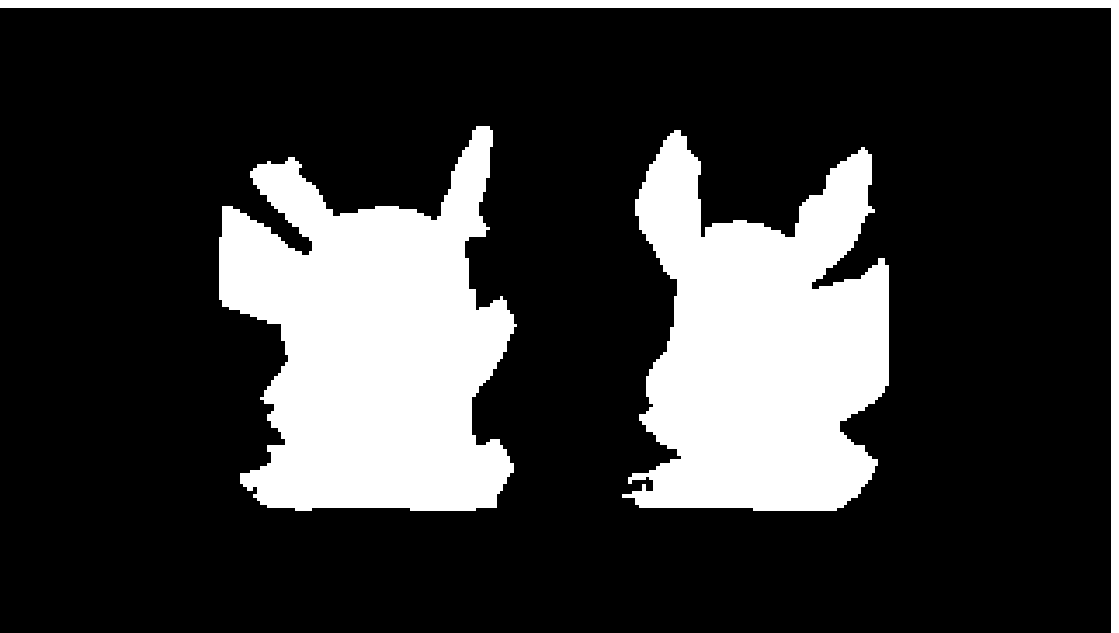}
	\includegraphics[width=2.5cm]{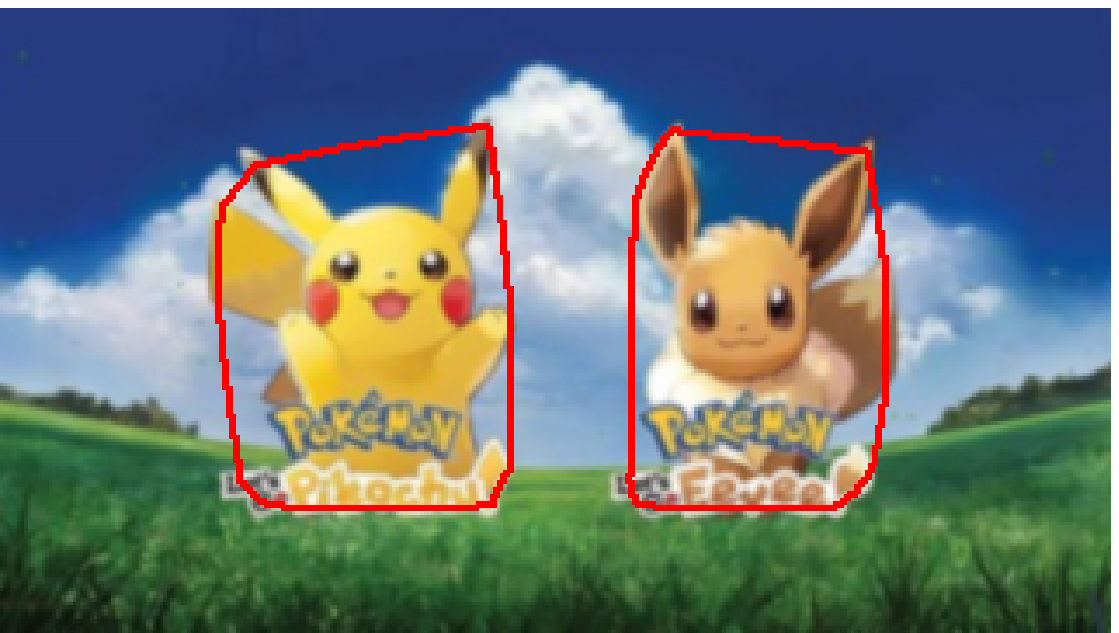}
	\includegraphics[width=2.5cm]{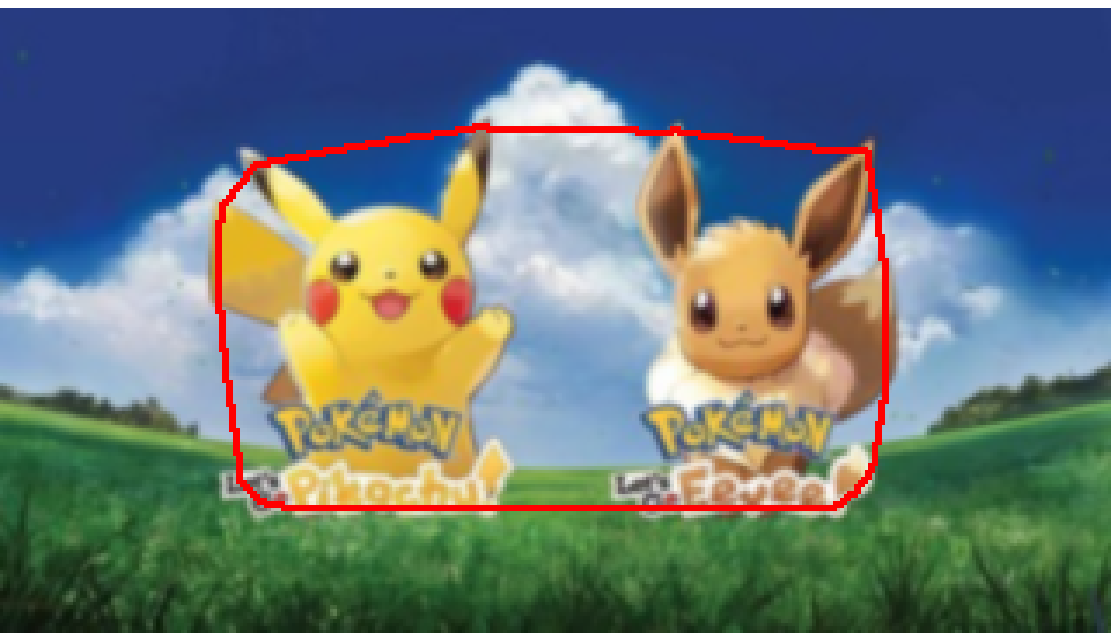}\\
	\includegraphics[width=2.5cm]{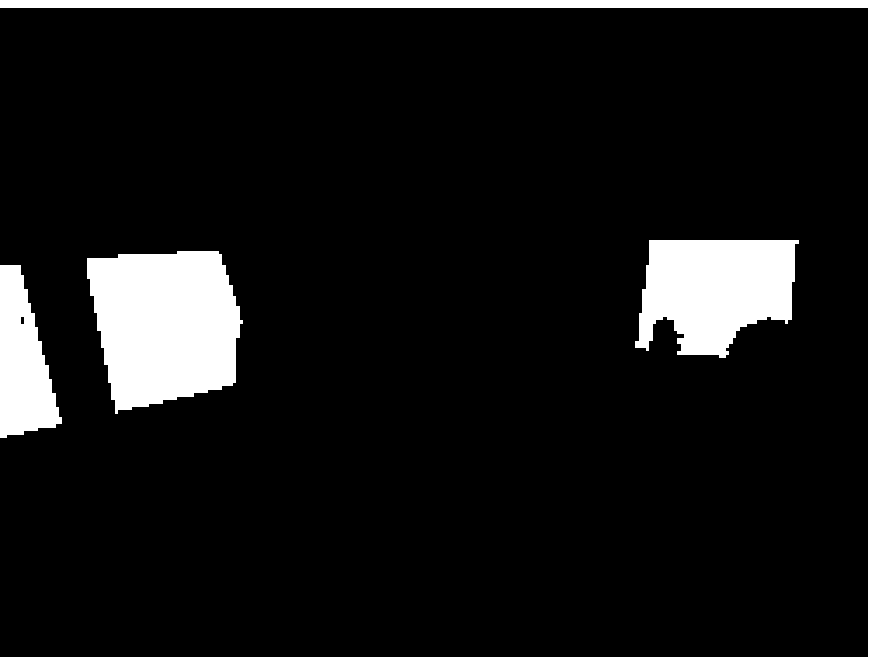}
	\includegraphics[width=2.5cm]{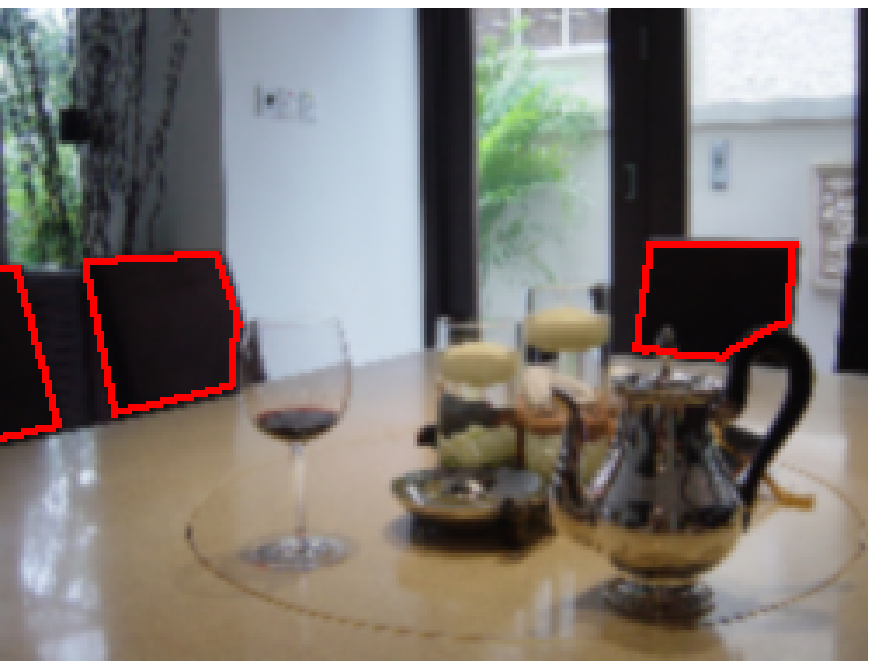}
	\includegraphics[width=2.5cm]{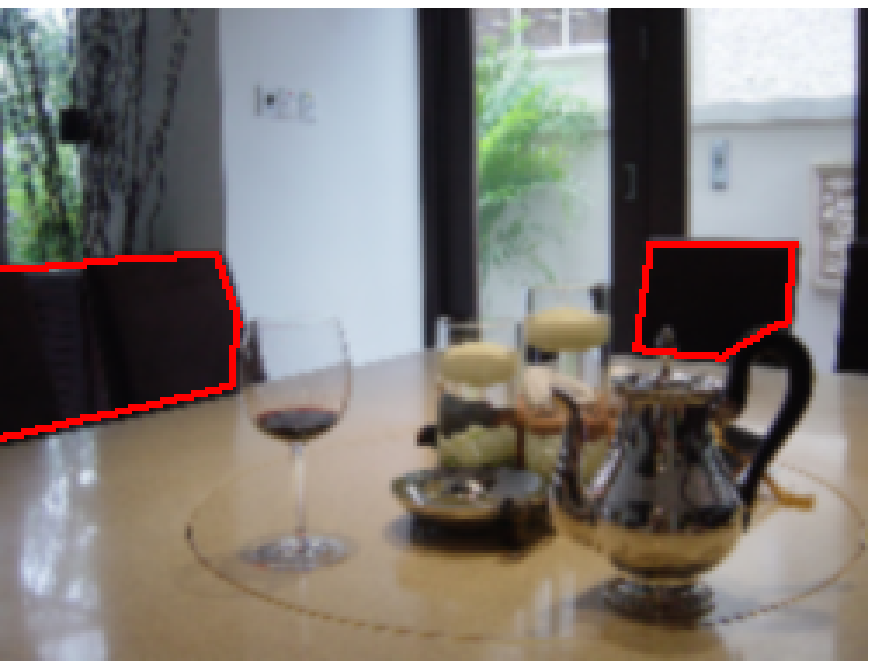}
	\caption{Multiple convex hulls computation. From left to right,
		the given sets and convex hulls using different radii for convexity constraint are demonstrated.}
	\label{fig:convexhull_multi}
\end{figure}

\section{Summary and Outlook}\label{sec5}
Shape priors are playing very important roles  in image segmentation, object recognition, computer vision. Computable characterizations for the shape priors are very vital for many applications. In this paper, we have shown that
the shape convexity in 2D space can be characterized by a quadratic	constraint on its associated indicator function, and experiments on segmentation and convex hull computation were conducted.

Over the past decade, great progresses on the research for the computable characterization of generic shapes (e.g., connectivity, star-shape and convexity) have been made. In our opinion, the following aspects for the research on shape priors deserve more studies.
Firstly, weak shape priors, e.g., objects with weak convexity that allows the curvature
of part of its boundary to be little less than zero, will be useful for many applications.
Secondly, combining generic shape priors and deep learning methods to improve the stability of deep learning methods and segmentation accuracy has attracted
increasing attentions \cite{Liu_star,binary2021Luo}. In addition, learning more precise and complicate shape priors for concrete segmentation issues (e.g. prostate and gland segmentation ) is more practical in real applications.

\section*{Acknowledgments}
This work was supported by the Programs for Science and Technology Development of Henan Province (192102310181, 212102310305), RG(R)-RC/17-18/02-MATH, HKBU 12300819, NSF/RGC Grant N-HKBU214-19, ANR/RGC Joint Research Scheme (A-HKBU203-19), RC-FNRA-IG/19-20/SCI/01, and National Natural Science Foundation of China (No.11971149).



\bibliography{ref}
\end{document}